\documentclass[Afour,sageh,times]{sagej}
\usepackage{natbib}
\usepackage{moreverb}
\usepackage{times}
\usepackage[ruled]{algorithm2e}
\usepackage{color}
\usepackage{mathtools}
\usepackage{bm}
\usepackage{diagbox}
\usepackage{float}
\usepackage{epstopdf}
\usepackage{pifont}
\usepackage{multirow}
\usepackage{url}
\usepackage{verbatim}
\usepackage{booktabs}
\usepackage{graphicx} 
\usepackage{makecell}
\usepackage{soul}
\usepackage{xcolor}
\usepackage[normalem]{ulem}
\usepackage{etoolbox} 
\usepackage{arydshln}
\usepackage{tabularx}
\usepackage{threeparttable}
\usepackage[linkcolor=red,citecolor=red,urlcolor=red,colorlinks=true]{hyperref}
\newcommand\BibTeX{{\rmfamily B\kern-.05em \textsc{i\kern-.025em b}\kern-.08em
T\kern-.1667em\lower.7ex\hbox{E}\kern-.125emX}}

\DeclareMathOperator*{\argmin}{arg\,min}

\newcommand{\tp}{{^{\mathrm{T}}}}

\newcommand{\Tp}[1]{{{#1}^{\mathrm{T}}}}
\newcommand{\norm}[1]{\Vert{#1}\Vert}
\newcommand{\Norm}[1]{\left\Vert{#1}\right\Vert}

\newtheorem{theorem}{Theorem}

\newtheorem{lemma}{Lemma}

\newtheorem{definition}{Definition} 

\newtheorem{corollary}{Corollary}

\def\volumeyear{2016}
\begin{document}
\bibliographystyle{SageH}

\runninghead{Wang et al.}
\title{Certifiable Mutual Localization and Trajectory Planning for Bearing-Based Robot Swarm}
%\author{Yingjian Wang\affilnum{1}, Xiangyong Wen\affilnum{1} and Fei Gao\affilnum{1}}
\author{Yingjian Wang, Xiangyong Wen and Fei Gao}
\affiliation{Institute of Cyber-Systems and Control, Zhejiang University, Hangzhou, P.R. China}
\corrauth{Fei Gao, Institute of Cyber-Systems and Control, Zhejiang University, Hangzhou, 310027, P.R.China.}
\email{feigaoaa@zju.edu.cn}

\begin{abstract}
%	Recent advancements in vision sensors have catalyzed the adoption of inter-robot bearing measurements in multi-robot systems, significantly enhancing mutual localization and swarm collaboration.
	Bearing measurements, as the most common modality in nature, have recently gained traction in multi-robot systems to enhance mutual localization and swarm collaboration.  
	Despite their advantages, challenges such as sensory noise, obstacle occlusion, and uncoordinated swarm motion persist in real-world scenarios, potentially leading to erroneous state estimation and undermining the system's flexibility, practicality, and robustness.
	In response to these challenges, in this paper we address theoretical and practical problem related to both mutual localization and swarm planning.
	Firstly, we propose a certifiable mutual localization algorithm.
	It features a concise problem formulation coupled with lossless convex relaxation, enabling independence from initial values and globally optimal relative pose recovery.
	Then, to explore how detection noise and swarm motion influence estimation optimality, we conduct a comprehensive analysis on the interplay between robots' mutual spatial relationship and mutual localization. 
	We develop a differentiable metric correlated with swarm trajectories to explicitly evaluate the noise resistance of optimal estimation.
	By establishing a finite and pre-computable threshold for this metric and accordingly generating swarm trajectories, the estimation optimality can be strictly guaranteed under arbitrary noise.	
	Based on these findings, an optimization-based swarm planner is proposed to generate safe and smooth trajectories, with consideration of both inter-robot visibility and estimation optimality.
	Through numerical simulations, we evaluate the optimality and certifiablity of our estimator, and underscore the significance of our planner in enhancing estimation performance.
	The proposed methods are implemented on three quadrotors, and their effectiveness are demonstrated through extensive real-world experiments.
	The results exhibit considerable potential of our methods to pave the way for advanced closed-loop intelligence in swarm systems.
\end{abstract}

\maketitle
\section{Introduction}
	Multi-robot systems have emerged as a focal point of research in recent years, given their remarkable capabilities in handling complex tasks beyond the reach of individual robots, such as space exploration, perimeter surveillance, and package delivery.
	To further bring robot swarms out from laboratory to unknown and complex environments, lightweight, robust and reliable relative localization is crucial in the absence of external positioning devices.
	Utilizing bearing measurements for mutual localization is particularly appealing due to the its ultra lightweight property (\cite{martinelli2005multi,zhou2012determining,walter2018fast}).
	Bearing is the most common modality of measurements in nature.
	Robots can obtain inter-robot bearings by onboard vision detection, communicate these bearings at low bandwidth and perform relative localization with them.
	However, the practicality and flexibility of bearing's application in robot swarm are still hindered by deficiencies in the following aspects:
	\begin{enumerate}
		\item Certifiablity and Robustness. 
		Erroneous relative pose estimations can directly lead to the coordinate misalignment and system collapse. 
		Hence, it's essential for an estimator to provide correctness guarantee for its solutions.
		Besides, as inevitable and uncontrollable sensory noise can significantly influence estimation accuracy, the estimator must be capable to certify the incorrectness of its  solutions and reject them when detection noise corrupts estimation.
		However, existing local optimization based methods (\cite{nguyen2020vision,jang2021multirobot}) heavily depend on initial values and lack certifiablity for optimal solutions.
		\item Stablity and Reliability.
		To facilitate cooperation, swarm robots need to maintain reference frame consensus at all times.
		It requires robots to consistently produce stable measurements and perform reliable relative localization, even in complex environments.
		Fulfilling this requirement is not straightforward.
		Difficulties arise when robots miss crucial data due to obstacle occlusion or are positioned in areas where measurements are insufficient for relative poses recovery.
		These issues, often resulting from uncoordinated swarm motion, potentially lead to estimator failure and compromise system reliability.
	\end{enumerate}

	From above analysis, we conclude that a certifiable and robust mutual localization method is absolutely necessary and essential.
	Beyond that, incorporating estimation considerations into swarm motion planning, referred as estimation-aware planning, can further remove the shackle of estimation.
	Current estimation-aware planning methods primarily focus on reducing estimation uncertainty (\cite{walls2015belief}), avoiding unobservable locations (\cite{le2018localizability}) or increasing observational information (\cite{cossette2022optimal}).
	Although these methods do improve the performance of estimator somewhat, they do not strictly guarantee that the improved estimator can stably provide correct solution in unknown environments.
	It can be seen as the contradiction between undeterministic estimation improvement and uncontrollable environment interaction.
	The root issue is that these methods still follow a conventional and modularized system framework.
	In such a framework, mutual localization module and motion planning module are isolated and separated, and no research conducts targeted theoretical analysis and algorithm design specifically tailored to individual estimator.
	As a result, the performance of the estimator, as well as the performance of the entire system, remains limited in complex environments.
	
	To systematically enhance bearing-based robot swarm, we propose dismantling the isolation of individual modules in the swarm system and establishing a certifiable closed-loop swarm intelligence.
	The \textit{certifiable closed loop} comprises two key aspects: ensuring that mutual localization is certifiable, thereby laying the groundwork for swarm collaboration; and enabling planning algorithms to generate certifiable trajectories that guarantee optimal relative localization under arbitrary noise.
	By integrating both aspects, we envision a swarm system capable of operating with stability, reliability, and robustness, even in unknown and complex environments.
	
	To meet the first requirement of the certifiable closed loop, we propose a certifiable mutual localization algorithm in this paper.
	Specifically, we formulate the relative pose estimation as a non-convex maximum-likelihood problem and transform it into a semi-definite programming solvable in polynomial time through convex relaxation.
	This approach does not require intial values and can provide a globally optimal solution in noise-limited cases.
	Additionally, it establishes a readily verifiable condition for certifying the global optimality of a solution.
	It enables the rejection of incorrect relative poses and preventing coordinate confusion. 
	Compared to previous works (\cite{jang2021multirobot, wang2022certifiably}), our proposed method enhances efficiency and stability with its concise formulation.
	
	To meet the second requirement of the certifiable closed loop, we further propose \textit{certifiable swarm planning}, which is to plan appropriate swarm motion to ensure globally optimal estimation.
	Our research first explore the deep-rooted relationship between mutual localization and swarm motion theoretically.
	We introduce the \textit{certificate matrix}, a special matrix for verifying the global optimality of candidate solutions. 
	The minimum non-zero eigenvalue of this matrix, termed the \textit{certificate eigenvalue}, serves as an indicator of estimation optimality. 
	Additionally, comprehensive analysis of the certificate matrix is provided, including its dual description, noise analysis, degeneration identification, and eigenvalue bounds. 
	The dual description reveals that the certificate matrix not only originates from the estimation process but also serves as a metric for swarm motion. 
	Through noise analysis, we explore the influence of detection noise on estimation optimality and demonstrate that the required magnitude of noise to disrupt optimality varies with different swarm motions.
	This magnitude measures the swarm's noise resistance and is positively correlated with the certificate eigenvalue.
	In cases where the certificate eigenvalue is zero, indicating no noise resistance (a state we term \textit{degeneration}), we provide specific conditions of robot motion to identify these situations.
	To circumvent degeneration, we propose enhancing the certificate eigenvalue by estimation-aware swarm motion planning.
	We also show that if the magnitude of noise is limited, the minimum required certificate eigenvalue to maintain estimation optimality is finite and can be explicitly pre-computed. 
	This finding allows us to generate trajectories that strike a tradeoff between estimation optimality guarantee and smoothness requirements.
	
%	These applications hinge on a fundamental requirement: the determination of relative robot poses, essentially achieving consensus on a reference frame.
%	This task is relatively straightforward with external localization devices like GPS or motion capture systems.
%	However, it presents a significant challenge in settings where such devices are not viable, including indoor, underground, or extraterrestrial environments. 
%	As a solution, mutual localization emerges, utilizing measurements between robots for estimating relative poses, thus lessening reliance on external systems.
%	A prevalent method in this regard is the use of bearing measurements, typically acquired through vision sensors capable of mutual detection (\cite{nguyen2020vision,walter2018fast}).
%	Bearing, defined as a directed unit vector within a robot's local frame, offers high accuracy and independence from environmental constraints.
%	As such, its integration into multi-robot systems has seen a substantial upswing, aligning with the rise of bearing-based mutual localization (\cite{martinelli2005multi,zhou2012determining, walter2019uvdar,wang2022certifiably}).

	\begin{figure}[t]
		\centering
		\includegraphics[width=0.5 \textwidth]{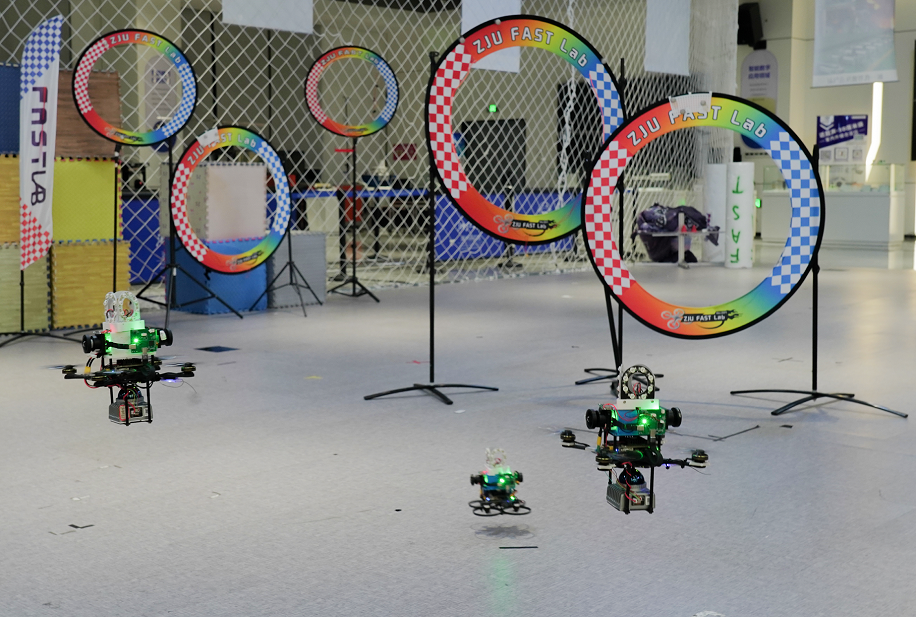}
		\caption{\label{fig:exp3} A snapshot of a bearing-based robot swarm. 
		Robots use active light emitter and omnidirectional sensor to obtain inter-robot bearings with pixel-level precision.}
		\vspace{-0.3cm}
	\end{figure}

	Building upon these theoretical foundations, we further consider the issus of line-of-sight occlusion and propose a practical implementation of swarm planning.
	In the frontend, we generate paths for robots that ensure continuous inter-robot visibility. 
	Specifically, we perform kinodynamic path searching for a designated robot, generate a sequence of star convex polytopes (\cite{katz2007direct}) along the searched path, and constrain other robots to to remain within these polytopes.
	This design ensures that all robots within a polytope can produce stable inter-robot bearing measurements with the designated robot. 
	Then, swarm trajectories optimization is performed in the backend.
	We focus on raising the certificate eigenvalue to a specified bound to guarantee the estimation optimality, while ensuring safety and dynamic feasibility.

	The concept of certifiable mutual localization was first proposed in our previous work (\cite{wang2022certifiably}), which focused solely on the isolated estimation module.
	However, as we analyzed before, estimation performance is often hindered by swarm motion and enviroment.
	This paper extends certifiable mutual localization to certifiable swarm planning, bringing systematic improvements in certifiability, robustness, and flexibility.
	The contributions of this paper are summarized as:
	\begin{enumerate}
		\item We present a certifiable algorithm for bearing-based mutual localization.
		It can recover globally optimal relative poses from non-convex problems and provide an approach to reject incorrect solution.
		\item We conduct theoretical analysis on how detection noise and swarm motion affect the localization optimality. 
		A series of novel conclusions lay the theoretical groundwork for certifiable swarm planning.
		\item We propose a practical implementation of certifiable swarm planning.
		It generates trajectories ensuring continuous inter-robot visibility in complex environments and estimation optimality under arbitrary noise.
		\item We conduct extensive experiments in both simulated and real-world environments.
		The results demonstrate the outperforming efficacy of our methods and their potential in various applications.
	\end{enumerate}

	\section{Related Work}
	\label{sec:RelatedWork} 
	\subsection{Relative Pose Estimation}
	\label{subsec:relatedwork1}
	Relative pose estimation is a fundamental module in multi-robot systems, and existing methods can generally be categorized into three main types: map-based, range-based, and bearing-based relative pose estimation methods.
	
	Map-based methods rely on the exchange of map feature or landmark information between robots to achieve relative localization.
	These methods typically involve interloop detection, feature extraction in overlapping areas, and the construction of geometric constraints to estimate relative poses.
	In CCM-SLAM (\cite{schmuck2019ccm}), robots transmit local visual features to a central server for centralized place recognition and geometric verification. 
	DDF-SAM (\cite{cunningham2010ddf, cunningham2013ddf}) presents a distributed architecture that employs Gaussian elimination to exchange marginals.
	A two-stage approach is proposed in (\cite{choudhary2017distributed}) which uses distributed Gauss–Seidel algorithm for initialization and Gauss–Newton method for optimization.
	%	This method has also been implemented as the distributed optimization back-end in decentralized multi-robot SLAM systems \cite{lajoie2020door, cieslewski2018data}.
	%	Recently, tightly-coupled visual-inertial fusion has been applied in cooperative localization to achieve high accuracy.
	%	Zhu \cite{zhu2021distributed} proposed a covariance intersection based estimator to compensate
	%	for unknown correlations between robots. 
	%	Xu propsed Omni-swarm \cite{xu2022omni}, which extends their previous work \cite{xu2020decentralized} by enhancing the field of View to achieve omnidirectional perception.
	Map-based methods require sufficient overlapping features between robots and significant bandwidth for information exchange, which limits the performance in environments with many similar scenes or poor communication conditions.
	
	Mutual localization, which utilizes inter-robot range or bearing measurements for relative pose estimation, is an effective approach to reduce dependence on the environment. 
	Range data is commonly obtained from strengths of selected broadcasted signals, such as ultra-wideband (UWB).
	Some works (\cite{guo2017ultra, guo2019ultra, ziegler2021distributed}) rely on range data from fixed UWB anchors to prevent robots' odometry drift.
	The additional infrastructure requirenment limits the practical applicability.
	Other approaches (\cite{jiang20193, li2020robot, li20223,nguyen2023relative}) deploy UWB on robots and fuse distance measurements with ego-motion estimations to achieve relative localization.
	These methods typically yield accuracy at the meter to decimeter level due to the limited precision of UWB measurements.
	
	Bearing-base methods (\cite{martinelli2005multi, chang2011vision,zhou2012determining, walter2019uvdar,nguyen2020vision,jang2021multirobot,wang2022certifiably,wang2023bearing}) play a significant role in relative pose estimation due to the stability of bearing measurements and their close relationship with orientation.
	Early works such as (\cite{martinelli2005multi, chang2011vision}) utilize the extended Kalman Filter as a nonlinear estimator to fuse bearing measurements for relative localization.
	In (\cite{zhou2012determining}, algebraic and numerical methods are provided for mutual localization with combinations of range and bearing measurements. 
	It suffered from degeneration under noise as it relies on the minimum required measurements.
	A coupled-probabilistic-data-association filter is propsoed in (\cite{nguyen2020vision}) to fuse anonymous detection results and IMU data.
	This method requires intensive computation for data association.
	An alternating optimization method is presented in (\cite{jang2021multirobot}) for mutual localization in multi-robot monocular SLAM using only bearings, but it relies on good initial values and has no guarantee for the correctness of obtained solutions.
	Our prior works (\cite{wang2022certifiably, wang2023bearing}) address this issue through lossless convex relaxation.
	However, although robustness has been considered, an explicit analysis of how much noise can be tolerated is rarely presented in previous works .

	\subsection{Estimation-aware Planning}
	\label{subsec:relatedwork3}
	Previous works acknowledge the significance of mutual spatial geometry in relative state estimation and consider how to control the motion of robots to improve estimation.
	Based on the measure of localization quality considered, these approaches can be categorized into covariance matrix based methods (\cite{zhou2011multirobot,bahr2012dynamic, walls2015belief}), Fisher Information Matrix (FIM)  based methods (\cite{le2018localizability,cano2021improving,papalia2022prioritized,cossette2022optimal}), and rigidity based method (\cite{zelazo2012rigidity}).
	
	In covariance matrix based methods, Zhou and Roumeliotis (\cite{zhou2011multirobot}) propose utilizing the trace of the error covariance matrix as a cost function for trajectory planning. 
	Walls et al. (\cite{walls2015belief}) perform belief space planning for server robots to find informative relative trajectories by minimizing clients' uncertainty. 
	These methods typically rely on prior knowledge on the robots' kinematics to derive state and covariance propagation.
	
	FIM was first introduced into observability analysis in (\cite{hammel1989optimal}). 
	Since then, some metrics are defined to evaluate the amount of information obtained, such as A-optimal, D-optimal and E-optimal designs.
	These metrics are commonly used for motion optimization to improve localizability (\cite{le2018localizability,cano2021improving,papalia2022prioritized, cossette2022optimal}).
	In (\cite{le2018localizability}), a FIM-based planning method is proposed to avoid unobservable positions in localization with static UWB anchors.
	In (\cite{cossette2022optimal}), authors further proposed a optimal multi-robot formation for relative localziation without the use of anchors.
	%	It should be noted that FIM is associated with the measurement model and independent of any particular estimator, while our metric in this paper has direct, tight and exact connection with our proposed estimatior.
	
	Rigidity (\cite{roth1981rigid}) is a property of robot network that quantifies how range or bearing measurements constrain the shape of the network. 
	This property is connected to the possibility of uniquely determining robots' positions from inter-robot measurements (\cite{eren2004rigidity}), and hence useful to enforce robots.
	Zelazo et al. (\cite{zelazo2012rigidity, zelazo2015decentralized}) characterize the rigity matirx for range-only robot swarm, and developed a decentralized planning framework to maintain the rigidity as robots move.
	In (\cite{le2018localizability}), authors reveal that the FIM exactly corresponds to a rigidity matrix when the noise follows a Gaussian distributions.
	Walter et al. (\cite{walter2023distributed}) derive a formation-enforcing control for beaing-based multi-robot system from rigidity theory to address the issus of obsearvation noise.
	In contrast to solely improving robustness, our method in this paper aims to provide an exact assurance of estimation optimality under noise.

	\section{Framework and Notation}
	\label{sec:system} 
	
	\subsection{Overall Framework}
	\begin{figure}[t]
		\centering
		\includegraphics[width=0.5\textwidth]{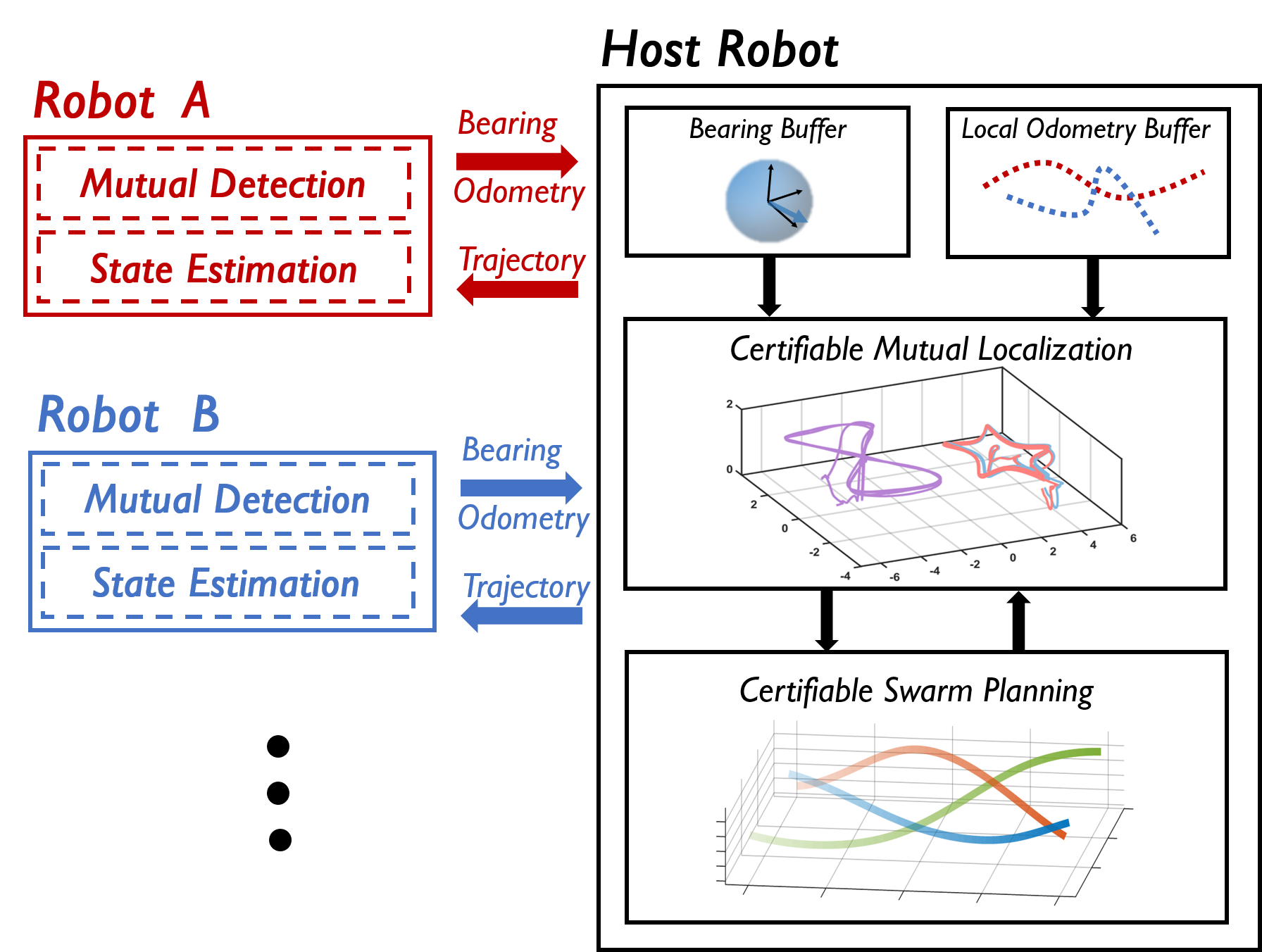}
		\caption{\label{fig:overview} Proposed framework of our bearing-based robot swarm.
		Team robots estimate repective states, including rotation and translation and obtain inter-robot bearings by mutual detection.
	These data are communicated to a host robot, and stored in the bearing buffer and local odometry buffer.
	The certifiable mutual localization module recover relative poses using these data and align robots' reference frames.
	The certifiable swarm planning module generate swarm trajectories in the common frame, which are sent back robots for tracking. }
		\vspace{-0.4cm}
	\end{figure}
	
	The system architecture of our system is illustrated in Fig. \ref{fig:overview}.
	Each robot sends outputs of mutual detection and state estimation to a designated host robot. 
	The host robot performs certifiable mutual localization to align reference frames of all robots using inter-robot bearing measurements and local odometry (Sec.\ref{sec:estimation}). 
	As the relationship between estimation optimality and the configuration of swarm trajectories is investigated (Sec.\ref{sec:theoretical}), a differentiable metric of swarm trajectories is proposed to evaluate the noise resistance of estimation.
	Leveraging this metric, the host robot then performs certifiable swarm planning within the common reference frame to generate trajectories (Sec.\ref{sec:planning}) and sends them back to robots.
	
	\subsection{Notation}
	\label{subsec:notation}
	In this paper, lowercase and bold uppercase letters (e.g. v and $\mathbf{A}$) are reserved for vectors and matrices, respectively. 
	We denote the set of real $d \times d$ symmetric and symmetric positive semi-definite (PSD) matrix as $\mathbb{S}^{n}$ and $\mathbb{S}_+^{n}$ respectively. 
	For general matrices $\mathbf{A}$ and $\mathbf{B}$, $\mathbf{A} \succeq \mathbf{B}$ indicates matrix $\mathbf{A}-\mathbf{B}$ is positive semi-definite, $\mathbf{A} \otimes \mathbf{B}$ the Kronecker (matrix tensor) product, $\mathbf{A}^{\dagger}$ the Moore-Penrose  pseudoinvers, $\text{Tr}(\mathbf{A})$ the trace, $\text{det}(\mathbf{A})$ the determinant and vec($\mathbf{A}$) the vectorization operation that concatenates the columns of $\mathbf{A}$. 
	We write $\mathbf{I}_d \in \mathbb{R}^{d\times d}$ and $\mathbf{0}_{d\times s} \in \mathbb{R}^{d\times s}$ for the identity matrix and zero matrix, $e_i \in \mathbb{R}^d$  for the $i$th unit coordinate vector, and $\mathbf{1}_d \in \mathbb{R}^d$ for the all-ones vector.
	$SO(n)$ (special orthogonal group) and $O(n)$ (orthogonal group) are defined as: $SO(n)  = \{\mathbf{R} \in \mathbb{R}^{n\times n}: \mathbf{R}\tp\mathbf{R} = \mathbf{I}_n, \text{det}(\mathbf{R} ) = 1\}$ and $O(n)  = \{\mathbf{R} \in \mathbb{R}^{n\times n}: \mathbf{R}\tp\mathbf{R} = \mathbf{I}_n\}$, and $\mathbf{G}_1 \times \mathbf{G2}$ means the direct product group of the group $\mathbf{G}_1$ and $\mathbf{G}_2$.
	We denote by Diag($\mathbf{A}_1,\cdots,\mathbf{A}_N$), the block-diagonal matrix with matrices $\mathbf{A}_1, \cdots, \mathbf{A}_N$ as blocks on its main diagonal. 
	For a ($d\times d$)-block matrix $\mathbf{M}$, $\mathbf{M}_{ij}$ denotes the $(i,j)$th ($d\times d$) matrix.

	\section{Bearing-based Certifiable Mutual Localization}
	\label{sec:estimation} 
	In this section, a certifiable algorithm for mutual localization is proposed. 
	It achieves precise coordination alignment among multiple robots by leveraging inter-robot bearing measurements and each robot's local odometry.
%	Our method has some requirements to met.
%	Robots need to be equipped omnidirectional sensors and an ego-motion estimator, enabling them to obtain inter-robot bearing measurements and their own local odometry.
%	Besides, the data association between bearing measurements and observed robots is given.
	In Sec.\ref{subsec:formulation}, we formulate mutual localization problem as a maximum likelihood estimation (MLE) with $SO(3)$ constraints.
	To address the inherent non-convexity of the problem, we adopt semi-definite relaxation (SDR) in Sec.\ref{subsec:sdr}, transforming the original problem into a convex programming to ensure global optimality of solutions.
%	A practical sliding-window framework is introduced in Sec.\ref{subsec:slide_window} to enhance computational efficiency and mitigate odometry drift. 
	
	\subsection{Problem Formulation}
	\label{subsec:formulation}
	
	\begin{figure}[t]
		\centering
		\includegraphics[width=0.45\textwidth]{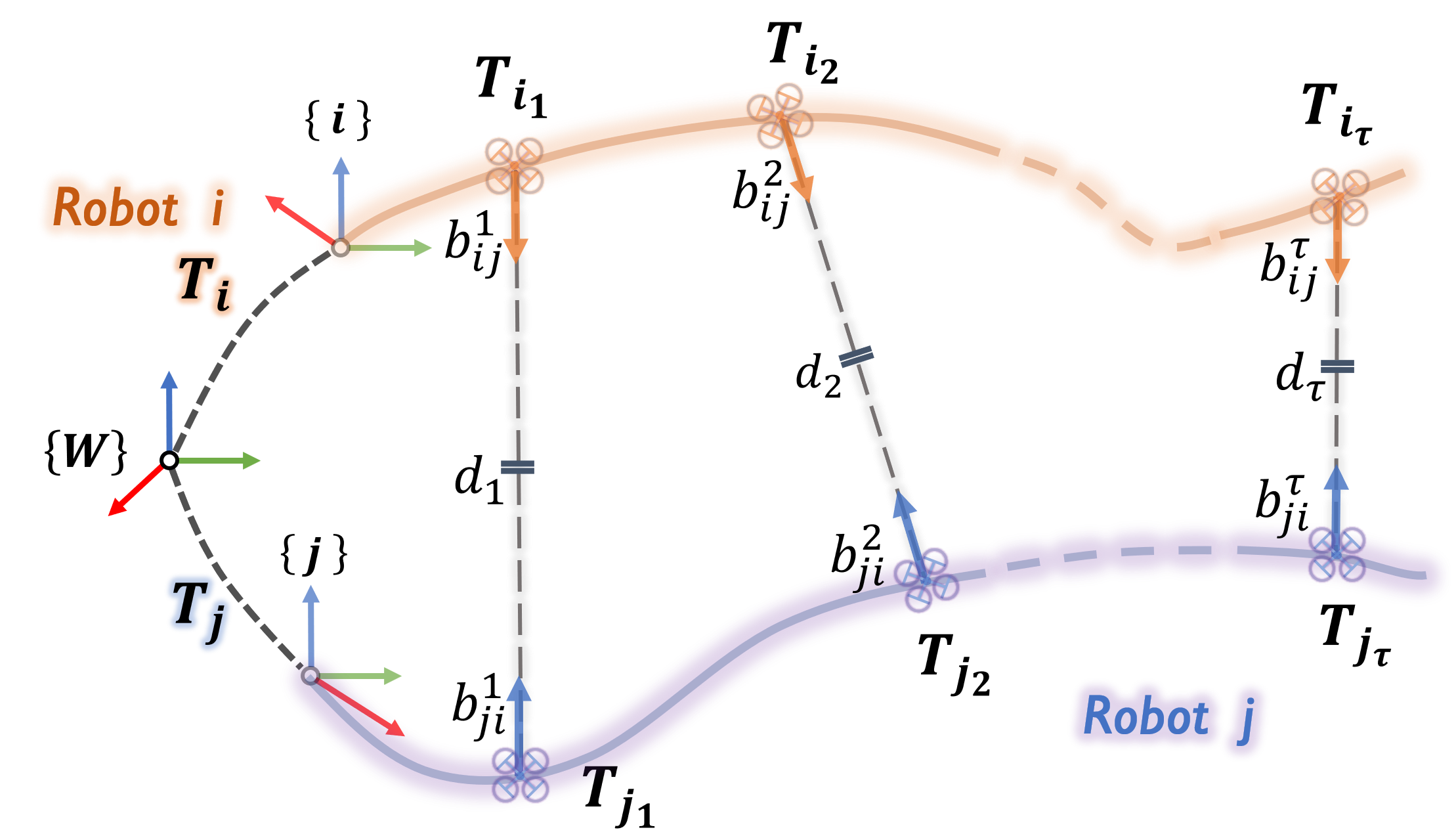}
		\caption{\label{fig:two-robot} Demonstration of initial poses in the world ($\mathbf{T}_x = \{\mathbf{R}_x, t_x\}$) and local odometry ($\mathbf{T}_{x_\tau} = \{\mathbf{R}_{x_\tau},t_{x_\tau}\}$) for the robot $x$, $x\in\{i,j\}$. The ominidirectional bearing observations ($b_{ij}^\tau$ and $b_{ij}^\tau$) and the distance ($d_\tau$)  between robots at time $\tau$ are represented as colored arrows and doted lines.}
		\vspace{-0.4cm}
	\end{figure}
	We consider a multi-robot system where robots have random initial poses.
	The relevant coordinates of two moving robots are illustrated in Fig. \ref{fig:two-robot}.
	As depicted, the initial poses of robots in the world frame, denoted as $\mathbf{T}_i$ and $\mathbf{T}_j$, are unknown, while their poses in their respective local frames at time $\tau$, denoted as $\mathbf{T}_{i_\tau}$ and $\mathbf{T}_{j_\tau}$, are obtained from real-time odometry.
	At time $\tau$, we define $p_i^{\tau}$ and $p_j^{\tau}$ as the positions of robots in the world frame. 
	The true (latent) relative bearing of robot $j$ repect to robot $i$ is represented as a unit vector $\hat{b_{ij}^{\tau}}$:
	\begin{align}
		\hat{b}_{ij}^{\tau} = \left(\mathbf{R}_i \mathbf{R}_{i_\tau}\right)\tp \frac{p_j^{\tau} - p_i^{\tau}}{\Norm{p_j^{\tau} - p_i^{\tau}}}. \label{equ:b1}
	\end{align}
	We employ unit vectors instead of angles to represent bearings here to avoid complicated trigonometric functions.
	
	In practice, the actual bearing measurements obtained through vision detection, denoted as $b_{ij}^{\tau}$, are inevitably affected by noise.
	To model the noise, we follow the approach in (\cite{li2022three,bishop2010optimality}) to express the noisy bearing measurement in an additive form:
	\begin{align}
		b_{ij}^{\tau} = \hat{b}_{ij}^{\tau} + \epsilon_i^\tau.  \label{equ:noise}
	\end{align}
	where the measurement is corrupted by independent and identically distributed Gaussian noise $ \epsilon_i^\tau \sim \mathcal{N}(0,\sigma^2 \mathbf{I}_3)$.
	It is noticed that in the observation process, the distance between robots, $d_{ij}^{\tau} =\Norm{p_j^{\tau} - p_i^{\tau}}$, is unknown. 
	%	Denote the collection of time instances when robots are intervisible as $J  = \{\tau_1,\cdots,\tau_n\}$. 
	The objective of bearing-based mutual localization in a $N$-robot system is to estimate a set of initial poses $\mathbf{\Theta}  = [\mathbf{T}_1, \mathbf{T}_2, \cdots, \mathbf{T}_N] \in \mathbb{R}^{4\times4N}$ from some noisy inter-robot bearing measurements $b_{ij}^\tau$ and the local odometry $\mathbf{T}_{i_\tau}$ of each robot.
	The parameter space of $\mathbf{\Theta}$ is $\mathcal{P}  = \{SE(3) \times \cdots \times SE(3)\}_N$.

	Consider the  $N=2$ case firstly. 
	Multiplying $\mathbf{R}_i \mathbf{R}_{i_\tau} d_{ij}^\tau $ on both sides of (\ref{equ:noise}) gives
	\begin{align}
		\mathbf{R}_i \mathbf{R}_{i_\tau} d_{ij}^\tau b_{ij}^\tau &= p_j^{\tau} - p_i^{\tau} + \mathbf{R}_i \mathbf{R}_{i_\tau} d_{ij}^\tau \epsilon_i^\tau. \label{equ:e1}
	\end{align}
	Similarly, for $b_{ji}^\tau$, we have:
	\begin{align}
		\mathbf{R}_j \mathbf{R}_{j_\tau} d_{ji}^\tau b_{ji}^\tau &= p_i^{\tau} - p_j^{\tau} + \mathbf{R}_j \mathbf{R}_{j_\tau} d_{ij}^\tau \epsilon_j^\tau, \label{equ:e2}
	\end{align}
	where $d_{ji}^\tau = d_{ij}^\tau$ and $\epsilon_j^\tau \sim \mathcal{N}(0,\sigma^2 \mathbf{I}_3)$. By adding (\ref{equ:e1}) and (\ref{equ:e2}), we obtain:
	\begin{align}
		&\mathbf{R}_i \mathbf{R}_{i_\tau} d_{ij}^\tau b_{ij}^\tau + \mathbf{R}_j \mathbf{R}_{j_\tau} d_{ji}^\tau b_{ji}^\tau = d_{ji}^\tau \eta^\tau_{ij},  \label{equ:e3} \\
		\Longrightarrow \ \ &\mathbf{R}_i \mathbf{R}_{i_\tau} b_{ij}^\tau + \mathbf{R}_j \mathbf{R}_{j_\tau} b_{ji}^\tau = \eta^\tau_{ij}, \label{equ:e4}
	\end{align}
	where $\eta^\tau_{ij}=\mathbf{R}_i \mathbf{R}_{i_\tau} \epsilon_i^\tau+\mathbf{R}_j \mathbf{R}_{j_\tau}\epsilon_j^\tau$ also follows a Gaussian distribution $\eta^\tau_{ij} \sim \mathcal{N}(0,\mathbf{\Sigma_{\tau}})$, $\mathbf{\Sigma_{\tau}} = 2\sigma^2\mathbf{I}_3$ (\cite{martin1979multivariate}). 
	The probability density function (PDF) of $\eta^\tau_{ij}$ is given by $p: \mathbb{R}^3 \rightarrow \mathbb{R}^{+}$
	\begin{align}
		p(\eta^\tau_{ij}) &= \frac{1}{(2\pi)^{3/2} \text{det}(\mathbf{\Sigma}_\tau)^{1/2} } \textup{exp}\left(-\frac{1}{2} \eta^\tau_{ij}\tp \mathbf{\Sigma}_\tau^{-1}\eta^\tau_{ij}\right).
	\end{align}

	Next we extend our analysis to a multi-robot system with $N\geq3$. 
	We model the set of intervisible robots as an undirected graph $\mathbf{\mathcal{G}}  = (\mathbf{\mathcal{V}} , \mathbf{\mathcal{E}} )$, where $\mathbf{\mathcal{V}}  = [1,N], \mathbf{\mathcal{E}}  \subseteq \mathbf{\mathcal{V}}  \times \mathbf{\mathcal{V}} $.
	Each node in $\mathbf{\mathcal{V}} $ corresponds to a robot, and each edge $(i,j) \in \mathcal{E} $ indicates that robots $i$ and $j$ are intervisible. 
	For a node $i \in \mathbf{\mathcal{V}} $, the set of neighboring nodes is defined as $\mathcal{N}_i$, $(i,j) \in \mathbf{\mathcal{E}}  \Leftrightarrow j \in \mathcal{N}_i$. 
	For a edge $(i,j) \in \mathbf{\mathcal{E}}$, the collection of time instances when robots are intervisible is denoted as $J_{ij}$.  
	Now, the problem transforms into finding a set of rotations  $\mathbf{\Theta_R}  = [\mathbf{R}_1, \mathbf{R}_2, \cdots, \mathbf{R}_N] \in \mathbb{R}^{3\times3N}$ to align the inter-robot bearing measurements. 
	The parameter space of $\mathbf{\Theta_R}$ is $\mathcal{P}_{SO(3)}  = \{SO(3) \times \cdots \times SO(3)\}_N$.

	Define selection matrix $\mathbf{C}_i  = e_{i} \otimes \mathbf{I}_3 \in \mathbb{R}^{3N\times3}$ and  $g_{ij}^{\tau}  = \mathbf{R}_{i_\tau} b_{ij}^{\tau}$, we rewrite (\ref{equ:e4}) as 
	\begin{align}
		\eta^{\tau}_{ij} &= \mathbf{\Theta_R} \mathbf{C}_i g_{ij}^{\tau} + \mathbf{\Theta_R} \mathbf{C}_j g_{ji}^{\tau} = \mathbf{\Theta_R} y_{ij}^{\tau} \in \mathbb{R}^3,
		%		 =(\mathbf{\Theta_R}(\Tp{y_{ij}^{\tau}} \otimes \mathbf{I}_{3}) \text{vec}
	\end{align}
	where $y_{ij}^{\tau} = [\mathbf{0}_{1\times3}, \cdots, {g_{ij}^\tau}\tp, \cdots, {g_{ji}^\tau}\tp, \cdots, \mathbf{0}_{1\times3}]\tp \in \mathbb{R}^{3N}$.
	Stacking the errors $\eta^{\tau}_{ij}$ gives the error vector $\eta_{ij}$ for $(i,j) \in \mathcal{E}$ and the overall error vector $\eta$ for the entire estimation process:
	%	\begin{align}
		%		\eta_{ij} & = [\Tp{\eta^{\tau_1}_{ij}}, \Tp{\eta^{\tau_2}_{ij}}, \cdots, \Tp{\eta^{\tau_n}_{(N-1)N}}]\tp, \\
		%		\eta  & = [\Tp{\eta_{12}}, \Tp{\eta^{\tau_2}}, \cdots, \Tp{\eta^{\tau_n}}]\tp \in \mathbb{R}^{3m}.
		%	\end{align}
	\begin{align}
		\eta_{ij} & = [\Tp{\eta^{\tau_1}_{ij}}, \Tp{\eta^{\tau_2}_{ij}}, \cdots, \Tp{\eta^{\tau_n}_{(N-1)N}}]\tp, \\
		\eta  & = [\eta_{12}^\mathrm{T}, \eta_{13}^\mathrm{T}, \cdots, \eta_{(N-1)N}^\mathrm{T}]\tp.
	\end{align}
	The PDF of $\eta$ is given by $p: \mathbb{R}^m \rightarrow \mathbb{R}^{+}$
	\begin{align}
		p(\eta) = \frac{1}{(2\pi)^{3m/2} \text{det}(\mathbf{\Sigma})^{1/2}} \textup{exp}\left(-\frac{1}{2} \eta\tp \mathbf{\Sigma}^{-1} \eta\right), \label{equ:e5}
	\end{align}
	where $m$ is the number of data and $\mathbf{\Sigma} = 2\sigma^2 \mathbf{I}_{3m}$.
	
	Taking negative logarithms of (\ref{equ:e5}), it follows that a maximum-likelihood estimate $\mathbf{\Theta_R^*}_{\textup{MLE}}$ is obtained as a minimizer of the following problem:
	\begin{align}
		\min_{\mathbf{\Theta_R} \in \mathcal{P}_{SO(3)}}  \frac{1}{2} \eta\tp \mathbf{\Sigma}^{-1} \eta. \label{equ:mle}
	\end{align}
	
	Expanding the cost function (\ref{equ:mle}), we have a quadratic reformulation:
	\begin{equation}\label{eq:multi-line}
		\begin{aligned}
			\frac{1}{2} \eta\tp \mathbf{\Sigma}^{-1} \eta &= \frac{1}{4\sigma^2}  \sum_{(i,j) \in \mathcal{E}} \sum_{\tau \in J_{ij}}  \eta^{\tau}_{ij}\tp \eta^{\tau}_{ij} \\
			&= \frac{1}{4\sigma^2}  \sum_{(i,j) \in \mathcal{E}} \sum_{\tau \in J_{ij}}  \textup{Tr}\left( y_{ij}^\tau y_{ij}^\tau\tp  \mathbf{\Theta_R^\mathrm{T}} \mathbf{\Theta_R}\right).
		\end{aligned}
	\end{equation}
	Define a (3$\times$3)-block symmetric data matrix $\mathbf{M}$, such that $\mathbf{M}= \sum_{(i,j) \in \mathcal{E}} \sum_{\tau \in J_{ij}}  y_{ij}^\tau y_{ij}^\tau\tp \in \mathbb{R}^{3N \times 3N}$, then the bearing-based relative pose estimation is reformulated as 
	
	%	\begin{equation}
		%		\begin{aligned}
			%			\mathbf{M}_{ij} &= 
			%			\begin{cases}
				%				\sum_{\tau \in J_{ij}}  g_{ij}^\tau g_{ji}^\tau\tp, \\ 
				%				\mathbf{0}
				%			\end{cases} \text{if}\ i \neq j \notag\\
			%			\mathbf{M}_{ii} &= \sum_{\tau \in J} \sum_{k \in V_\tau(i)} g_{ik}^\tau g_{ik}^\tau\tp
			%		\end{aligned}
		%	\end{equation}
	%	Then the bearing-based relative pose estimation problem is reformulated as 
	\begin{align}
		\min_{\mathbf{\Theta_R} \in \mathcal{P}_{SO(3)}} \textup{Tr}\left(\mathbf{M}  \mathbf{\Theta_R^\mathrm{T}} \mathbf{\Theta_R} \right). \label{equ:pro}
	\end{align}
	which is a non-convex problem with $SO(3)$ constraints.

	\subsection{Semidefinite Relaxation}
	\label{subsec:sdr}
	Since convex problem can be solved in polynomial time and the solution has strict global optimality, it is appealing to transform a non-convex problem into a convex problem losslessly.
	Relaxing problems of a specific form, such as (\ref{equ:pro}), has been studied extensively in both optimization (\cite{cifuentes2020local,lasserre2001global}) and robotics (\cite{aholt2012qcqp, eriksson2018rotation,rosen2019se}). 
	%	We turn our attention towards finding a convex relaxation method that ensure the global optimality of the solution to (\ref{equ:pro}).
	%	This form of problem has been studied extensively in both optimization \cite{cifuentes2020local,lasserre2001global} and robotics \cite{aholt2012qcqp, eriksson2018rotation,iglesias2020global,rosen2019se}. 
	%	We follow the steps outlined in \cite{rosen2019se} to obtain the SDP relaxtion formulation of the problem (\ref{equ:pro}). 
	First, the condition $\mathbf{\Theta_R} \in \mathcal{P_R}$ can be relaxed to obtain
	\begin{align}
		\min_{\mathbf{\Theta_R} \in \mathcal{P}_{O(3)}}  \textup{Tr}\left(\mathbf{M} \mathbf{\Theta_R^\mathrm{T}} \mathbf{\Theta_R} \right). \label{equ:r1}
	\end{align}
	where $\mathcal{P}_{O(3)}  = \{O(3) \times \cdots \times O(3)\}_N$. 
	This relaxation has been shown to often be tight in practice (\cite{rosen2019se}).
	Introduce Lagrange multipliers $\mathbf{\Lambda}_i \in \mathbb{S}^{3}$ for the symmetric orthogonality constraint $\mathbf{R}_i^\mathrm{T} \mathbf{R}_i = \mathbf{I}_3$, and define $\mathbf{\Lambda}  = \text{Diag}(\mathbf{\Lambda}_1, \cdots, \mathbf{\Lambda}_N)$ to consolidate the Lagrange multipliers,
	the Lagrangian corresponding to (\ref{equ:r1}) is given by 
	\begin{equation} \label{equ:dual0}
		\begin{aligned}
			\mathcal{L}(\mathbf{\Theta_R}, \mathbf{\Lambda})  &= \textup{Tr}(\mathbf{M} \mathbf{\Theta_R^\mathrm{T}}  \mathbf{\Theta_R}) - \textup{Tr}(\mathbf{\Lambda}  (\mathbf{\Theta_R^\mathrm{T}} \mathbf{\Theta_R} - \mathbf{I}_{3N})) \\
			&= \textup{Tr}((\mathbf{M-\Lambda}) \mathbf{\Theta_R^\mathrm{T}}\mathbf{\Theta_R}) + \textup{Tr}(\mathbf{\Lambda}).
		\end{aligned}
	\end{equation}
	Then, the Lagrangian dual problem for (\ref{equ:r1}) is formulated as:
	\begin{align}
		\max_{\mathbf{\Lambda}} \min_{\mathbf{\Theta_R}} \mathcal{L}(\mathbf{\Theta_R}, \mathbf{\Lambda}) = \textup{Tr}(\mathbf{\Theta_R} (\mathbf{M-\Lambda}) \mathbf{\Theta_R^\mathrm{T}}) + \textup{Tr}(\mathbf{\Lambda}). \label{equ:dual}
	\end{align}
	
	Observe that the unconstrained optimum of $\min_{\mathbf{\Theta_R}} L(\mathbf{\Theta_R}, \mathbf{\Lambda})$ is either $ \textup{Tr}(\mathbf{\Lambda})$ when $\mathbf{M-\Lambda} \succeq 0$, or $-\infty$ otherwise, the dual problem (\ref{equ:dual}) is equivalent to the following semi-definite program
	\begin{align}
		\max_{\mathbf{M-\Lambda}\succeq0}\textup{Tr}(\mathbf{\Lambda}). \label{equ:dual1}
	\end{align}
	%	A straightforward application of the duality theory for semi-definite programs (see Appendix for detail) shows the dual of (\ref{equ:dual1}) is 
	A straightforward application of the duality theory for semi-definite programs (\cite{boyd2004convex}) shows the dual of (\ref{equ:dual1}) is 
	\begin{align}
		\min_{\mathbf{Z} \in \mathcal{C}} \ &\textup{Tr}\left(\mathbf{M} \mathbf{Z} \right) \label{equ:r3} 
	\end{align}
	where $\mathcal{C}  = \{ \mathbf{X} \in \mathbb{S}_{+}^{3N}: \mathbf{X}_{ii} = \mathbf{I}_3, i \in [1,N] \}$ is a compact convex set. 
	%	The advantage of the SDP formulation (\ref{equ:r3}) over the original non-convex problem (\ref{equ:r1}) is that it can be solved in polynomial time, yielding a global solution $\mathbf{Z}^*$.
	%	We leverage an off-the-shelf SDP solver such as Mosek to efficiently solve the SDP problem.
	The problem  (\ref{equ:r3}) can be solved efficiently to obtain  $\mathbf{Z^*}$ using any off-the-shelf SDP solver. 
	To recover the relative rotation from $\mathbf{Z^*}$, a rank-3 decomposition is performed to obtain $\mathbf{Z^*} = \mathbf{Y^*}\tp \mathbf{Y^*}$, where $\mathbf{Y^*} \in \mathcal{P}_{O(3)}$.
	Then the optimal rotation matrix $\mathbf{R}_i^*$ can be recovered in accordance with the $SO(3)$ constraints:
	\begin{align*}
		\mathbf{R}_i^* = \mathbf{V}_i \begin{bmatrix}
			1 & 0 & 0 \\
			0 & 1 & 0\\
			0 & 0 & \textup{det}(\mathbf{U}_i \mathbf{V}_i^\mathrm{T}) \\
		\end{bmatrix}  \mathbf{U}_i^\mathrm{T}
	\end{align*} 
	where $\mathbf{U}_i$ and $\mathbf{V}_i$ are unitary matrices obtained from the singular value decomposition $\mathbf{Y}_i^* = \mathbf{U}_i \mathbf{\Sigma}_i \mathbf{V}_i^\mathrm{T}$.
	
	As the optimal relative rotations $\mathbf{\Theta_R^*}$ have been obtained, the optimal relative translations and distances between robot can be solved in closed-form. 
	Define the optimization variable as 
	\begin{align*}
		x  = [t_1^\mathrm{T},\ t_2^\mathrm{T},\ \cdots, t_N^\mathrm{T},\ \{d_\tau^{ij}\}]\tp.
	\end{align*}
	Recalling (\ref{equ:e1}) and (\ref{equ:e2}), we define a linear-form error involving distance and translation variables as follows:
	\begin{align*}
		e_{ij}^\tau = d^\tau_{ij}\left(\mathbf{R}_i^*  g^\tau_{ij} + \mathbf{R}_j^* g^\tau_{ji}\right) - \mathbf{R}_i^* t_{i_\tau} - t_i - \mathbf{R}_j^* t_{j_\tau} - t_j.
	\end{align*}
	It can be reformulated as 
	\begin{align}
		e_{ij}^\tau &=  \mathbf{A}_{ij}^\tau x - \beta_{ij}^\tau, \label{equ:t3}
	\end{align}
	where
	\begin{align*}
		&\mathbf{A}_{ij}^\tau = \begin{bmatrix}
			\mathbf{0}_{3\times1},\cdots, -\mathbf{I}_3,\cdots, -\mathbf{I}_3,\cdots, \alpha_{ij}^\tau, \cdots,\mathbf{0}_{3\times1}
		\end{bmatrix},  \\
		&\alpha_{ij}^\tau = \mathbf{R}_i^* g_{ij}^\tau + \mathbf{R}_j^* g_{ji}^\tau, \ \ \beta_{ij}^\tau = \mathbf{R}_{j}^* t_{j_\tau} + \mathbf{R}_{i}^* t_{i_\tau}. 
	\end{align*}
	The $\alpha_{ij}^\tau$ is in the column corresponding to the $d^\tau_{ij}$. Then the optimal $x^*$ minimizes the sum of the norm of errors and can be obtained in closed form: 
	\begin{align*}
		x^* &= \argmin_{x} \sum_{(i,j) \in \mathcal{E}} \sum_{\tau \in J_{ij}} (e^\tau_{ij})\tp \mathbf{W}_{ij}^\tau e^\tau_{ij}  \\
		&= \argmin_{x} (\mathbf{A}x-\beta)\tp \mathbf{W} (\mathbf{A}x-\beta) \\
		&= (\mathbf{A}\tp \mathbf{W} \mathbf{A})^{-1} \mathbf{A}\tp \mathbf{W} \beta. 
	\end{align*} 
	where $\mathbf{W}_{ij}^\tau$ is the weight matrix and 
	\begin{align}\
		&\mathbf{A} = \begin{bmatrix}
			\mathbf{A}_{ij_1}^{\tau} \\
			\vdots \\
			\mathbf{A}_{ij_{\textup{end}}}^{\tau}
		\end{bmatrix}, 
		\beta = \begin{bmatrix}
			\beta_{ij_1}^{\tau} \\
			\vdots \\
			\beta_{ij_{\textup{end}}}^{\tau}
		\end{bmatrix}, 
		\mathbf{W} = \begin{bmatrix}
			\mathbf{W}_{ij_1}^\tau & & \\
			& \ddots & \\
			& & \mathbf{W}_{ij_{\textup{end}}}^{\tau}
		\end{bmatrix}. \notag
	\end{align} 
	
	Once the closed-form solution $x^*$ is computed, we extract the corresponding components as the relative translation $t^*$.

	\subsection{Sliding Window Optimization}
	\label{subsec:slide_window}
	The algorithm we have developed are capable of recovering relative poses using data collected over a period of time. 
	However, a practical challenge arises in the form of accumulated odometry drift, especially pronounced in robots undertaking long-range motion.
	If the entirety of motion data is used for mutual localization, this drift can adversely affect the accuracy of the estimation.
	To effectively tackle this issue, we adopt a strategy of utilizing data only within a specific, limited time interval  $[\tau_s, \tau_e]$ for performing relative pose estimation. 
	This approach ensures that all robots are localized within a common reference frame, which subsequently serves as the foundation for swarm planning.
	By confining the data to this window, we mitigate the impact of odometry drift on the estimation process.
	The sliding window-based method is detailed in Algorithm \ref{alg:estimator}. 
	The algorithm consists of several critical functions: 
	 $\text{TIME\_MATCH}$, which aligns bearings and local odometry through nearest time search; $\text{FORMULATE\_AND\_SOLVE}$ esponsible for problem formulation, convex relaxation, and SDP solving; and $\text{RECOVER\_RELATIVE\_POSE}$, which is tasked with recovering the relative poses from $\mathbf{Z}^*$.
	\begin{algorithm}[t]
		\caption{\label{alg:estimator}Certifiable Mutual Localization over Sliding Window}%算法名字
		%\LinesNumbered %要求显示行号
		\KwIn{Buffers of bearings $\mathcal{B}^b_{ij}, (i,j)\in\mathcal{E}$\\
			Buffers of local odometry $\mathcal{B}^o_i, i \in [1,N]$}
		\KwOut{Robots' poses in a common reference frame $\mathbf{T}^C_i$ at $\tau_e$}
		$\mathcal{D} \leftarrow \varnothing$ \\
		\ForEach{$(i,j) \in \mathcal{E}$}
		{
			\ForEach{$b^\tau_{ij} \in \mathcal{B}^b_{ij}$}
			{
				$b^\tau_{ji}, \mathbf{T}_{i_\tau},\mathbf{T}_{j_\tau} \leftarrow$ TIME\_MATCH($b^\tau_{ij}.time,\mathcal{B}^b_{ji},\mathcal{B}^o_i,\mathcal{B}^o_j$)\\
				$\mathcal{D} \leftarrow \mathcal{D} \cup \{b^\tau_{ij},b^\tau_{ji}, \mathbf{T}_{i_\tau},\mathbf{T}_{j_\tau}\}$\\
			}
		}
		$\mathbf{Z}^* \leftarrow$ FORMULATE\_AND\_SOLVE($\mathcal{D}$) \\
		\If{rank($\mathbf{Z}^*$)=3}
		{
			$\mathbf{\Theta}^* \leftarrow$ RECOVER\_RELATIVE\_POSE($\mathbf{Z}^*$)\\
			\ForEach{$i \in [1,N]$}
			{
				$\mathbf{T}^C_i$ = ($\mathbf{T}_1^*)^{-1} \mathbf{T}_i^* \mathbf{T}_{i_{\tau_s}}^{-1} \mathbf{T}_{i_{\tau_e}}$
			}
			\Return{$\mathbf{T}^W$}
		}
	\end{algorithm}

	\section{Theoretical Analysis of Certificate Swarm Trajectory Planning}
	\label{sec:theoretical} 
	In this section, we conduct theoretical analysis to understand how detection noise and swarm motion affect the optimality of our proposed mutual localization algorithm.
	In Sec.\ref{subsec:optimality}, based on the global optimality condition, we introduce a matrix, refered as \textit{certificate matrix}.
	A special eigenvalue of the certificate matrix, refered as \textit{certificate eigenvalue},  can certify the global optimality of a candidate solution. 
	Besides, another description of the certificate matrix indicates that the certificate eigenvalue can be seen as a metric of swarm motion.
	In Sec.\ref{subsec:robustness}, the motivation behind the dual description is explained: detection noises will destabilize the estimation optimality, while swarm motion with differetn certificate eigenvalue exhibit varying levels of noise resistance.
	In Sec.\ref{subsec:degeneration}, some conditions are provided to identify degeneration cases, where the certificate eigenvalue is zero and the swarm trajectories are entirely susceptible to noise.
	Raising certificate eigenvalue by optimizing swarm trajectories is feasible to avoid degeneration, but it sometimes conflict the requirement of smoothness. 
	To address the issue, in Sec.\ref{subsec:error_bound}, we show that if the magnitude of detection noise is bounded, the the minimum threshold that the certificate eigenvalue must meet to maintain estimation optimality can be derived as a finite and pre-computable value. 
	These conclusions lay the theoretical foundations that allow us to perform swarm planning considering estimation optimality.
	\subsection{The Dual Description of Certificate Matrix}
	\label{subsec:optimality}
	\subsubsection{Certificate Matrix and Certificate Eigenvalue}
	The global optimality condition of our proposed certifiable mutual localization is given as
	\begin{theorem}
		\label{theorem:1}
		If a local minimizer $\mathbf{\Theta_R^*}$ of (\ref{equ:pro}) is given, the corresponding optimal Lagrange multiplier $\mathbf{\Lambda^*}$ can be computed in closed-form, $\mathbf{\Lambda}_i^* = \sum_{j=1}^N \mathbf{M}_{ij} \mathbf{R}^*_j\tp \mathbf{R}_i^*$.
		Furthermore, If $\mathbf{M-\Lambda^*} \succeq 0$, then:
		\begin{enumerate}
			\item There is a zero-duality-gap between the problem (\ref{equ:r3}) and the problem (\ref{equ:pro}).
			\item $\mathbf{\Theta_R^*}$ is a global minimum for the original problem (\ref{equ:pro}).
		\end{enumerate}
	\end{theorem}
	\begin{proof}
		The proof is given in Appendix \ref{appendix_a}.
	\end{proof}
	
	It is seen in Theorem \ref{theorem:1} that we can certify the global optimality of a candidate solution using $\mathbf{M-\Lambda^*}$. 
	To represent $\mathbf{M-\Lambda^*}$ in a compact form, we define
	\begin{align}
		\mathbf{D}_{\mathbf{\Theta_R^*}} & = \text{Diag}(\mathbf{R}^*_1, \mathbf{R}^*_2, \cdots, \mathbf{R}^*_N) \in \mathbb{R}^{3N \times 3N},\\
		\varphi_{ij}^{\tau} & =  \mathbf{R}^*_i \mathbf{R}_{i_\tau} b^\tau_{ij} \in \mathbb{R}^{3},
	\end{align}
	where $\varphi_{ij}^{\tau}$ represents the bearing from robot $i$ to robot $j$ in robots' common frame when relative pose $\mathbf{R}_i^*$ is given.
	Then we define the \textit{certificate matrix} as 
	\begin{align}
		\mathbf{K}  = \mathbf{D}_{\mathbf{\Theta_R^*}} (\mathbf{M-\Lambda^*}) \mathbf{D}_{\mathbf{\Theta_R^*}}^\mathrm{T}. \label{equ:K}
	\end{align}
	Since $\mathbf{D}_{\mathbf{\Theta_R^*}}$ is orthogonal, 	
	$\mathbf{K}$ is semi-definite as long as $\mathbf{M-\Lambda^*}$ is. 
	Thus, when a candidate solution $\mathbf{\Theta_R^*}$ is given, we can certify its global optimality by checking whether $\mathbf{K} \succeq 0$.
	Expanding (\ref{equ:K}) gives:
	\begin{gather}
		\mathbf{K} = 
		\begin{bmatrix}
			- \sum\limits_{j\neq1} \mathbf{K}_{1j}  & \mathbf{K}_{12} & \cdots & \mathbf{K}_{1N} \\
			\mathbf{K}_{21} & -\sum\limits_{j\neq2} \mathbf{K}_{2j}   & \cdots & \mathbf{K}_{2N} \\
			\vdots & \vdots & \ddots & \vdots \\
			\mathbf{K}_{N1} & \mathbf{K}_{N2} & \cdots & - \sum\limits_{j\neq N}\mathbf{K}_{Nj}
		\end{bmatrix} \label{equ:reK} \\
		\mathbf{K}_{ij} = \mathbf{R}^*_i \mathbf{M}_{ij} \mathbf{R}^*_j\tp = 
		\begin{cases}
			\sum\limits_{\tau \in J_{ij}} \varphi_{ij}^\tau \varphi_{ji}^\tau\tp &(i,j)\in \mathcal{E}\\
			\mathbf{0}_{3\times 3} &(i,j)\notin \mathcal{E}
		\end{cases} \label{equ:reK1}
	\end{gather}
	
	The structure of $\mathbf{K}$ columns that the columns of $\mathbf{N} = \mathbf{1}_N \otimes \mathbf{I}_3$ lie in the nullspace of $\mathbf{K}$.
	It implies  $\mathbf{K}$ invariably possesses three zero eigenvalues, with the columns of  $\mathbf{N}$ acting as the corresponding eigenvectors.
	Therefore, for a sufficienly large value $\mu$, the criterion  $\lambda_1(\mathbf{K} + \mu\mathbf{NN\tp}) \geq 0$ can be used to ascertain the positive semi-definiteness of $\mathbf{K}$, where $\lambda_1(\star)$ denotes the smallest eigenvalue.
	We refer to this value as the \textit{certificate eigenvalue}.
	
	\subsubsection{Swarm Trajectory Metric for Certifiable Estimation}
	\label{subsubsec:metric}
	
	\begin{figure*}[!t]
		\centering
		\includegraphics[width=1.0\textwidth]{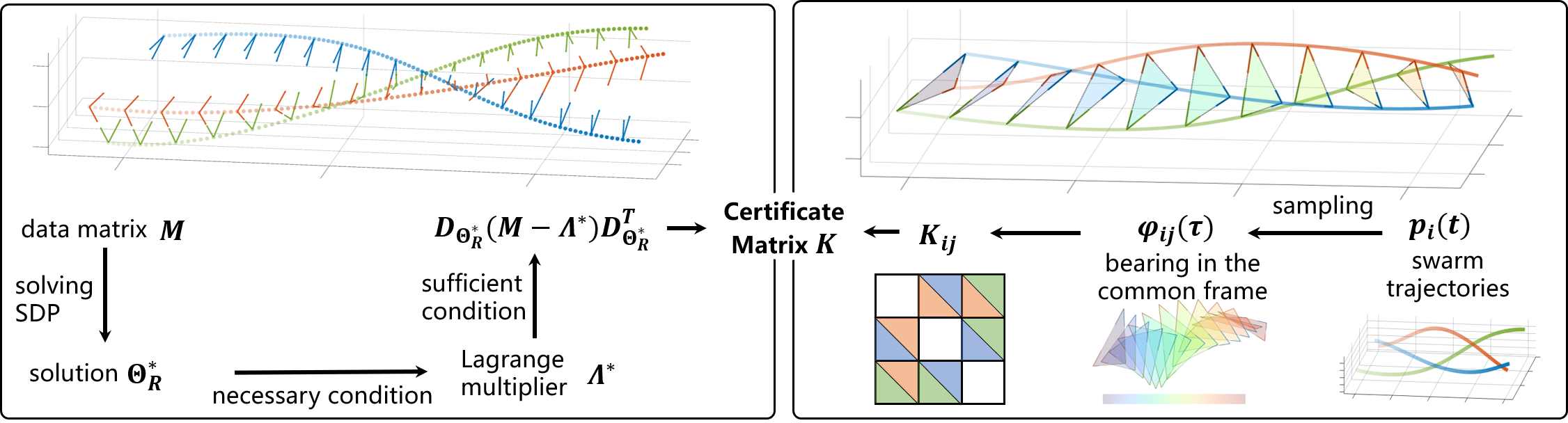}
		\caption{\label{fig:dual_description} Demonstration of the dual description of the certificate matix $\mathbf{K}$. On the left, discrete poses and bearing measurements represent the estimation process, with $\mathbf{K}$ being constructed from the actual measurement data.
			On the right, continuous swarm trajectories and triangle sections symbolize the swarm motion and sampling process, respectively, where $\mathbf{K}$ is formulated by sampling along  these trajectories.}
		\vspace{-0.4cm}
	\end{figure*}

	Observe (\ref{equ:reK}) and (\ref{equ:reK1}), another representation of $\mathbf{K}$ can be obtained as
	\begin{equation} \label{equ:reK2}
		\begin{aligned}
			&\mathbf{K} = \sum_{(i,j) \in \mathcal{E}} \mathbf{K}_{(i,j)}, \
			\mathbf{K}_{(i,j)} = \sum\limits_{\tau \in J_{ij}} \mathbf{K}_{(i,j)}^\tau, \\
			\mathbf{K}_{(i,j)}^\tau & = 
			\begin{bmatrix}
				\ddots & & & & \\
				& -\varphi_{ij}^\tau \varphi_{ji}^\tau \tp & \cdots & \varphi_{ij}^\tau  \varphi_{ji}^\tau \tp & \\
				& \vdots & \ddots & \vdots & \\
				& \varphi_{ji}^\tau  \varphi_{ij}^\tau \tp & \cdots & -\varphi_{ji}^\tau  \varphi_{ij}^\tau \tp & \\
				& & & & \ddots 
			\end{bmatrix}.
		\end{aligned}
	\end{equation}
	The representation implies that $\mathbf{K}$ can be formulated using bearings $\varphi_{ij}^\tau$.
	It allows us to recognize the certificate matrix from the perspective of swarm motion.
	Consider the following scenario: relative poses of robots have been known as $\mathbf{\hat{\Theta}_R}$, and each robot (e.g., $i$) moves along a trajectory $p_i(t)$ in robots' common reference frame. 
	Then, the groud-truth bearing $\hat{\varphi}_{ij}^\tau$ can be obtained by sampling robots' positions:
	\begin{align}
		\hat{\varphi}_{ij}^\tau = \frac{p_j(\tau) - p_i(\tau)}{\norm{p_j(\tau) - p_i(\tau)}} \label{equ:gt_bearing}
	\end{align}
	Next, a special certificate matrix $\hat{\mathbf{K}}$ can be constructed using the groud-truth bearing $\hat{\varphi}_{ij}^\tau$ as:
	\begin{equation} \label{equ:reK2}
		\begin{aligned}
			&\hat{\mathbf{K}} = \sum_{(i,j) \in \mathcal{E}} \hat{\mathbf{K}}_{(i,j)}, \
			\hat{\mathbf{K}}_{(i,j)} = \sum\limits_{\tau \in J_{ij}} \hat{\mathbf{K}}_{(i,j)}^\tau, \\
			\hat{\mathbf{K}}_{(i,j)}^\tau & = 
			\begin{bmatrix}
				\ddots & & & & \\
				& -	\hat{\varphi}_{ij}^\tau 	\hat{\varphi}_{ji}^\tau \tp & \cdots & 	\hat{\varphi}_{ij}^\tau  	\hat{\varphi}_{ji}^\tau \tp & \\
				& \vdots & \ddots & \vdots & \\
				& 	\hat{\varphi}_{ji}^\tau  	\hat{\varphi}_{ij}^\tau \tp & \cdots & -	\hat{\varphi}_{ji}^\tau  	\hat{\varphi}_{ij}^\tau \tp & \\
				& & & & \ddots 
			\end{bmatrix}.
		\end{aligned}
	\end{equation}

	In this context, assume robots obtain local bearing $\hat{b}^\tau_{ij}$ during motion along  $\{p_i(t)\}_{i\in[1,N]}$, then $\mathbf{\hat{\Theta}_R}$ can be recovered by solving (\ref{equ:r3}) using these local bearings as soon as $\hat{\mathbf{K}} \succeq 0$, according to Theorem \ref{theorem:1}.
	
	Since the construction of $\hat{\mathbf{K}}$ solely depends on the swarm trajectories $\{p_i(t)\}_{i\in[1,N]}$ in the common reference frame, which allows it to serve as a metric for these trajectories. 
	This perspective, along with the original definition, constitutes the dual description of the certificate matrix, as depicted in Fig.\ref{fig:dual_description}.
	The practical siginicance of this dual description emerges from the following observations:
	In noise-free cases, $\mathbf{K}$ is equivalent to $\hat{\mathbf{K}}$.
	However, when considering the presence of detection noise in the actual bearing measurements $\varphi_{ij}^\tau$, $\mathbf{K}$ can be interpreted as a matrix derived by adding perturbations to $\hat{\mathbf{K}}$.
	While  $\hat{\mathbf{K}}$ maintains positive semi-definiteness,the addition of noise might compromise the positive semi-definiteness of $\mathbf{K}$, a phenomenon that will be analyzed subsequently.
	
	\subsection{Noise Analysis of Certificate Matrix}
	\label{subsec:robustness} 

	In this section, we explore the impact of detection noise on the positive semi-definiteness of the certificate matrix $\mathbf{K}$.
	Due to the complexity inherent in directly analyzing $\mathbf{K}$, our focus shifts to its minimum component $\mathbf{K}_{(i,j)}^\tau$, as detailed in (\ref{equ:reK2}).
	To simplify the analysis, we remove the time stamp $\tau$ and denote $\mathbf{K}_s = \mathbf{K}_{(i,j)}^\tau$. 
	Then, we define $\mathbf{P}_{ij} = [e_i, e_j]\tp \otimes \mathbf{I}_3$ and proceed to extract the non-zero components of $\mathbf{K}_s$.
	This extraction allows us to define $\mathbf{Q}_s$ as:
	\begin{align}
		\mathbf{Q}_s = \mathbf{P}_{ij} \mathbf{K}_s \mathbf{P}_{ij}^\mathrm{T} = 
		\begin{bmatrix}
			-\varphi_{ij} \varphi_{ji}^\mathrm{T}  & \varphi_{ij} \varphi_{ji}^\mathrm{T} \\
			\varphi_{ji} \varphi_{ij}^\mathrm{T}   & -\varphi_{ji} \varphi_{ij}^\mathrm{T}
		\end{bmatrix} \in \mathbb{R}^{6\times6}. \label{equ:Ks}
	\end{align}
	$\mathbf{Q}_s$ has the same non-zero eigenvalues as $\mathbf{K}_s$.
	
	To distinguish noise-free and noisy cases, we describe the ground-truth and noisy bearings:
	\begin{align}
		 \varphi_{ij} = \hat{\varphi}_{ij} + \delta_{ij}, \ \varphi_{ji} = \hat{\varphi}_{ji} + \delta_{ji}, \label{equ:noise_bearing}
	\end{align}
	where $\hat{\varphi}$, $\varphi$ and $\delta$ represent the ground-truth, noisy bearing and detection noise, recpectively. 
	Denote the norm of detection noise as $\xi_{ij}=\norm{\delta_{ij}}$ and $\xi_{ji}=\norm{\delta_{ji}}$. 
	Then, maxtrces which are constructed with ground-truth bearings are denoted as $\mathbf{\hat{Q}}_s$, $\mathbf{\hat{K}}_s$ and $\mathbf{\hat{K}}$.
	The curruption on them is denoted as
	\begin{align}
		\Delta \mathbf{Q}_s & =  \mathbf{Q}_s - \mathbf{\hat{Q}}_s \label{equ:deltaQ} \\
		\Delta \mathbf{K}_s & = \mathbf{K}_s - \mathbf{\hat{K}}_s = \mathbf{P}_s^\mathrm{T} \Delta \mathbf{Q}_s \mathbf{P}_s, \label{equ:deltaKs} \\
		\Delta \mathbf{K} & = \mathbf{K} - \mathbf{\hat{K}} = \sum^{\mathcal{L}}  \Delta \mathbf{K}_s. \label{equ:deltaK}
	\end{align}
	where $\mathcal{L}$ is the total number of data.
	As $\hat{\varphi}_{ij} + \hat{\varphi}_{ji}=0$, bring $ \hat{\varphi}_{ji} = -\hat{\varphi}_{ij}$ to (\ref{equ:Ks}) gives:
	\begin{align}
		\hat{\mathbf{Q}}_s = 
		\begin{bmatrix}
			\hat{\varphi}_{ij} \hat{\varphi}_{ij}^\mathrm{T}  & -\hat{\varphi}_{ij} \hat{\varphi}_{ij}^\mathrm{T} \\
			-\hat{\varphi}_{ij} \hat{\varphi}_{ij}^\mathrm{T}   & \hat{\varphi}_{ij} \hat{\varphi}_{ij}^\mathrm{T}
		\end{bmatrix} = u_{ij} u_{ij}^\mathrm{T}
	\end{align}
	where $u_{ij} = [\hat{\varphi}_{ij}^\mathrm{T}, -\hat{\varphi}_{ij}^\mathrm{T}]\tp$.
	Thus, we obtain that $\mathbf{\hat{K}}$ is a Gram matrix:
	\begin{gather}
		\mathbf{\hat{K}} = \sum^{\mathcal{L}} \mathbf{\hat{K}}_s = \sum^{\mathcal{L}} v_s v_s^\mathrm{T} \\
		v_s = \left[  \cdots\   \mathbf{0}\  \hat{\varphi}_s^\mathrm{T}\  \mathbf{0}\  \cdots \ \mathbf{0}\  -\hat{\varphi}_s^\mathrm{T}\  \mathbf{0}\  \cdots\  \right]\tp \in \mathbb{R}^{3N} \label{equ:gtK}
	\end{gather}
	with $\hat{\varphi}_s$ and $-\hat{\varphi}_s$ occupying blocks corresponding to involved robots in bearing $\hat{\varphi}_s$.
	It suggests that  $\mathbf{\hat{K}}$ is always semi-definite. 
	Since $\mathbf{K}$ always has three zero eigenvalues, it is obvious that $\lambda_1(\hat{\mathbf{K}} + \mu \mathbf{NN\tp}) = \lambda_4(\hat{\mathbf{K}})$.
	
	In general noisy cases, we construct $\mathbf{\hat{Q}}_s$ and $\mathbf{Q}_s$ using the ground-truth and noisy bearings respectively and expand the curruption $\Delta \mathbf{Q}_s$ in (\ref{equ:deltaQ}) to obtain
	\begin{gather}
		\Delta \mathbf{Q}_s =
		\begin{bmatrix}
			 -\mathbf{\Delta}_{ij} & \mathbf{\Delta}_{ij} \\
			 \mathbf{\Delta}_{ji} & -\mathbf{\Delta}_{ji}
		\end{bmatrix}, \label{equ:deltaQ1}
	\end{gather}
	where
	\begin{gather}
		\mathbf{\Delta}_{ij} =\hat{\varphi}_{ij} \delta_{ji}^\mathrm{T} + \delta_{ij} \varphi_{ji}^\mathrm{T}, \\
		\mathbf{\Delta}_{ji} = \delta_{ji} \hat{\varphi}_{ij}^\mathrm{T} + \varphi_{ji} \delta_{ij}^\mathrm{T}
	\end{gather}

	Next, another formulation of the perturbation $\Delta \mathbf{Q}_s$ is provide as below.
	\begin{theorem}
		\label{theorem:2}
		Define $\mathring{\delta}_{ij} = \frac{\delta_{ij}}{\xi_{ij}}$, $\mathring{\delta}_{ji} = \frac{\delta_{ji}}{\xi_{ji}}$, 
		$\phi_{ij} = \hat{\varphi}_{ij} \times \delta_{ji}$ and $\phi_{ji} = \delta_{ij} \times \varphi_{ji}$, then $\Delta \mathbf{Q}_s$ can be written as
		\begin{gather}
			\Delta \mathbf{Q}_s = \mathbf{V}_1 \mathbf{U}_1 \mathbf{V}_1^\mathrm{T} + \mathbf{V}_2 \mathbf{U}_2 \mathbf{V}_2^\mathrm{T} \\
			\mathbf{V}_1 = 
			\left[\begin{array}{c|c|c|c|c|c}
				\hat{\varphi}_{ij}      & \hat{\varphi}_{ij}     &  \phi_{ij}   &  e_1 & e_2 & e_3 \\
				-\mathring{\delta}_{ji} & \mathring{\delta}_{ji} &  -\phi_{ij}  &  e_1 & e_2 & e_3 \\
			\end{array}\right], \\
			\mathbf{V}_2 = 
			\left[\begin{array}{c|c|c|c|c|c}
				-\mathring{\delta}_{ij} & \mathring{\delta}_{ij} &  \phi_{ji}   &  e_1 & e_2 & e_3 \\
				\varphi_{ji}            & \varphi_{ji}           &  -\phi_{ji}  &  e_1 & e_2 & e_3 \\
			\end{array}\right], \\
			\mathbf{U}_1 = \textup{diag}([-\xi_{ji} - \delta_{ji}^\mathrm{T} \hat{\varphi}_{ij} , \xi_{ji} - \delta_{ji}^\mathrm{T} \hat{\varphi}_{ij}, \mathbf{0}_{1\times4}]) \\
			\mathbf{U}_2 = \textup{diag}([-\xi_{ij} - \delta_{ij}^\mathrm{T} \varphi_{ji}, \xi_{ij} - \delta_{ij}^\mathrm{T} \varphi_{ji}, \mathbf{0}_{1\times4}])
		\end{gather}
	\end{theorem}
	\begin{proof}
		The proof is given in Appendix \ref{appendix_b}.
	\end{proof}
	\begin{figure*}[!t]
		\centering
		\includegraphics[width=1.0\textwidth]{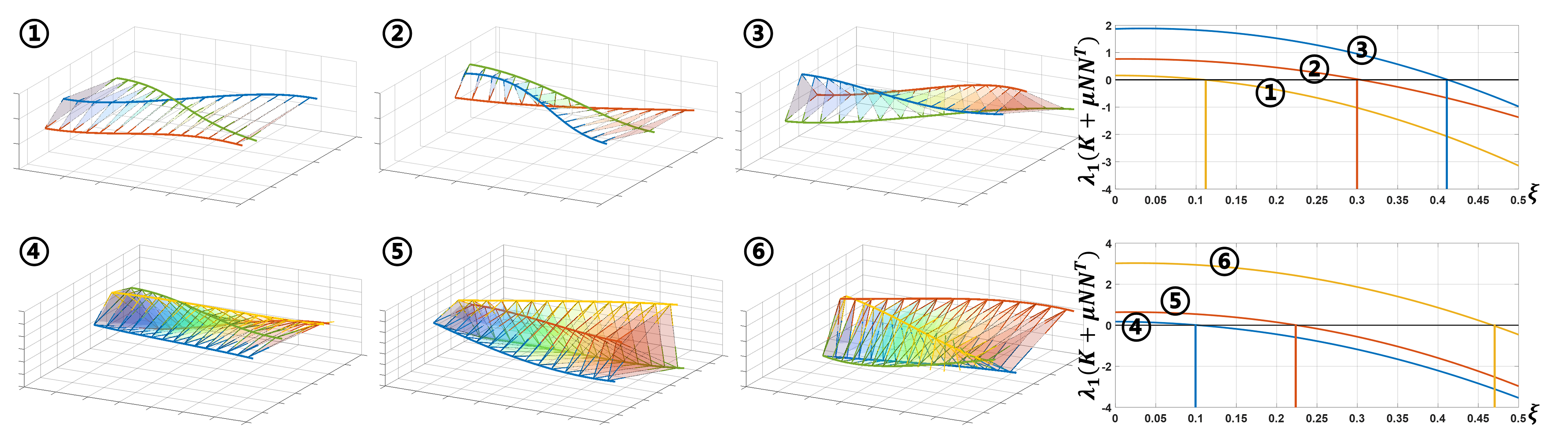}
		\caption{\label{fig:34} Demonstration of the impact of noise on estimation optimality. 
			Figures $\normalsize{\textcircled{\scriptsize{1}}}\normalsize$-$\normalsize{\textcircled{\scriptsize{3}}}\normalsize$ illustrate  different swarm trajectories for three robots, while figures $\normalsize{\textcircled{\scriptsize{4}}}\normalsize$-$\normalsize{\textcircled{\scriptsize{6}}}\normalsize$ illustrate scenarios involving four robots. 
			The corresponding right figures show the variation of the certificate eigenvalue under different levels of noise.}
		\vspace{-0.1cm}
	\end{figure*}

	As $-\xi_{ji} - \delta_{ji}^\mathrm{T} \hat{\varphi}_{ij} <0$ and $-\xi_{ij} - \delta_{ij}^\mathrm{T} \varphi_{ji} < 0$, the theorem shows that the presence of detection noise is shown to introduce negative eigenvalues, potentially compromising the positive semi-definiteness of $\mathbf{K}$.
	Consider a specific scenario: if $\delta_{ji}$ is collinear with  $\hat{\varphi}_{ij}$ and $\delta_{ij}$ is collinear with  $\varphi_{ji}$, then $\Delta \mathbf{Q}_s$ will not possess any positive eigenvalue. 
	If this scenario applies to every pair of bearings, then $\Delta \mathbf{K}$ becomes a negative semi-definite matrix, and it is clear that $\lambda_1(\mathbf{K}) < \lambda_1(\mathbf{\hat{K}})$.
	Then, if $\lambda_{1}(\mathbf{\Delta K}) < -\lambda_{\textup{max}}(\mathbf{\hat{K}})$ under certain  detection noises, $\mathbf{K}$ will lose its positive semi-definiteness, as $\lambda_1(\mathbf{K}) \leq \lambda_{\textup{max}}(\mathbf{\hat{K}}) + \lambda_1(\mathbf{\Delta K}) < 0$.
%	\begin{align*}
%		\lambda_1(\mathbf{K}) \leq \lambda_{\textup{max}}(\mathbf{\hat{K}}) + \lambda_1(\mathbf{\Delta K}) < 0.
%	\end{align*}
%	As $\xi$ changes, if the minimum eigenvalue of $\Delta \mathbf{K}$, $\lambda_{1}(\mathbf{\Delta K})$, is negtive and $|\lambda_{1}(\mathbf{\Delta K})| > \lambda_{\textup{max}}(\mathbf{\hat{K}})$, $\mathbf{K}$ is not semi-definite since
%	\begin{align*}
%		\lambda_1({\mathbf{K}}) \le \lambda_{\textup{max}}(\mathbf{\hat{K}}) + \lambda_{1}(\mathbf{\Delta K}) < 0.
%	\end{align*}
	
	Except the aforementioned special case, numerical simulations are conducted to further validate the impact of detection noise on the positive semi-definiteness of the certificate matrix $\mathbf{K}$.
	In these simulations, robots follow various simulated trajectories and generate noisy bearings. 
	The detection noise magnitude is uniform, denoted as $\xi = \xi_{ij} = \xi_{ji}$ for all bearing pairs.
	Using these noisy bearings,  $\mathbf{K}$ is constructed, and we observe changes in $\lambda_1(\mathbf{K} + \mu \mathbf{NN\tp})$ as $\xi$ increases.
	Select results are presented in Fig.\ref{fig:34}, illustrating the variation of $\lambda_1(\mathbf{K} + \mu \mathbf{NN\tp})$ for specific swarm trajectories involving 3 and 4 robots, respectively.
	The results show two key findings.
	Firstly, when $\xi$ surpasses a certain threshold, the positive semi-definiteness of $\mathbf{K}$ is indeed compromised, aligning with our prior analysis. 
	Secondly, swarm trajectories characterized by a larger $\lambda_1(\hat{\mathbf{K}} + \mu \mathbf{NN\tp})$ -the certificate eigenvalue when $\xi = 0$-exhibit greater resistance to noise. 
	This observation ties back to the discussion of dual descriptions in Sec.\ref{subsubsec:metric},
	confirming that $\lambda_1(\hat{\mathbf{K}} + \mu \mathbf{NN\tp})$, or equivalently  $\lambda_4(\hat{\mathbf{K}})$, can effectively function as a metric for evaluating a swarm's resistance to noise.
	
	To further our understanding of $\lambda_4(\hat{\mathbf{K}})$, we delve into two related problems.
	First, we investigate the conditions under which  $\lambda_4(\hat{\mathbf{K}}) = 0$, signifying that the swarm trajectories are entirely susceptible to noise.
	We refer to these scenarios as \textit{degeneration} and identify these cases in Sec.\ref{subsec:degeneration}. 
	Secondly, in scenarios where the magnitude noise is known to be bounded, we explore the minimum threshold that  $\lambda_4(\hat{\mathbf{K}})$ must meet to ensure the positive semi-definiteness of the disturbed certificate matrix $\mathbf{K}$. 
	Our objective is to explicitly characterize the noise resistance of swarm motion, which aids in planning swarm trajectories that ensure $\mathbf{K} \succeq 0$.
	The threshold, termed the \textit{certificate eigenvalue bound} is derived and elaborated in Sec.\ref{subsec:error_bound}.
%	\begin{lemma}
%		\label{lemma:0}
%		Weyl's inequality:
%		\begin{gather}
%			\lambda_1(\mathbf{A}) + \lambda_1(\mathbf{B}) \leq \lambda_1(\mathbf{A+B}) \leq \lambda_{max}(\mathbf{A}) + \lambda_1(\mathbf{B}) 
%		\end{gather}
%	\end{lemma}
%	
	
	\subsection{Degeneration Identification}
	\label{subsec:degeneration}
	In this section, what kind of swarm motion will lead $\lambda_1(\hat{\mathbf{K}} + \mu \mathbf{NN\tp}) = 0$ is presented.
	Define $\mathbf{G} = \hat{\mathbf{K}} + \mu \mathbf{NN\tp}$.
	Recall (\ref{equ:gtK}), it be rewritten as $\mathbf{G} = \mathbf{J}\tp \mathbf{J}$ where
	\begin{align}
		\mathbf{J} = \left[ \sqrt{\mu} \mathbf{N}\tp,\ v_1^\mathrm{T},\ \cdots,\ v_\mathcal{L}^\mathrm{T} \right]\tp \in \mathbb{R}^{(3+\mathcal{L})\times 3N}.
	\end{align}
	Since $\lambda_1(\mathbf{G}) = 0 \Leftrightarrow \textup{det}(\mathbf{G}) = 0$, we study the determinant of $\mathbf{G}$ instead of its eigenvalues.
	Applying the Cauchy-Binet formula (\cite{broida1989comprehensive}) on $\mathbf{G}$ gives:
	\begin{align}
		\textup{det}(\mathbf{G}) = \sum_{S\in {[3+\mathcal{L}] \choose 3N}}\ \textup{det}(\mathbf{J}_{[3N],S})^2.
	\end{align}
	Here we write $[n]$ for the set $[n]= \{1,\cdots,n\}$ and $[n] \choose m$ for the set of $m$-combinations of $[n]$.
	$\mathbf{J}_{[m],S}$ is a $m \times m$ matrix whose rows  are the rows  of $\mathbf{J}$ at indices from $S$.
	It is obvious that 
	\begin{align}
		\textup{det}(\mathbf{G}) = 0 \Leftrightarrow \textup{det}(\mathbf{J}_{[3N],S}) = 0, \forall S\in {[3+\mathcal{L}] \choose 3N}. \label{equ:G}
	\end{align}
	Thus, we study under what conditions $\textup{det}(\mathbf{J}_{[3N],S}) = 0$,  for different number of robots $N$.
	\subsubsection{N = 2}
	\label{subsec:N_2}
	According to how many rows of $\sqrt{\mu} \mathbf{N}^\tp$ are selected in $\mathbf{J}_{[6],S}$, there are four types of matrix $\mathbf{J}_{[6],S}^{(0)}, \mathbf{J}_{[6],S}^{(1)}, \mathbf{J}_{[6],S}^{(2)}$ and $\mathbf{J}_{[6],S}^{(3)}$.
	Examples of $\mathbf{J}_{[6],S}^{(2)}$ and $\mathbf{J}_{[6],S}^{(3)}$ are written as 
	\begin{align*}
		\mathbf{J}_{[6],S}^{(2)} = 
		\begin{bmatrix}
			\sqrt{\mu} e_1 & \sqrt{\mu} e_1 \\
			\sqrt{\mu} e_2 & \sqrt{\mu} e_2 \\
			\mathbf{S}_4 & -\mathbf{S}_4
		\end{bmatrix},
		\mathbf{S}_4=
		\begin{bmatrix}
			\varphi^{\tau_i} \\
			\varphi^{\tau_j} \\
			\varphi^{\tau_k} \\
			\varphi^{\tau_m}
		\end{bmatrix},\\
		\mathbf{J}_{[6],S}^{(3)} = 
		\begin{bmatrix}
			\sqrt{\mu}\mathbf{I}_3 & \sqrt{\mu}\mathbf{I}_3 \\
			\mathbf{S}_3 & -\mathbf{S}_3
		\end{bmatrix},
		\mathbf{S}_3=
		\begin{bmatrix}
			\varphi^{\tau_i} \\
			\varphi^{\tau_j} \\
			\varphi^{\tau_k}
		\end{bmatrix}.
	\end{align*}
	where $\varphi^{\tau_i}$ denotes inter-robot bearing in the common frame at $\tau_i$.
	Then we have the following conclusion.
	\begin{theorem}
		\label{theorem:3}
		(i) $\textup{det}(\mathbf{J}_{[3N],S}^{(d)}) = 0$ for $d = 0,1$ and $2$.
		(ii) For $d = 3$, if and only if $\varphi^{\tau_i}, \varphi^{\tau_j}$ and $\varphi^{\tau_k}$ are coplanar, $\textup{det}(\mathbf{J}_{[3N],S}^{(3)}) = 0$.
	\end{theorem}
	\begin{proof}
		The proof is given in Appendix \ref{appendix_c}.
	\end{proof}
	The theorem implies that if and only if all inter-robot bearings are coplanar, i.e., two robots move coplanarly, $\textup{det}(\mathbf{J}_{[3N],S}^{(3)}) = 0$ for any $\mathbf{J}_{[3N],S}^{(3)}$. 
	Recall (\ref{equ:G}), we can identify degeneration cases as below.
	\begin{corollary} 
		For a two-robot systems, if and only if two robots move coplanarly, $\textup{det}(\mathbf{G}) = 0$.
	\end{corollary}
	
	\subsubsection{N $\geq$ 3}
	\label{subsec:N_3}
	Extending above analysis of $N=2$ case to $N\geq3$ cases, it's straightforward to conclude that all robots moving coplanarly will lead to degeneration.
	Here, we provide a more generalized identification condition for $\textup{det}(\mathbf{J}_{[3N],S}) = 0, \forall S\in {[3+\mathcal{L}] \choose 3N}$.
%	Characterizing a necessary and efficient condition of degeneration for $N \geq 3$ is intractable, here we provide a necessary condition for $\textup{det}(\mathbf{J}_{[3N],S}) = 0, \forall S\in {[3+\mathcal{L}] \choose 3N}$.
	
%	Above analysis suggests that robots' coplanar motion play an important role in degeneration identification.
%	Depart $\mathcal{G}$ into two disjoint subgraphs $\mathcal{G}_c = (\mathcal{V}_c, \mathcal{E}_c)$ and $\mathcal{G}_r = (\mathcal{V}_r, \mathcal{E}_r)$, such that $\mathcal{V}_c \cup \mathcal{V}_r = \mathcal{V}, \mathcal{V}_c \cap \mathcal{V}_r =  \emptyset$, $\vert \mathcal{V}_c  \vert = s$.
%	A necessary condition for $\textup{det}(\mathbf{J}_{[3N],S}) = 0, \forall S\in {[3+\mathcal{L}] \choose 3N}$ is provided.
	\begin{theorem}
		\label{theorem:4}
		Depart $\mathcal{G}$ into two disjoint subgraphs $\mathcal{G}_c = (\mathcal{V}_c, \mathcal{E}_c)$ and $\mathcal{G}_r = (\mathcal{V}_r, \mathcal{E}_r)$, such that $\mathcal{V}_c \cup \mathcal{V}_r = \mathcal{V}, \mathcal{V}_c \cap \mathcal{V}_r =  \emptyset$.
		Assume that robots in $\mathcal{V}_c$ move in coplanar, and at least one robot $i \in \mathcal{V}_c$, whose neighboring nodes in $\mathcal{G}$ satisfy that $\mathcal{N}_i \cup \mathcal{V}_r = \emptyset$, then $\textup{det}(\mathbf{J}_{[3N],S}) = 0, \forall S\in {[3+\mathcal{L}] \choose 3N}$.
	\end{theorem}
	\begin{proof}
		The proof is given in Appendix \ref{appendix_d}.
	\end{proof}
	An example of degeneration case for $N = 5, |\mathcal{V}_c|=3$ is shown in Fig. \ref{fig:degeneration}.
	Actually, existing work (\cite{nguyen2023relative}) analysis under what condition the determinant of FIM equals zero in range-based swarm.
	It provides similar conclusion that two robots' coplanar motion will leads to singular configuration.
	In this paper, we study $\textup{det}(\mathbf{G}) = 0$ for bearing-based swarm and extend analysis to $N \geq 3$ cases.
	More importantly, our results show the significance of graph stucture: even though some robots don't move coplanarly, degeneration is also possible under some graph.
	
	\begin{figure}[t]
		\centering
		\includegraphics[width=0.5\textwidth]{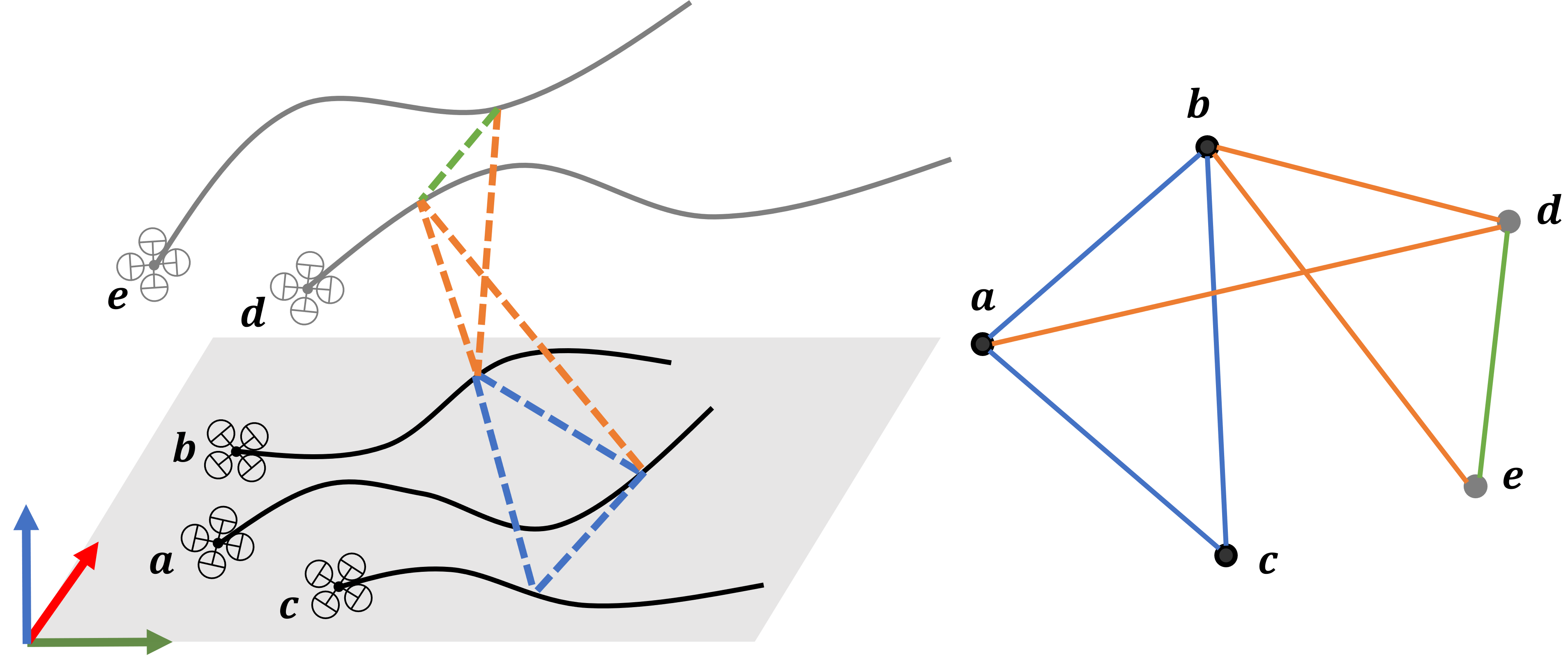}
		\caption{\label{fig:degeneration} An example of degeneration case for $N=5$, where 
			$\mathcal{V}_c = \{a,b,c\}$, $\mathcal{V}_r = \{d,e\}$, and $\mathcal{N}_c \cup \mathcal{V}_r = \emptyset$.
			Although robots $d,e$ move randomly, $\lambda_1(\hat{\mathbf{K}} + \mu \mathbf{NN\tp}) = 0$ for this swarm.}
		\vspace{-0.1cm}
	\end{figure}

	\subsection{Certificate Eigenvalue Bound}
	\label{subsec:error_bound}
	In Sec.\ref{subsec:degeneration}, our analysis on degeneration identification reveals that merely avoiding coplanar motion among robots is insufficient to prevent degeneration.
	Further, as detailed in Sec.\ref{subsec:robustness}, we demonstrate that $\lambda_4(\hat{\mathbf{K}})$ effectively serves as a metric for assessing noise resistance. 
	Consequently, optimizing swarm trajectories to enhance $\lambda_4(\hat{\mathbf{K}})$ emerges as a logical strategy to both avert degeneration and improve noise resistance. 
	However, an intriguing finding is that excessively maximizing $\lambda_4(\hat{\mathbf{K}})$ can result in highly distorted swarm trajectories, which are impractical for effective tracking and control in the real world.

	To address this issue, this section delves into determining the minimum threshold that $\lambda_4(\hat{\mathbf{K}})$ must achieve to maintain
	$\mathbf{K}$'s positive semi-definiteness, particularly when the maximum magnitude of detection noise, denoted as $\xi_{\text{max}}$, is known.
	We refer to this threshold as the \textit{certificate eigenvalue bound} $\mathcal{B}$.
	This bound stems from a practical consideration: detection noise is typically constrained and can be quantified through sensor calibration.
	Then, we can pre-calculate $\mathcal{B}$ using the evaluated maximum noise magnitude  $\xi_{\text{max}}$.
	Ensuring that the swarm trajectories with respect to $\mathcal{B}$ are tracked by robots, the preservation of estimation optimality can be ensured  even in the presence of noise, epitomizing the concept of certifiable swarm planning
	
	The derivation of $\mathcal{B}$ comes from the following inequality
	\begin{equation} \label{equ:error_bound1}
		\begin{aligned}
			\lambda_1(\mathbf{K} + \mu \mathbf{NN\tp})
			&=\lambda_1(\mathbf{\hat{K}} + \Delta \mathbf{K} + \mu \mathbf{NN\tp}) \\
			&\geq\lambda_1(\mathbf{\hat{K}} + \mu \mathbf{NN\tp}) + \lambda_1(\Delta \mathbf{K}) \\
			&\geq \lambda_1(\mathbf{\hat{K}} + \mu \mathbf{N} \mathbf{N}\tp) - \max |\lambda(\Delta \mathbf{K})| \\
			&= \lambda_4(\mathbf{\hat{K}}) - \max |\lambda(\Delta \mathbf{K})|. 
		\end{aligned}
	\end{equation}
	It is obvious that if $\lambda_4(\mathbf{\hat{K}}) \geq \max |\lambda(\Delta \mathbf{K})|$, $\mathbf{K}$ can be ensured to remain positive semi-definite.	
	For $\Delta \mathbf{K}$, bringing (\ref{equ:deltaQ1}) into (\ref{equ:deltaKs}) and (\ref{equ:deltaK}) gives
	\begin{equation}
		\begin{aligned}
			\Delta \mathbf{K}_{ij} &= 
			\begin{cases}
				\sum^{\mathcal{T}} (\hat{\varphi}_{ij} \delta_{ji}^\mathrm{T} + \delta_{ij} \varphi_{ji}^\mathrm{T} )  &(i,j)\in \mathcal{E}\\
				\mathbf{0}_{3\times 3} &(i,j)\notin \mathcal{E}
			\end{cases} \\
			\Delta \mathbf{K}_{ii} &= -\sum_{k\neq i}^N \Delta \mathbf{K}_{ik} 
		\end{aligned}
	\end{equation}
	where $\mathcal{T}$ is the sampling number.
	Then, the following theorem characterizes $|\lambda(\Delta \mathbf{K})|$ given $\xi_{\text{max}}$.
	\begin{theorem}
		\label{theorem:5}
		For a graph  $\mathbf{\mathcal{G}}  = (\mathbf{\mathcal{V}} , \mathbf{\mathcal{E}})$, if the detection noise is limited as $\xi_{\text{max}}$, then $|\lambda(\Delta \mathbf{K})| < 2d_{\textup{max}}\mathcal{T}\sqrt{2\xi_{\text{max}}^2 + \xi_{\text{max}}^3}$, where $d_{\textup{max}}$ is the maximal vertex degree in $\mathbf{\mathcal{G}}$.
	\end{theorem}
	\begin{proof}
		The proof is given in Appendix \ref{appendix_e}
	\end{proof}
	Therefore, we can take $2d_{\textup{max}}\mathcal{T}\sqrt{2\xi_{\text{max}}^2 + \xi_{\text{max}}^3}$ as the certificate eigenvalue bound $\mathcal{B}$.
	As long as $\lambda_4(\mathbf{\hat{K}}) \geq 2d_{\textup{max}}\mathcal{T}\sqrt{2\xi_{\text{max}}^2 + \xi_{\text{max}}^3}$, we can ensure that $\lambda_1(\mathbf{K} + \mu \mathbf{NN\tp}) > 0$.
	The bound is related to the graph structure.
	Three important graphs, complete graph, star graph and cycle graph, are shown in Fig.\ref{fig:graph}.
	Their $d_{\textup{max}}$ are $N-1$, $N-1$ and 2, respectively.
	Hence, $\mathcal{B} = 2 (N-1) \mathcal{T}\sqrt{2\xi_{\text{max}}^2 + \xi_{\text{max}}^3}$  for complete graph and star graph , and 
	$\mathcal{B} = 4 \mathcal{T}\sqrt{2\xi_{\text{max}}^2 + \xi_{\text{max}}^3}$ for cycle graph.
	The feature that a complete graph has the same $\mathcal{B}$ as a star graph will be utilized in trajectories optimization in Sec.\ref{subsec:backend}.
	\begin{figure}[t]
		\centering
		\includegraphics[width=0.5\textwidth]{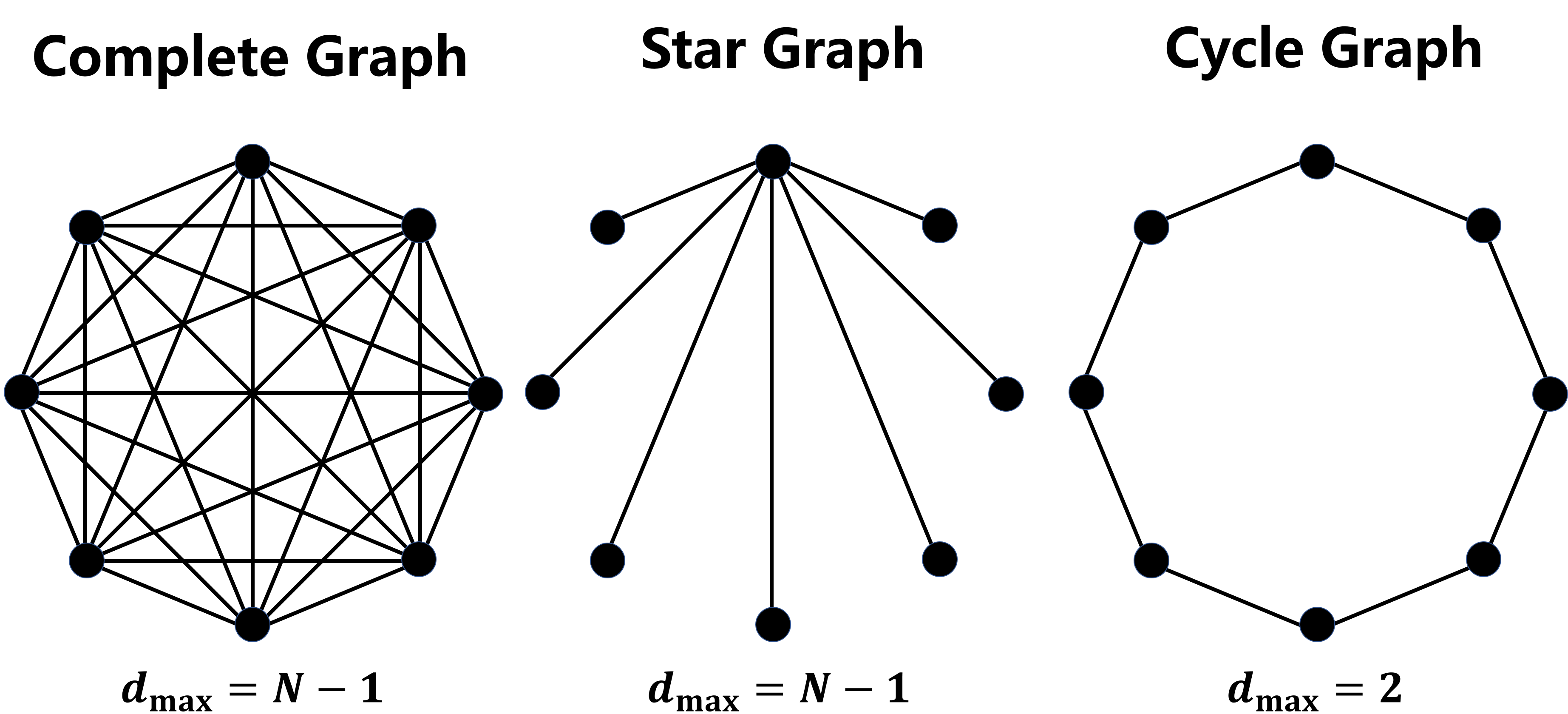}
		\caption{\label{fig:graph} Three important graphs and their $d_{\text{max}}$.}
		\vspace{-0.3cm}
	\end{figure}

%	
%	The practical implication of the certificate eigenvalue bound for limited error is that if we can model the sensor's detection precision through calibration, we can aim to satisfy the eigenvalue requirement of $\lambda_4(\mathbf{\hat{K}})$ to enforce optimality in mutual localization even in the presence of noise. 
%	We will further explore this insight in the optimality-enhanced swarm trajectory planning.

	\section{Certifiable Trajectory Planning for Bearing-based Swarm}
	\label{sec:planning} 
	In this section, a complete trajetory planning method for bearing-based swarm is presented.
	Here, we define a set of trajectories as
	\begin{definition}
		(Certifiable swarm trajectories) Given swarm trajectories, if robots can perform certifiable optimal mutual localization as they move along them even in the presence of noise, these trajectories are certifiable.
	\end{definition}
	Based on the theoretical analysis in Sec.\ref{sec:theoretical}, we conclude that swarm trajectories that satisy $\lambda_4(\mathbf{\hat{K}}) \geq 2d_{\textup{max}}\mathcal{T}\sqrt{2\xi_{\text{max}}^2 + \xi_{\text{max}}^3}$ are certifiable.
	In addition to the certifiability requirement, we have also following considerations for expected trajectories.
	First, inter-robot visibility should be guaranteed in bearing-based swarm since it provides the observation basis. 
	Second, safety and dynamical feasibility should be satisfied for swarm navigation in clustered environments.
	Hence, we propose a hierarchical planning method, including visibility-guaranteed path searching detailed in Sec.\ref{subsec:frontend}, and trajectory optimization detailed in Sec.\ref{subsec:backend}.
	%	We expect that robots can perform certifiable optimal mutual localization as they move along the generated special trajectories, even in the presence of noise.
	
%	Additionally, the findings in Sec.\ref{sec:theoretical} emphasize the significance of $\lambda_4(\mathbf{\hat{K}})$ as a measure of resistance in noised cases. 
%	Considering above factors, we propose a complete motion planning framework for bearing-based swarms that ensures inter-robot visibility and estimation resistance. 
	%	Thus, in this section, we design a decentralized framework for localization-aware swarm motion planning, as shwon in Fig, including a visibility-guaranteed frontend and a optimality-enhanced trajectory optimization backend. 
	%	The distrbuted architecture allow less computation for single robot, leading better scalability.
	
	\subsection{Visibility-guaranteed Swarm Path Searching}
	\label{subsec:frontend}
	Our estimator requires that each robot needs to be observed by at least one another robot. 
	If any pair of robots is inter-visible, it means that robots must move in a same topology, which restict the movements of robots in clustered environment.
	To allow robots move more freely and still guarantee connected graph.
	We designate a special robot as the center robot and only require other robots ensure visibility to it.
	To achieve it, we introduce a special polytope.
	\begin{definition}
		Star convex polytope. A region $\mathcal{P}_S$ is star convex if there exists a point $p$ in $\mathcal{P}_S$ such that the line segment from $p$ to any point in $\mathcal{P}_S$ is contained in $\mathcal{P}_S$.
	\end{definition}
	
	Hence, if a region $\mathcal{P}_S$ is star convex for a point $p_c$, then any point in $\mathcal{P}_S$ is visible to $p_c$.
	It means that $\mathcal{P}_S$ represent a visible region for $p_c$.
	Given obstacle points and $p_c$, we can obtain its star convex polytope following method in (\cite{katz2007direct}).
	By the H-representation, the star convex polytope of $p_c$ is modeled as
	\begin{align}
		\mathcal{P}_S = \{x\in\mathbb{R}^3 | \mathbf{A}x \preceq b\}
	\end{align}
	where $\mathbf{A}=[n_1^T,\cdots,n_K^T] \in \mathbb{R}^{K\times3}$ is build by the outer normal vector $n_i$ of each face and $b = [n_1^T a_1, \cdots, n_K^T a_K] \in \mathbb{R}^K$ is formed by arbitray points $a_i$ on each faces. 
	A point $p$ is in $\mathcal{P}_S$ if it satisfies that $\max_{k\in[1,K]}\{d_k(p_c)\} \leq 0$	where 
	\begin{align}
			 d_k(p_c) &=n_k^T \left(f(p,p_c)-a_k \right), \\
			 f(p,p_c) &=  (2r / \norm{p-p_c} - 1)(p-p_c).
	\end{align}
	$f(p,p_c)$ is a sphere flipping transformation, and $r$ is a user-defined parameter.
	
	\begin{figure}[t]
		\centering
		\includegraphics[width=0.5\textwidth]{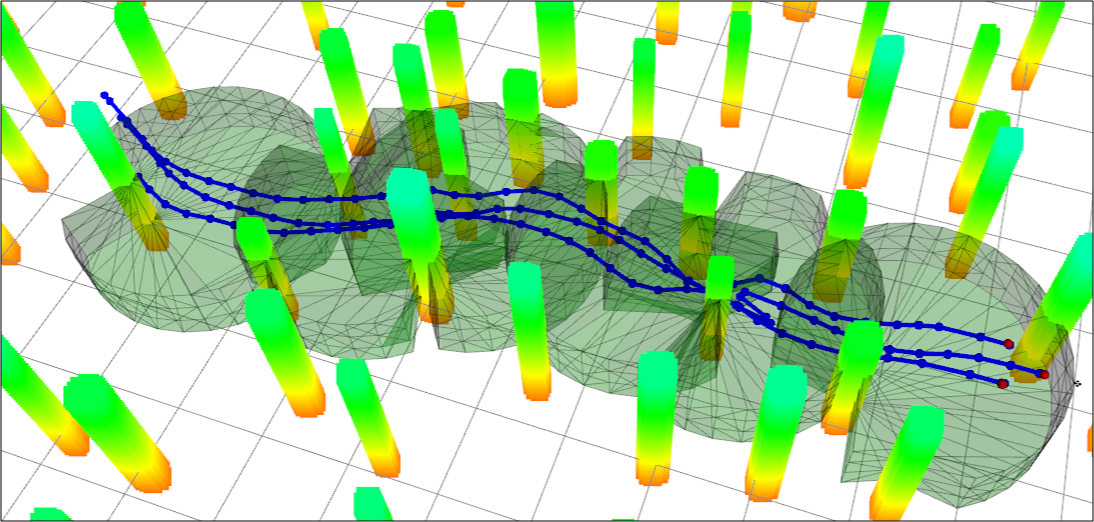}
		\caption{\label{fig:star_convex} A sequence of star convex polytopes and generated front-end path with visibility guarantee.}
		\vspace{-0.4cm}
	\end{figure}
	
%	In anther word, the visibility graph $\mathbf{\mathcal{G}}_\tau$ must be a connected graph.
%	Thus, it's indispensible to guarantee the visibility of bearing-based swarm during motion in obstacle environment.
%	To achieve it, we restrict other robots in the visible space of a center robot, which represent a zone that any point in it is visible to this robot.
%	Following our previous work \cite{zhong2020generating,liu2022star}, we model the visible space of a point $c$ as a starconvex polytope (SCP).

%	Then, we have that $p$ is visible to $c$ $\Longleftrightarrow \mathcal{D}(p) > 0$, where 
%	\begin{equation} \label{equ:scv}
%		\begin{aligned}
%			&\mathcal{D}(p) = \max_{i\in[1,K]}\{d_i\}, \ d_i =n_i^T\left(\mathcal{F}_c(p)-a_i\right), \\
%			&\ \ \ \mathcal{F}_c(p) =  (2r / \norm{p-c} - 1)(p-c).
%		\end{aligned}
%	\end{equation}
	
%	\begin{figure}[t]
%		\centering
%		\includegraphics[width=0.45\textwidth]{./figure/pic_front.png}
%		\caption{\label{fig:front} Demonstration of the front end path seach.}
%		\vspace{-0.4cm}
%	\end{figure}
	
	With visible region representation, we propose a centralized method for visibility-guaranteed swarm path searching.
	First, the designated center robot searches a collision-free path for itself by kinodynamic path searching.
	Next, waypoints are sampled densely along the searched path, which serve as $p_c$ to build a sequence of star convex polytopes.
	Then, some pre-determined directions are designated for other robots.
	Raycasting from each waypoint along these directions is performed, until these rays go beyond the polyhedron.
	These raycasting points are taken as waypoints for these robots.
	Following the obtained waypoints, these robots are visible to the center robot, which is suffcient to provide bearing measurements for mutual localization.
	An example of star convex polytopes and generated front-end paths is shown in Fig. \ref{fig:star_convex}.
	\subsection{Trajectory Optimization}
	\label{subsec:backend}
	Frontend paths are utilized as initial values to formulate the trajectory generation as a numerical optimization problem. 
	This problem aims to find safe, dynamically feasible, smooth, and certifiable trajectories $p_i(t): \mathbb{R} \in [0, T_f] \rightarrow \mathbb{R}^3$, with a specified flight time $T_f$.
	Given the initail state $s_i^0$ and the terminal state $s_i^f$, the trajectory optimization is formulated as 
	\begin{align}
		\min_{\{p_i(t)\}, T_f} &\  \sum_{i=1}^N \int_{0}^{T_f} \Vert p_i^{(3)}(t) \Vert^2 dt   + T_f \label{equ:traj1} \\
		%		s.t. \ \ \ &p(t) = \mathcal{M}_{\mathbf{q}, \mathbf{T}}(t), \forall t \in [0, T_f]\\
		s.t. \ \ \ & p_i^{[2]}(0) = s_i^0, \ p_i^{[2]}(T_f) = s_i^f \label{equ:traj3} \\
		&\mathcal{G}(p_i(t), ...,p_i^{(s)}(t)) \preceq \mathbf{0} \label{equ:traj4} \\
		&\mathcal{G}_c (\{p_i(t)\}) \preceq \mathbf{0}  \label{equ:traj5} 
	\end{align}
	The optimization problem in (\ref{equ:traj1}) optimizes both robots' trajectories and the flight time $T_f$ to minimize the total control efforts and the flight time.
	Minimizing control efforts is to produce smooth trajectories that are friendly to track and control, while minimizing the total time $T_f$ generates more aggressive trajectories.
	The boudary conditions (\ref{equ:traj3}) involve robots' state of position, velocity and acceleration $p_i^{[2]}(t) = [p_i(t), p_i^{(1)}(t), p_i^{(2)}(t)]$.
	Safety, dynamical feasibility, and inter-robot visibility are considered in (\ref{equ:traj4}).
	Safety guarantees collision avoidance with obstacles and other robots, dynamical feasibility ensures the robots' actuation capabilities, and inter-robot visibility ensures that robots remain visible to the central robot.
	These requirements are met by the continuous-time constraint $\mathcal{G}$ in (\ref{equ:traj4}), which must be satisfied at all times.
	Additionally, the certifiability requirement, a constraint for the entire swarm's trajectories, is incorporated into constraint $\mathcal{G}_c$ in (\ref{equ:traj5}).
	
	The trajectory generation in (\ref{equ:traj1}) is a nonlinear constrained optimization problem. 
	We follow our previous work (\cite{Wang2022geometrically}) to parameterize trajectories and solve the problem.
	A trajectory $p_i(t)$ is parameterized by $M$-piece polynomials $p_{ij}(t): \mathbb{R} \in [0, t_j] \rightarrow \mathbb{R}^3 , j \in [1,M]$, where a 5th-order polynomial with $\mathcal{C}^3$ continuity is used to connect two adjacent waypoints $q_{j-1}$ and $q_{j}$.
	Then, a trajectory with minimum control effort can be uniquely determined by boundary states $s_i^0, s_i^f$, waypoints $\mathbf{q}_i = [q_{i1}, \cdots, q_{i(M-1)}]\tp$ and passing time $\mathbf{T}_i = [t_{i1}, \cdots, t_{iM}]\tp$, having the boudary conditions naturally satisfied.
	To deal with the continuous-time constraint $\mathcal{G}$ in (\ref{equ:traj4}), we sample enough states along each trajectory, design feasible objective functions with states, and relax constraints into soft penalties. 
	The related objective functions are listed in Table.\ref{tab:cost}, where $v_{\textup{max}}$ and $a_{\textup{max}}$ the maximum velocity and acceleration, $d(p)$ the distance to nearest obstacle evaluated by Euclidean Signed Distance Field (ESDF), $d_s$ and $d_r$ the clearance and $\psi(*) = \textup{max}(*,0)^3$.
	In visibility, we use differentiable sum-of-log function to replace the undifferentiable maximum function.

	\begin{table}[t]
		\centering
		\caption{\label{tab:cost} Objective Functions}
		\begin{tabular}{cc}
			\hline
			\textbf{Item} &  \textbf{Objective Function}   \\
			\hline
			\multicolumn{1}{c}{\begin{tabular}[c]{@{}c@{}}Dynamic\\ Feasibility\end{tabular}} &
			\multicolumn{1}{c}{
				\begin{tabular}[c]{@{}c@{}}
					\\$\mathcal{P}_d = \psi(\Norm{p^{(1)}_i(\tau)}^2 - v_\textup{max}^2)$ \\
					$+ \psi(\Norm{p^{(2)}_i(\tau)}^2 - a_\textup{max}^2)$\\ \\
			\end{tabular}} \\
			\hline
			\multicolumn{1}{c}{\begin{tabular}[c]{@{}c@{}}Obstacle\\ Avoidance\end{tabular}} &
			\multicolumn{1}{c}{
				\begin{tabular}[c]{@{}c@{}}
					\\$\mathcal{P}_s = \psi\left(d_s - d(p_i(\tau))\right)$\\ \\
			\end{tabular}} \\
			\hline
			\multicolumn{1}{c}{\begin{tabular}[c]{@{}c@{}}Reciporal\\ Avoidance\end{tabular}} &
			\multicolumn{1}{c}{
				\begin{tabular}[c]{@{}c@{}}
					\\$\mathcal{P}_r = \sum_{j \neq i}^{N} \psi(d_r - \Norm{ p_i(\tau) - p_j(\tau)})$\\ \\
			\end{tabular}} \\
			\hline
			Visibility&
			\multicolumn{1}{c}{
				\begin{tabular}[c]{@{}c@{}}
					\\$\mathcal{P}_v = \psi(d_v - \textup{log}(\sum_{k=1}^K e^{\alpha d_k(p_c)})/\alpha)$\\ \\
			\end{tabular}} \\
			\hline
		\end{tabular}
	\end{table}
	
	For certifiability constraint, we sample the positions of robots to construct $\mathbf{\hat{K}}$ and transform the constraint into a penalty. 
	The objective function for this penalty is formulated as follows:
	\begin{align}
		\mathcal{P}_c &= \psi(2 d_{\textup{max}} n_t \sqrt{2\xi_{\text{max}}^2 + \xi_{\text{max}}^3} - \lambda_4(\mathbf{\hat{K}})),
	\end{align}
	where $n_t = \textup{int}(T_f / \Delta t)$ is the sampling number, $ \Delta t$ is the sampling time interval.
	Here, we take $d_{\textup{max}} = N-1$ for a star graph since we consider that all robots remain visible to a central robot. 
	Moreover, even in scenarios where two non-central robots are inter-visible, certifiability is still upheld due to the fact that the certificate eigenvalue bound remains the same for both complete graph and star graph.
	The matrix $\mathbf{\hat{K}}$ is constructed as below:
%	Sampling is conducted to satisfy the certifiability constraint $\mathcal{G}_c$, which exert on the entire swarm trajectories.
%	Recall (\ref{equ:reK2}) and (\ref{equ:gt_bearing}), we sample all robots' positions with time interval $\Delta t$ to construct $\mathbf{\hat{K}}$ 
%	\begin{equation}
		\begin{gather}
			\mathbf{\hat{K}}_{ij} = \sum_{\alpha=1}^{n_t} \mathbf{\Phi}_{ij}(\alpha \Delta t), \ i \neq j, \\
			\mathbf{\hat{K}}_{ii} = -\sum_{k\neq i} \mathbf{\hat{K}}_{ik}, \\
			\mathbf{\Phi}_{ij}(t) = -\varphi_{ij}(t) \varphi_{ij}(t)^T, \\
			\varphi_{ij}(\tau) = \frac{p_j(\tau)-p_i(\tau)}{\norm{p_j(\tau)-p_i(\tau)}}.
		\end{gather}
%	\end{equation}
	
	Finally, the constrained trajectory generation in (\ref{equ:traj1}) is transformed into a unconstrained nonlinear optimization problem, whose decision variables are waypoints $\mathbf{q}_i$ and passing time $\mathbf{T}_i$ for $i \in [1,N]$.
	It can be solved by a quasi-Newton method (\cite{Wang2022geometrically}).
	For efficient optimization, it is imperative to obtain the gradients of the involved objective functions. 
	The derivation of $\mathcal{P}_d$, $\mathcal{P}_s$, $\mathcal{P}_r$ and $\mathcal{P}_v$ are covered in our previous work.
	For $\mathcal{P}_c$, the derivative of the $\mathbf{\hat{K}}$ with respect to any coordinate $\rho_i \in \{x_i, y_i, z_i\}$ of the robot $i$ located at $p_i = [x_i,y_i,z_i]\tp$ need be evaluated.
	Define the notation $\rho_{ij} = \rho_j - \rho_i$ and $\gamma_{ij} = \frac{1}{d_{ij}^2}$. 
	Then, for $\mathbf{\Phi}_{ij}, i \neq j$, we have 
	\begin{gather}
		\frac{\partial \mathbf{\Phi}_{ij}}{\partial x_i} = \frac{2}{d_{ij}^2}
		\begin{bmatrix}
			\frac{x_{ij}^3}{d_{ij}^2} - x_{ij}  & \frac{x_{ij}^2y_{ij}}{d_{ij}^2}- \frac{y_{ij}}{2} & \frac{x_{ij}^2z_{ij}}{d_{ij}^2}- \frac{z_{ij}}{2}\\
			\star & x_{ij}y_{ij}^2 &  x_{ij}y_{ij}z_{ij} \\
			\star & \star & x_{ij}z_{ij}^2
		\end{bmatrix} \\
		\frac{\partial \mathbf{\Phi}_{ij}}{\partial y_i} = \frac{2}{d_{ij}^2}
		\begin{bmatrix}
			x_{ij}y_{ij}^2  &\frac{x_{ij}y_{ij}^2}{d_{ij}^2}- \frac{x_{ij}}{2} & x_{ij}y_{ij}z_{ij}\\
			\star & \frac{y_{ij}^3}{d_{ij}^2} - y_{ij} &  \frac{y_{ij}^2z_{ij}}{d_{ij}^2}- \frac{z_{ij}}{2} \\
			\star & \star & y_{ij}z_{ij}^2
		\end{bmatrix} \\
		\frac{\partial \mathbf{\Phi}_{ij}}{\partial z_i} = \frac{2}{d_{ij}^2}
		\begin{bmatrix}
			x_{ij}^2z_{ij}  & x_{ij}y_{ij}z_{ij} & \frac{x_{ij}z_{ij}^2}{d_{ij}^2}- \frac{x_{ij}}{2} \\
			\star & y_{ij}^2z_{ij} & \frac{y_{ij}z_{ij}^2}{d_{ij}^2}- \frac{y_{ij}}{2} \\
			\star & \star & \frac{z_{ij}^3}{d_{ij}^2} - z_{ij}
		\end{bmatrix} 
	\end{gather}
	where the symbol $\star$ represents symmetric terms.
	These expressions are sufficient to compute the gradient of $\lambda_4(\mathbf{\hat{K}})$ respect to $p_i(t)$ as below
	\begin{align}
		\frac{\partial \lambda_4(\mathbf{\hat{K}})}{\partial p_i(t)} &= u^T \frac{\partial \mathbf{\hat{K}}}{\partial p_i(t)} u \\
		&= \sum_{(j,k)\in \mathcal{E}} (u_j - u_k)^T \frac{\partial \mathbf{\hat{K}}_{jk}}{\partial p_i(t)} (u_j - u_k) \\
		&= \sum_{j \in \mathcal{N}_i} (u_i - u_j)^T \frac{\partial \mathbf{\Phi}_{ij}(t)}{\partial p_i(t)} (u_i - u_j) 
	\end{align}
	where $u$ is the eigenvector corresponding to $\lambda_4(\mathbf{\hat{K}})$.

	\section{Simulation Experiments}
	\label{sec:simulation}
	Simulation experiments are meticulously designed to evaluate the effectiveness of our proposed methods, including the estimator and the planner.
	In Sec.\ref{subsec:result rpe}, we assess the optimality, certifiablity, accurancy and efficiency of our estimator, benchmarking it against state-of-the-art methods.
	In Sec.\ref{subsec:result metric}, we validate our theoretical conclusions. 
	This involves generating swarm trajectories under varying levels of maximum detection noise $\xi_{\text{max}}$ and conducting estimations to illustrate the impact of swarm motion on the estimator's performance.
	Additionally, we compare our swarm planner with other existing planners to highlight its effectiveness in ensuring visibility and estimation optimality.
	In Sec.\ref{subsec:overall}, our estimator and planner are integrated into a complete system.
	This system is then applied in an extensive 700+ meter swarm navigation in complex environment to verify the overall performance of our methods.
	The detai of simulation experiments is presented in the supplementaried videos.
	\begin{figure*}[!t]
		\centering
		\includegraphics[width=1.0 \textwidth]{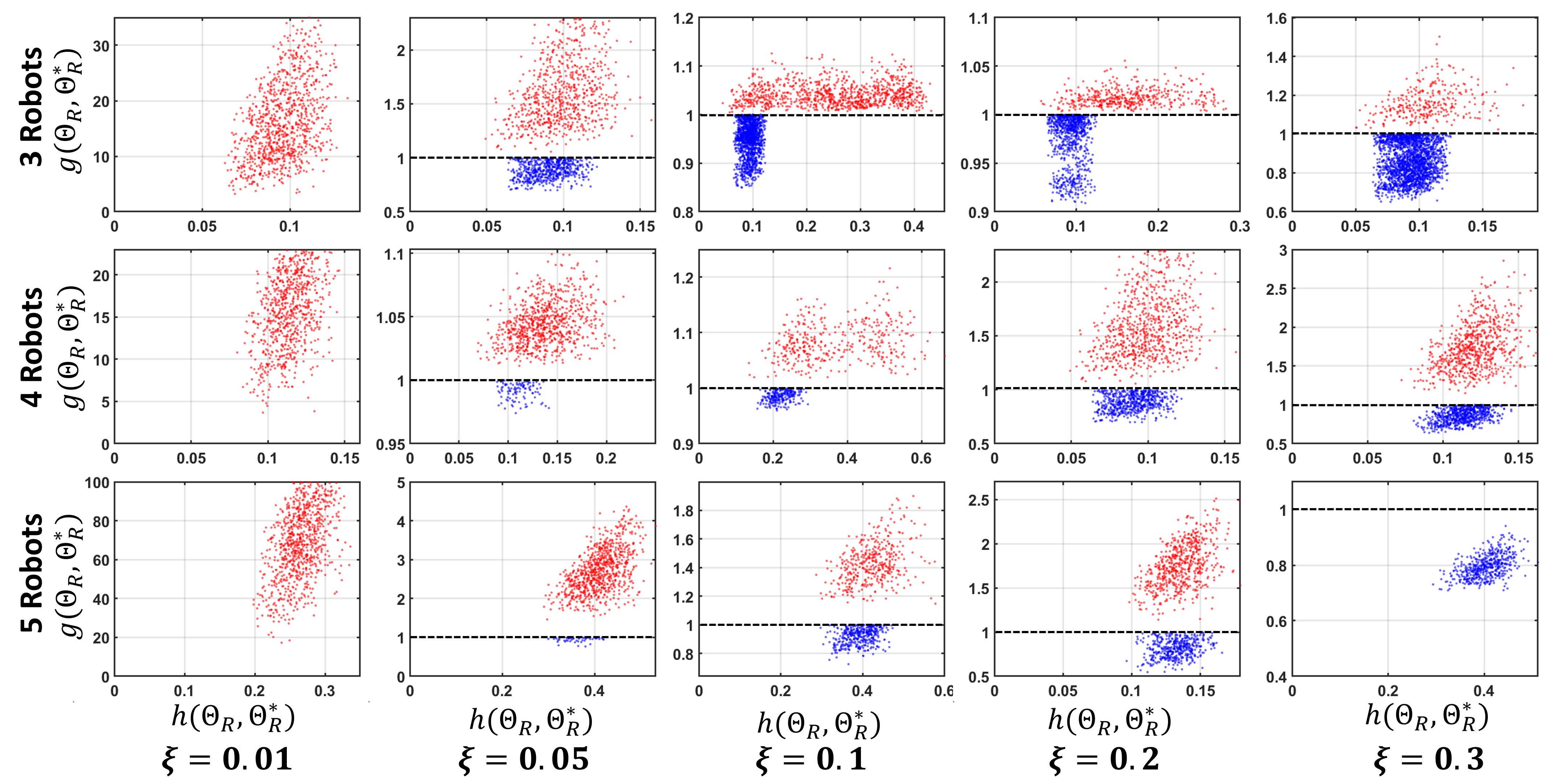}
		\caption{\label{fig:certifiabilit} Results of certifiability evaluation over different magnitude of added noise $\xi$ and numbers of robots.
			Our estimator can always correctly certify the global optimality of obtained solutions.}
		\vspace{-0.5cm}
	\end{figure*}
	\begin{figure}[t]
		\centering
		\includegraphics[width=0.5 \textwidth]{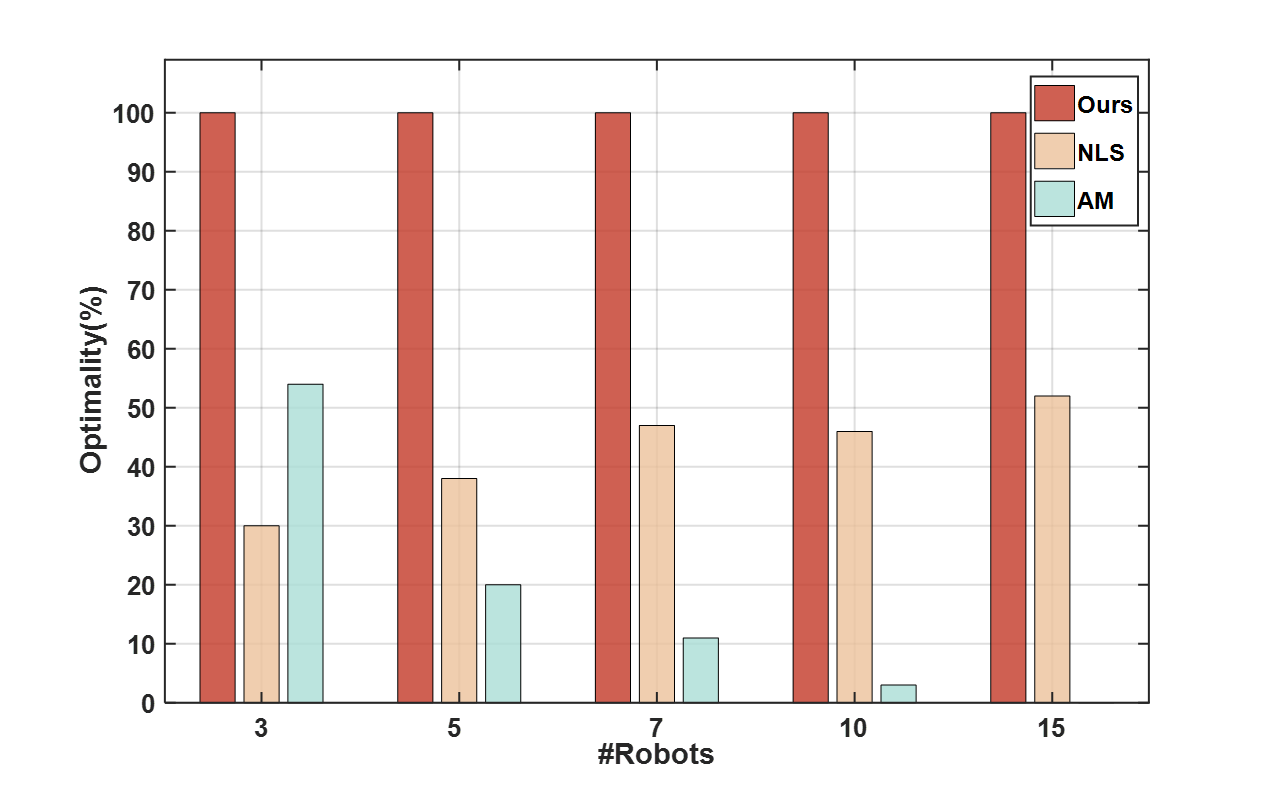}
		\caption{\label{fig:optimality} Results of optimality benchmark between our estimator and local optimization based methods.
			Our estimator can guarantee global optimality in noise-free cases.}
		\vspace{-0.3cm}
	\end{figure}
	
	\begin{figure*}[!t]
		\centering
		\includegraphics[width=1.0 \textwidth]{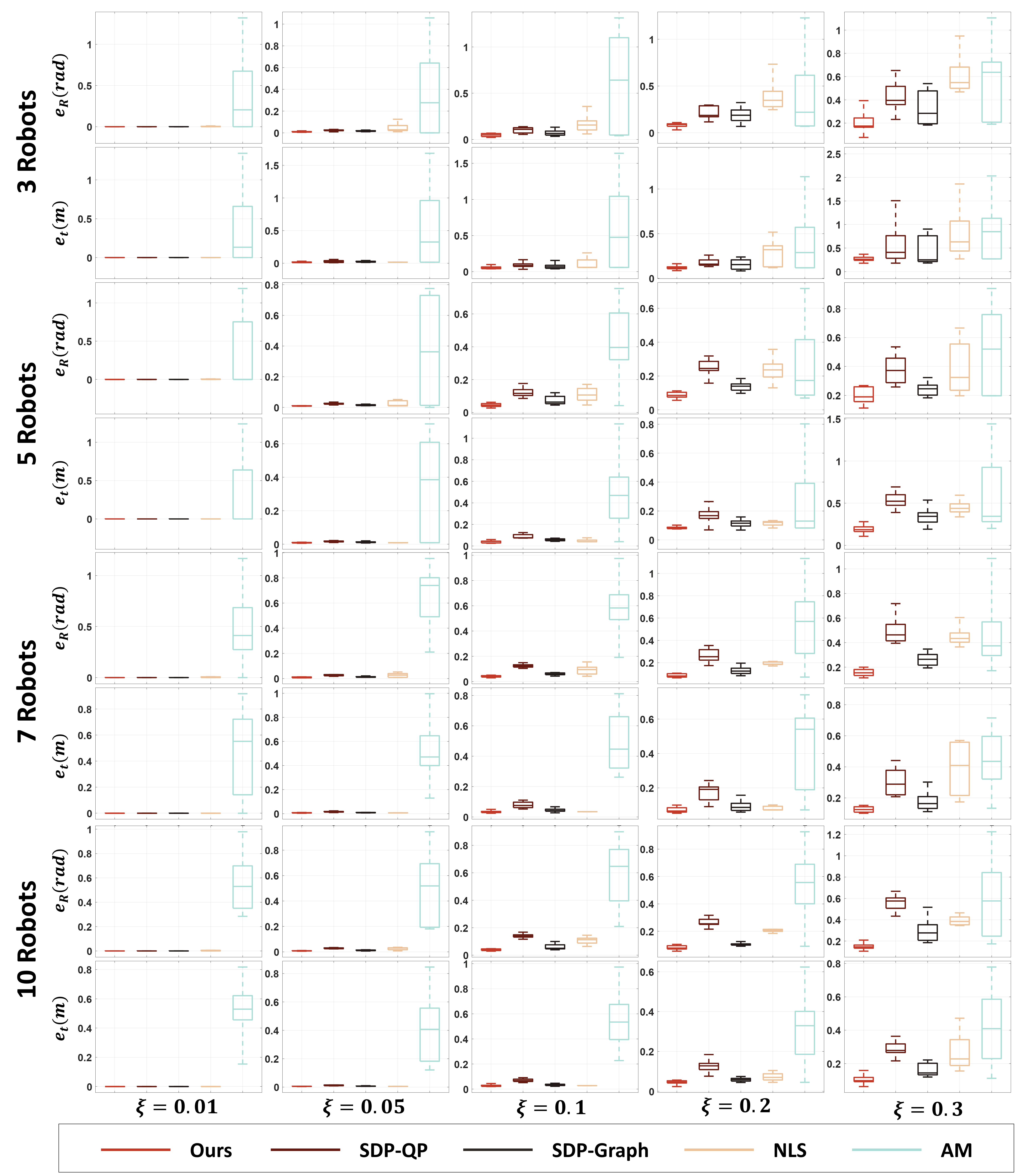}
		\caption{\label{fig:simulation1} Distribution of estimation errors using different methods.
			Simulations are performed with different magnitude of added noise $\xi$ and numbers of robots.
			Our estimator consistently outperforms existing methods in accurancy and consistency.}
		\vspace{-0.5cm}
	\end{figure*}
	
	\subsection{Evaluation of Certifiable Mutual Localization} 
	\label{subsec:result rpe}
	\subsubsection{Implementation Details}
	In our benchmark experiments, we select several methods for comparison, including the Nonlinear Least Squares (NLS) method, thr Alternating Minimization (AM) (\cite{jang2021multirobot}) method and prior SDP-based methods (\cite{wang2022certifiably}) and (\cite{wang2023bearing}).
	The NLS method employs euler angles as optimization variables and solve the original non-convex problem using the Levenberg-Marquardt algorithm.
	Local optimization-based methods, such as NLS and AM, require initial guesses, which are generated through random sampling, since we consider no prior knowledge about the problem.
	The method in (\cite{wang2022certifiably}) relaxes a MIQCQP problem to recover data correspondence. 
	We remove data association consideration and denote the simplified method as SDP-QP.
	Another SDP-based method (\cite{wang2023bearing}) performs relaxation on a directed graph.
	We denote the method as SDP-Graph.
	AM and SDP-QP rely on a parameter $s$, the distance between two data frames, which significantly influences accuracy.
	We set this parameter to a constant value of 5 in all experiments.
	All methods are implemented in MATLAB, utilizing SeDuMi for the SDP solver and \textit{lsqnonlin} for the NLS solver.
	
	In simulated data generation, we produce 10 waypoints for each robot.
	These waypoints were then used to construct an $SE(3)$ B-spline, following the method in (\cite{geneva2020openvins}). 
	The distance between consecutive waypoints can be changed to adjust the smoothness.
	In terms of local odometry, we sampled 100 poses along each robot's trajectory.
	Based on these poses, noisy bearing measurements are generated following the noise model in (\ref{equ:noise_bearing}). 
	The magnitude of detection error are denoted is $\xi$.
	%	For the AM method \cite{jang2021multirobot} and SDP method \cite{wang2022certifiably}, which are designed for only two-robot case, we solve the relative pose between one robot and others, respectively.
	%	 which varies from groud truth ($\xi = 0$) to a extreme noise level ($\xi = 0.2$).
	%	We also compare the average errors over all runs $\bar{e}_t$ and $\bar{e}_r$.

	\subsubsection{Optimality Benchmark}
	To validate the effectiveness of our certifiable estimator in recovering ground-truth relative poses, we compare its optimality with local optimization-based methods (NLS and AM).
	Experiments are conducted with varying numbers of robots, and we set $\xi = 0$ to generate noise-free bearings.
	In each simulation, we perform 100 Monte-Carlo tests with radomly generated swarm trajectories.
	Given that the measurements are noise-free, we consider the solution $\mathbf{\Theta_R}$ to be globally optimal if $\norm{\mathbf{\Theta_R} - \mathbf{\hat{\Theta}_R}}_F \leq 10^{-3}$.
	The proportion of optimal solutions obtained in these simulations is depicted in Fig. \ref{fig:optimality}.
	The results demonstrate that our estimator consistently achieves optimal solutions across all robot numbers, while both NLS and AM are possible to be trapped in local minima due to random initial values.
	
	\subsubsection{Certifiability Evaluation}
%	Except obtaining ground-truth relative poses, our method can certify whether a candidate solution is global optimal in noisy cases by checking whether $\mathbf{K}$ is positive semi-definite.
	To evaluate the certifiability of our method, we perform relative pose estimation using various magnitudes of added noise and different numbers of robots. 
	After obtaining a solution $\mathbf{\Theta_R^*}$ with our estimator, we certify its optimality by checking the positive semi-definiteness of $\mathbf{K}$.
	To confirm the correctness of our certification, we sample 1000 random relative rotations $\mathbf{\Theta_R}$, calculate their costs $f(\mathbf{\Theta_R})$ and obtain the following values:
	\begin{align}
		g(\mathbf{\Theta_R}, \mathbf{\Theta_R^*}) &= \frac{f(\mathbf{\Theta_R})}{f(\mathbf{\Theta_R^*})} \\
		h(\mathbf{\Theta_R}, \mathbf{\Theta_R^*}) &= \norm{\mathbf{\Theta_R} - \mathbf{\Theta_R^*}}_F
	\end{align}
	If any sampled candidate satisfies $g(\mathbf{\Theta_R}, \mathbf{\Theta_R^*}) < 1$, it indicates that the solution $\mathbf{\Theta_R^*}$ is not globally optimal.
	The distribution of these sampled candidates is depicted in Fig. \ref{fig:certifiabilit}.
	If a solution certified as globally optimal, the sampled candidates are represented with red points, while points corresponding to unoptimal solutions are in blue. 
	For clarity, we omit information-less blue points where $g(\mathbf{\Theta_R}, \mathbf{\Theta_R^*}) \geq 1$.	
	The results show that for all values of $\xi$ and robot numbers, if the obtained solution is certified to be globally optimal, all sampled candidates have higher costs.
	On the other hand, if we certify a solution not globally optimal, then there must exist one relative rotation with a lower cost.
	This experiment validates our method's ability to effectively certify the global optimality of obtained solutions.
	By enabling us to dismiss non-optimal solutions and prevent coordinate confusion, our approach ensures certifiability in mutual localization.

	\subsubsection{Accuracy Benchmark}
	
	\begin{figure}[t]
		\centering
		\includegraphics[width=0.45 \textwidth]{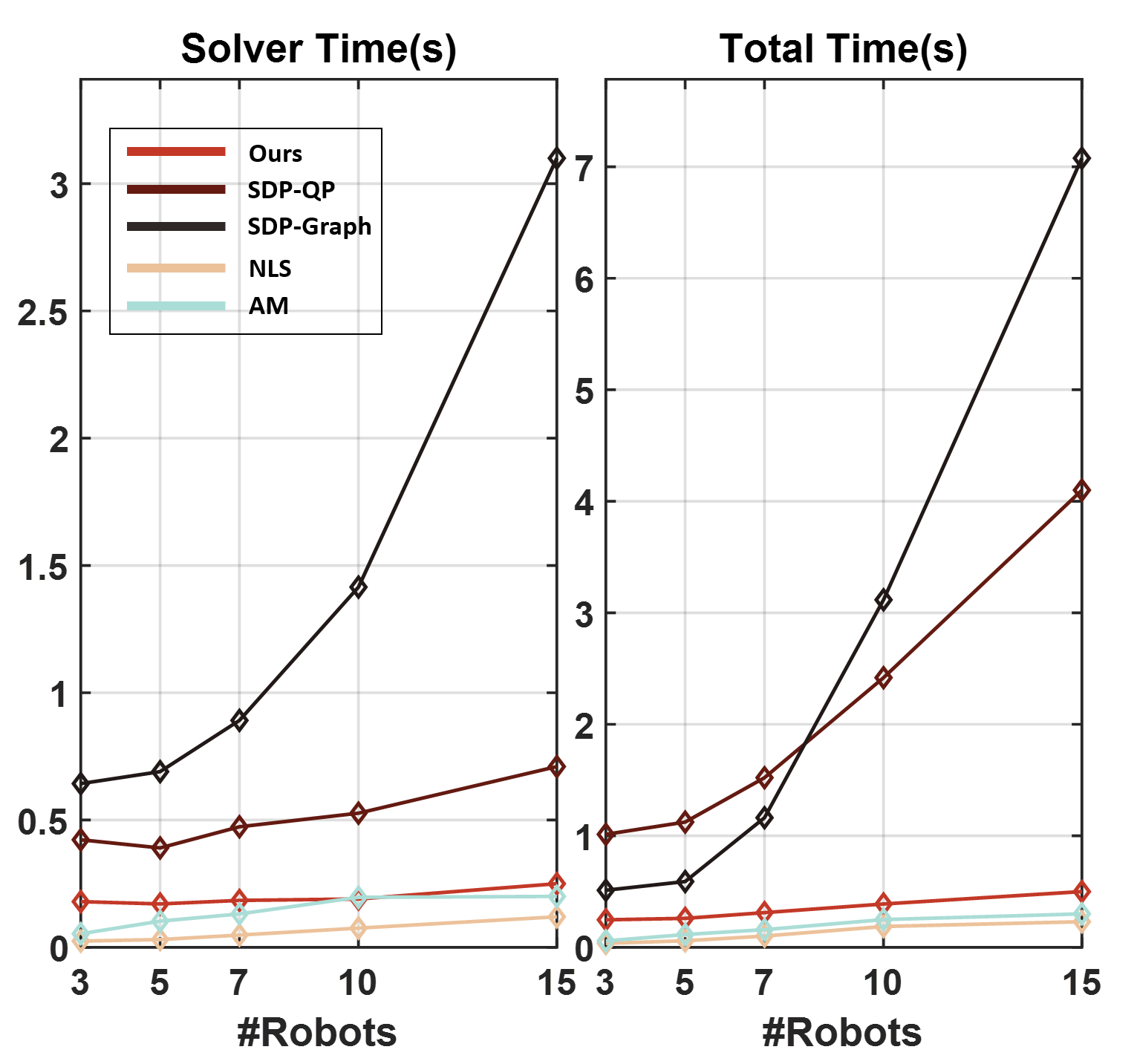}
		\caption{\label{fig:robot_num} Average time consumption with different robot numbers.
		Our estimator can always run in real-time.}
		\vspace{-0.5cm}
	\end{figure}
	In accuracy benchmark experiments, we change $\xi$ to simulate more generalized cases from noise-free cases ($\xi=0$) to extreme noised cases ($\xi = 0.2$).
	We evaluate the translation and rotation parameters for all methods, with the translation and rotation error denoted as
	\begin{align}
		e_t &= \sum_{i}^{N}\Norm{t_i - \hat{t}_i} / N \\
		e_r &= \sum_{i}^{N}\Norm{\angle(\mathbf{R}_i^\mathrm{T} \mathbf{\hat{R}}_i)} / N,
	\end{align}
	where $\angle(\star)$ is the angle of the rotation error (\cite{zhang2018atutorial}).
	The distribution of translation and rotation estimation errors are illustrated in Fig.\ref{fig:simulation1}.
	The results show that local optimization based methods (NLS and AM) always show randomness and could not guarantee accuracy of solution. 
	The disadvantage of relying on good initial guess for these methods are demonstrated.
	Among SDP-based methods, noise level $\xi$ affects accuracy in similar manner but at different scales.
	In the noise-free cases, all SDP-based methods can recover relative poses consistently and accruately.
	When $\xi$ is limited ($\xi \leq 0.1$), their solutions are all usable.
	But under extreme noise level ($\xi = 0.2$), SDP-QP and SDP-Graph begin to show randomness, meaning their relaxtion are not tight any more.
	Overall, our estimator significantly excel at accuracy and consistence.
	The reason can be attributed to two aspects.
	First, our problem formulation is more concise by mutual distance elimination, which reduces source of error from distance variables.
	Second, compared with SDP-QP and SDP-Graph, which use data of two frames to formulate error, our error formulation only use single frame data, avoiding degeneration in motionless cases.
	Thus, our estimator can be considered more useful and robust in most situations.

	\subsubsection{Efficiency Benchmark}
	
	\begin{figure*}[!t]
		\centering
		\includegraphics[width=1.0 \textwidth]{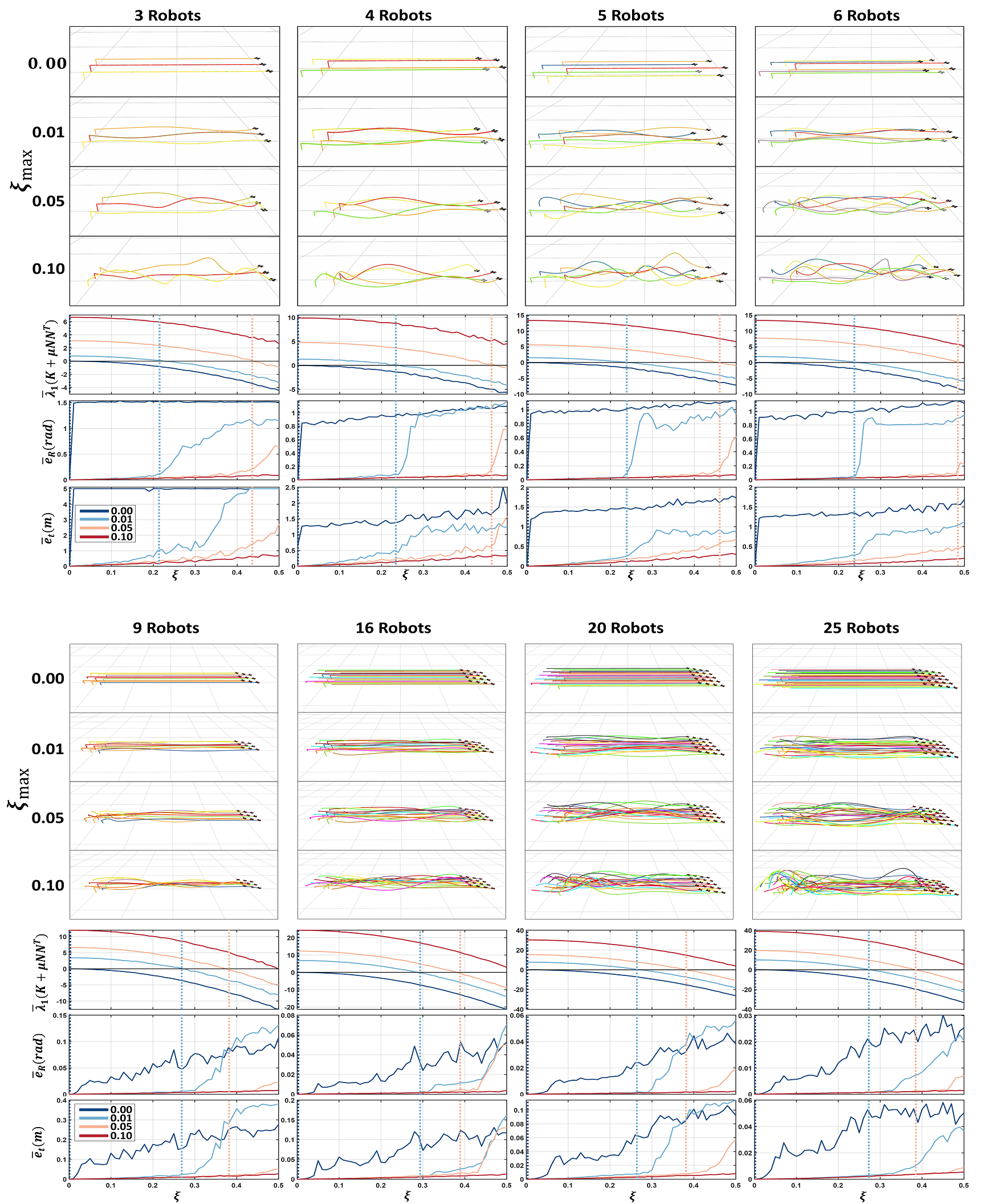}
		\caption{\label{fig:multi_traj} Demonstration of different swarm trajectories. 
			The column denotes different number of robots.
			The row denotes different desired maximum allowed noise.
			Optimality and average estimation error for different swarm trajectories. 
			Different color denotes swarm trajectories with different $\xi_{\text{max}}$.
			The solid dot represent the first time when the optimality is not hosted.}
		\vspace{-0.1cm}
	\end{figure*}

	To assess the efficiency of our algorithm with varying numbers of robots, we fix $\xi = 0.05$ and measure the time consumed during estimation.
	We ocus on comparing both the solver time $t_{opt}$ and the total time $t_{total}$, which includes data preparation and pre-processing time.
	100 simulations are conducted for each robot count, and average time consumption is presented in Fig.\ref{fig:robot_num}.
	The results show that the number of robots affects the time consumption at different scales for different methods.
	When the number of robots is limited (3-5), most methods can run in real-time.
	However, as the number of robots increases, both $t_{opt}$ and $t_{total}$ for SDP-QP increase rapidly, owing to the increase in variables proportional to the square of the robot count.
	NLS is consistenly the fastest, closely matched by AM.
	Our estimator ranks between local optimization-based methods and previous SDP-based methods in efficiency.
	Thanks to its concise problem formulation, our estimator can always run in real-time across all tested robot counts.

	\subsection{Evaluation for Certifiable Swarm Planning} 
	\label{subsec:result metric}

	\begin{figure*}[!th]
		\centering
		\includegraphics[width=1.0 \textwidth]{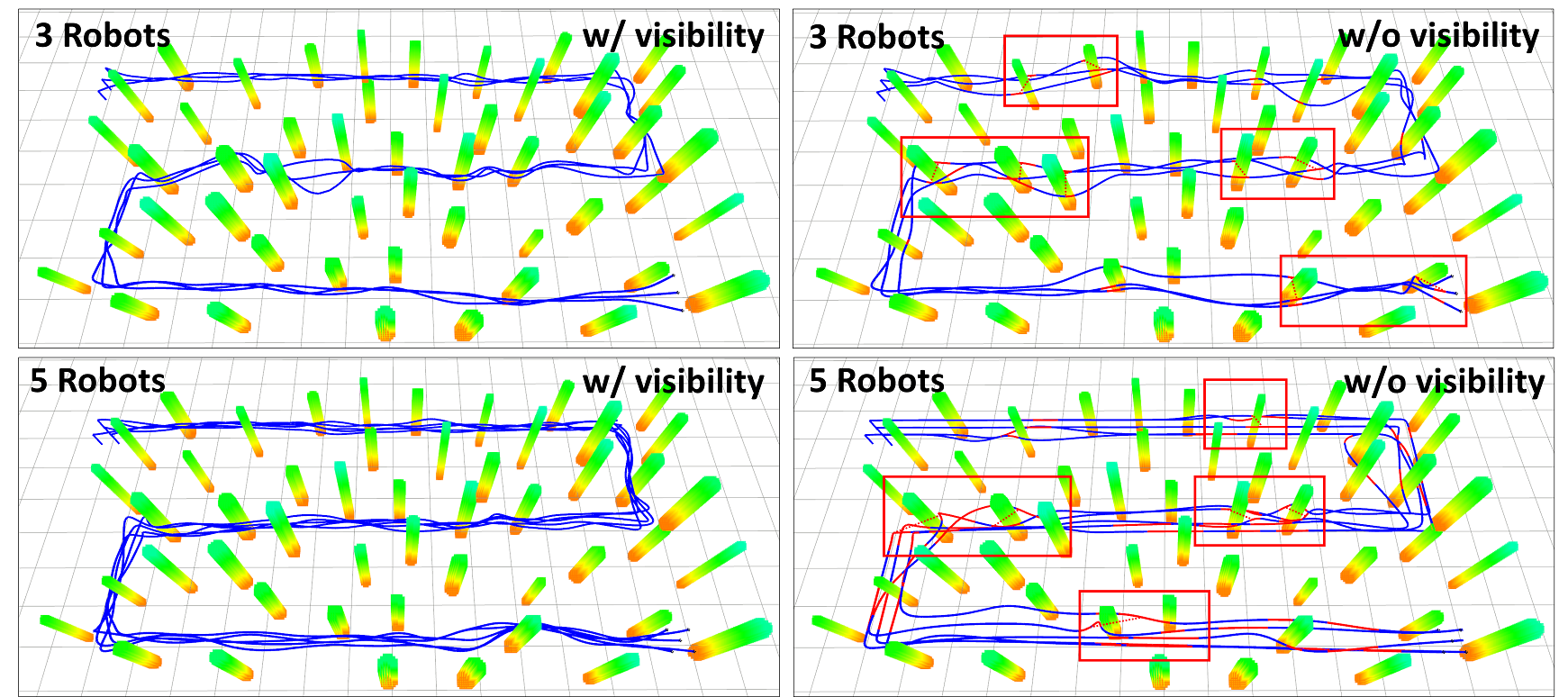}
		\caption{\label{fig:vis_compare_35}  Comparision of swarm trajectoies with and without inter-robot visibility.
		 Our planner can always guarantee that each robot is visible to at least one of other robots.
	 	}
		\vspace{-0.1cm}
	\end{figure*}

	\begin{figure*}[!th]
		\centering
		\includegraphics[width=1.0 \textwidth]{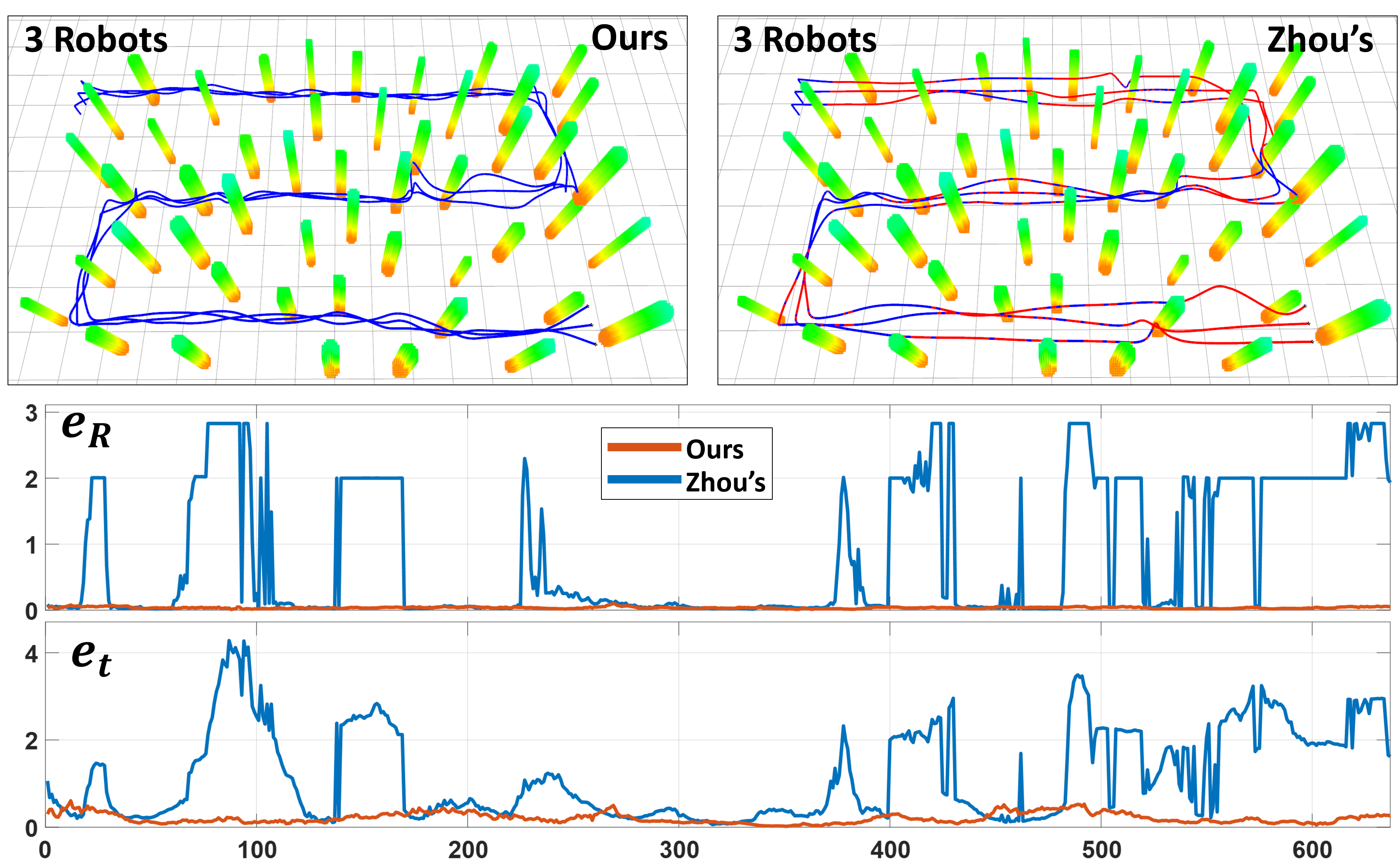}
		\caption{\label{fig:planning_benchmark_3} Swarm trajectories and estimation error during 3-robot swarm flight.
			Our planner can generate trajectories that always guarantee estimation optimality and accuracy.}
		\vspace{-0.4cm}
	\end{figure*}

	\subsubsection{Impact of Swarm Motion on Mutual Localization}
	\label{subsubsec:impact_of_motion}
	To evaluate the impact of swarm motion on our estimator, we generate swarm trajectories with different $\xi_{\text{max}}$ for robots, use them to produce noisy bearing measurements and then perform mutual localization.
	To highlight the effect of certifiable swarm planning, we eliminate consideration of obstacle avoidance and inter-robot visibility.
	The robots are programmed to take off from some points near $(0,0,1)$ and arrive respective targets around $(20,0,1)$, with parameters $v_{\text{max}}=2.0, a_{\text{max}}=3.0,d_r = 5.0$.
	By altering $\xi_{\text{max}}$ and the number of robots, a variety of swarm trajectories are generated.
	Based on them, we add noises of different $\xi$ to construct $\mathbf{K}$ and perform relative pose estimation.
	For each $\xi$, we conduct 10 tests, and calculate the average certificate eigenvalue  $\bar{\lambda}_1(\mathbf{K+\mu NN\tp})$ and the average estimation error $\bar{e}_t$ and $\bar{e}_r$.
	The results, including swarm trajectories, average certificate eigenvalue, and average estimation error respect to the noise magnitude $\xi$, are shown in Fig. \ref{fig:multi_traj}.
	
	The variation in swarm trajectories indicates that as $\xi_{\text{max}}$ increase, swarm trajectories become more twisted and complex to provide enough motion excitement.
	The change of $\bar{\lambda}_1(\mathbf{K+\mu NN\tp})$across different noise magnitudes corroborate our theoretical findings: the swarm trajectories generated with higher  $\xi_{\text{max}}$ demonstrate improved noise resistance to maintain estimation optimality.
	Especially, the swarm trajectories with $\xi_{\text{max}} = 0$ show no resistance to detection noise. 
	We define  the actual maximum noise magnitude $\mathring{\xi}_{\text{max}}$ under which  estimation optimality is preserved (corresponding to the dotted line).
	We observe that $\mathring{\xi}_{\text{max}}$ consistently exceeds our pre-computed $\xi_{\text{max}}$.
	This discrepancy arises from approximations and inequalities involved in deriving the certificate eigenvalue bound $\mathcal{B}$.
	 Interestingly, the actual maximum allowed noise $\mathring{\xi}_{\text{max}}$ for the same $\xi_{\text{max}}$ is similar regardless of the number of robots involved.
	It may come from that $\mathcal{B}$ already accounts for robot quantity.
	
	The estimation error results demonstrate the effect of swarm motion on accuracy of mutual localization.
	When the estimation optimality holds, the estimation error remains limited and the obtained solution is enough accurate to align robots' reference frames. 
	However, when noise surpasses $\mathring{\xi}_{\text{max}}$, the estimation errors increase sharply .
	It comes from unoptimal and erroneous solutions cased by the negtive impact of detection error. 
	Thus, raising the certificate eigenvalue $\lambda_4(\mathbf{\hat{K}})$ can effectively increase the maximum  allowable noise magnitude while maintaining estimation optimality and accuracy. 
	It validates our theoretical conclusions and provides motivation for enhancing $\lambda_4(\mathbf{\hat{K}})$ through  certifiable swarm planning.
	%	We also note that under the same noise, trajectories with larger $\lambda_{\neq0,\textup{min}}(\mathbf{K})$ always provide more accurate estimation, both in translation and rotation. 
	%	The reason is similar to the infomation theory that the enhanced trajectories are more twisted and can provide more motion excitement, leading to richer information and more accurate estimation solution.
	
	\subsubsection{Swarm Trajectories Comparision on Visibility}
	We conducted comparative experiments to validate that our planner can ensure the inter-robot visibility.
	In these experiments,	50 obstacles are generated randomly in 30m $\times$ 50m rectangular environment, and robots are instructed to navigate through specified waypoints in sequence.
	We compare two set of swarm trajectories: one considering inter-robot visibility as per our planner, and the other not considering it.
	These trajectories were thoroughly tested in scenarios involving 3-robot and 5-robot swarms. 
	The swarm trajectories are shown in Fig. \ref{fig:vis_compare_35}, where blue paths represent scenarios where each robot remains visible to at least one other robot, and red paths indicate instances where at least one robot is not observable by any other robot.
	
	The results show in complex environments, swarms that do not account for inter-robot visibility often encounter situations where lines of sight are obstructed by obstacles.
	In contrast, our planner, by incorporating inter-robot visibility, enables the swarm to maintain stable mutual observations, thereby providing ample data for effective mutual localization.
	\begin{figure*}[!th]
		\centering
		\includegraphics[width=1.0 \textwidth]{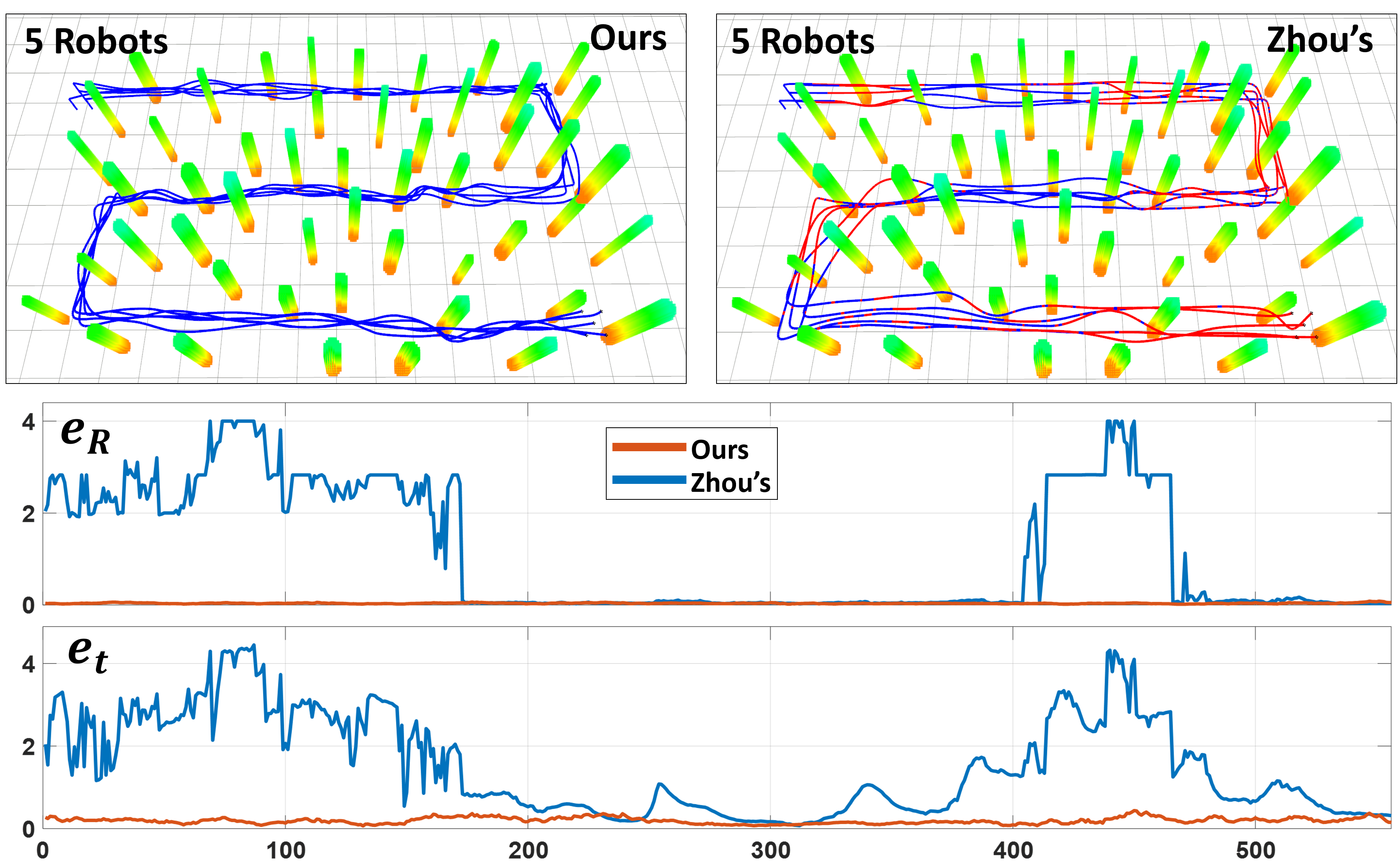}
		\caption{\label{fig:planning_benchmark_5} Swarm trajectories and estimation error during 5-robot swarm flight.
			Our planner can generate trajectories that always guarantee estimation optimality and accuracy.}
		\vspace{-0.1cm}
	\end{figure*}
	
%	\begin{figure*}[!th]
%		\centering
%		\includegraphics[width=1.0 \textwidth]{./figure/planning_benchmark_35.png}
%		\caption{\label{fig:planning_benchmark_35} Swarm trajectories and estimation error during swarm flight.
%		Our planner can generate trajectories that always guarantee estimation optimality and accuracy.}
%		\vspace{-0.4cm}
%	\end{figure*}
	
	\subsubsection{Swarm Trajectories Comparision on Estimation Performance}
	We assessed the efficacy of our certifiable swarm planning method by comparing it with Zhou's (\cite{zhou2021ego}) swarm planner, which does not take estimation optimality into account. 
	During the swarm flight in cluttered environment, robot swarm leverage ground-truth odometry and noisy bearing with $\xi = 0.05$ (a common noise level in the real world) to perform mutual localization.
	We conducted tests with both 3-robot and 5-robot configurations.
	The resulting swarm trajectories and estimation errors illustrated in Fig. \ref{fig:planning_benchmark_3} and Fig. \ref{fig:planning_benchmark_5}, respectively.
	Blue paths indicate points where the estimation solution was certified as globally optimal, while red paths denote the opposite.
	
	The results show that robots following Zhou's trajectories frequently fail to achieve globally optimal solutions during motion.
	In contrast, our planner consistently generates trajectories that ensure estimation optimality.
	This difference is also reflected in the estimation error comparisons: robots on Zhou's trajectories often exhibit inaccuracies in both rotation and translation due to unoptimal solutions.
	Although optimal estimations are sporadically achievable, neither their frequency nor their accuracy are sufficient to reliably align the robots' reference frames. 
	In contrast, robots navigating our planned trajectories consistently exhibit low estimation errors, demonstrating the superior performance of our certifiable swarm planning method.
	
	\subsection{Integrated Evaluation in Swarm Navigation}
	\label{subsec:overall}
	%	\begin{figure*}[!th]
	%		\centering
	%		\includegraphics[width=1.0 \textwidth]{./figure/simulation_navigation.png}
	%		\caption{\label{fig:sim_navigation} Simulation navigation}
	%		\vspace{-0.4cm}
	%	\end{figure*}
	\begin{figure*}[!th]
		\centering
		\includegraphics[width=1.0 \textwidth]{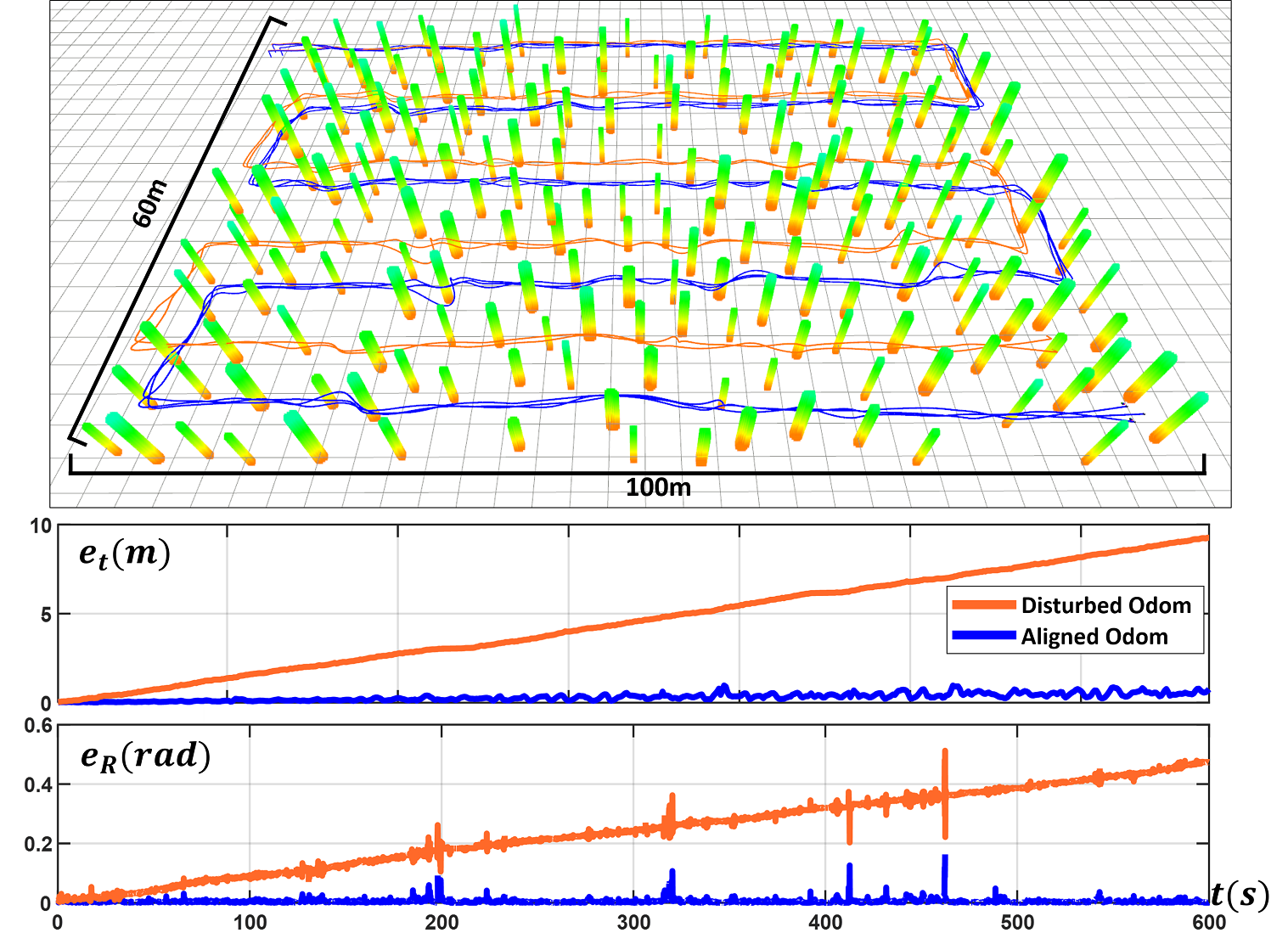}
		\caption{\label{fig:long_navigation} Swarm trajectories and estimation error in a long-range navigation.
		Our estimator and planner effecitively mitigate the accumulated odometry drift.}
		\vspace{-0.1cm}
	\end{figure*}
	We conducted a long-range swarm navigation experiment to evaluate the overall performance of our estimator and planner. 
	As depicted in Fig.\ref{fig:long_navigation}, three robots navigate through a complex 60m $\times$ 100m environment.
	 To simulate odometry drift typically encountered in long-range flights, we follow the method in (\cite{geneva2020openvins}) to obtian noisy position  $p_n(t)$ and noisy yaw $\theta_n(t)$:
	 \begin{align*}
	 	p_n(t) = p(t) + b_p(t),\  \theta_n(t) = \theta(t) + b_\theta(t).
	 \end{align*}
	 $p(t)$ and $\theta(t)$ are the ground-truth position and yaw.
	 $b_p(t)$ and $b_\theta(t)$ are random walk biases:
	\begin{align*}
		&b_p(t+\Delta t) = b_p(t) + \sigma_p \sqrt{v(t) \Delta t}\ \text{Gennoise}(0,1,3), \\
		&b_\theta(t+\Delta t) = b_\theta(t) + \sigma_\theta \sqrt{\omega(t) \Delta t}\ \text{Gennoise}(0,1,1),
	\end{align*}
	where $v(t)$ and $\omega(t)$ are linear and angular velocity, and $\text{Gennoise}(m,n,v)$ generates a $\mathbb{R}^v$ vector with mean $m$ and variance $n$.
	We considered only yaw disturbances as they are typically unobservable in inertial-based odometry
	Two of the robots' odometry were disturbed, allowing us to compare localization errors before and after alignment with the undisturbed robot.
	We set $\sigma_p = 0.1$ and $\sigma_\theta = 0.01$.
	
	During navigation, the robots used 10Hz LiDAR to gather environmental data and build local occupancy maps.
	They generated swarm trajectories considering $\xi_{\text{max}} = 0.05$ with a replanning framework to reach their target points. 
	Throughout the flight, the robots effectively avoided collisions and maintained mutual visibility, with trajectories designed to satisfy the certificate eigenvalue requirement.
	
	Noisy bearing measurements with $\xi = 0.05$ were collected for estimation. 
	By optimizing over a sliding window, the robots recovered relative poses and obtained aligned odometry to mitigate accumulated drift.
	The absolute estimation error of translation and rotation for both disturbed and aligned odometry is illustrates in Fig.\ref{fig:long_navigation}.
	For translation and yaw, the disturbed odometry showed increasing error due to accumulated drift, while odometry aligned using our estimator maintained extremely low error.
	Overall, the results demonstrate that our estimator and planner can effectively mitigate accumulated odometry drift and enhance swarm consistency in long-range navigation."
	
	\section{Real-world Experiments}
	\label{sec:real-world experiments}

	In this section, we apply our proposed method in real-world experiments.
	In Sec.\ref{subsec:hardware}, we introduc the hardware platform.
	In Sec.\ref{subsec:exp_1_estimation} and Sec.\ref{subsec:exp_2_planning}, we evaluate the certifiable mutual localization, and certifiable swarm planning planner using the real-world data, respectively.
	In Sec.\ref{subsec:exp_3_mapfusion}, we apply our methods in several missions, including large-scale map fusion and long-range swarm navigation, to present the effectiveness and practical robustness of our methods.
	The detai of the real-world  experiments is presented in the supplementaried videos.
	
	\begin{table*}[t]
		\centering
		\caption{\label{tab:real_estimation}Estimation errors of different methods with real-world data}
		\begin{threeparttable}
			\begin{tabular}{c|c|c|cc|cc|cc|cc|cc}
				\hline
				\multirow{2}{*}{Exp.}   &\multirow{2}{*}{Configuration} & \multirow{2}{*}{Metrics} & \multicolumn{2}{c|}{NLS} & \multicolumn{2}{c|}{SDP-QP} & \multicolumn{2}{c|}{SDP-Graph} & \multicolumn{2}{c|}{AM} & \multicolumn{2}{c}{Proposed}  \\ \cline{4-13} 
				&  &      & Pos  & Rot   & Pos       & Rot     & Pos         & Rot     & Pos        & Rot       & Pos         & Rot     \\ \hline
				\multirow{2}{*}{1} & \multirow{2}{*}{2robots (2465)}
				& AE & 1.162 & 24.0$^\circ$  & 0.482  & 7.5$^\circ$  &  0.171  & 1.7$^\circ$  & 0.221  & 1.5$^\circ$ & \textbf{0.152} & \textbf{1.4$^\circ$ } \\ 
				& & RE & 1.374 & 26.1$^\circ$  & 0.761  & 7.9$^\circ$  &  0.223  & 1.8$^\circ$  & 0.263  & 1.6$^\circ$ & \textbf{0.203} & \textbf{1.3$^\circ$ } \\ \hline
				\multirow{2}{*}{2} & \multirow{2}{*}{2robots (806)}
				& AE & 1.891 & 31.1$^\circ$  & 0.630  & 6.2$^\circ$  & 0.575   & 5.2$^\circ$  & 0.604 & 9.8$^\circ$ & \textbf{0.115} & \textbf{3.2$^\circ$ } \\ 
				& & RE & 1.182 & 22.3$^\circ$  & 0.615  & 5.8$^\circ$  & 0.537   & 4.6$^\circ$  & 0.572 & 6.8$^\circ$ & \textbf{0.100} & \textbf{1.9$^\circ$}  \\ \hline
				\multirow{2}{*}{3} & \multirow{2}{*}{3robots (321)}
				& AE & 1.323  & 21.0$^\circ$ & 0.358  & 7.9$^\circ$  & 0.230   & 5.0$^\circ$  & 0.254 & 6.4$^\circ$ & \textbf{0.122} & \textbf{2.0$^\circ$ } \\ 
				& & RE & 1.091  & 25.7$^\circ$  & 0.328  & 4.9$^\circ$ & 0.193   & 1.8$^\circ$  & 0.206 & 2.9$^\circ$  & \textbf{0.123} & \textbf{1.2$^\circ$ } \\ \hline
				\multirow{2}{*}{4} & \multirow{2}{*}{3robots (923)}
				& AE & 1.378 & 14.1$^\circ$  & 0.206  & 3.1$^\circ$  & 0.103   & 2.9$^\circ$  & 0.308 & 5.9$^\circ$ & \textbf{0.171} & \textbf{2.2$^\circ$ } \\ 
				& & RE & 1.168 & 16.9$^\circ$  & 0.184   & 2.0$^\circ$ & 0.158   &1.9$^\circ$   & 0.431 & 4.2$^\circ$  & \textbf{0.127} & \textbf{1.4$^\circ$ } \\ \hline
				\multirow{2}{*}{5} & \multirow{2}{*}{3robots (951)}
				& AE & 1.584 & 19.9$^\circ$  &  0.213 & 2.5$^\circ$  & 0.385  & 4.3$^\circ$   & 1.682  & 6.1$^\circ$   & \textbf{0.138} & \textbf{2.3$^\circ$ } \\ 
				& & RE & 1.980 & 13.9$^\circ$  &  0.154 & 1.2$^\circ$   & 0.292  & 2.6$^\circ$ & 1.493 & 5.9$^\circ$   & \textbf{0.105} & \textbf{1.0$^\circ$ } \\ \hline
				\multirow{2}{*}{6} & \multirow{2}{*}{4robots (973)}
				& AE & 1.882 & 19.3$^\circ$  & 0.349  & 2.9$^\circ$  & 0.411 & 3.2$^\circ$  & 1.669   & 38.1$^\circ$ & \textbf{0.137} & \textbf{2.1$^\circ$ } \\
				& & RE & 1.583 & 23.0$^\circ$  &  0.275  & 1.9$^\circ$  & 0.377 & 2.4$^\circ$  & 1.576  & 31.9$^\circ$  & \textbf{0.103} & \textbf{1.1$^\circ$ } \\ \hline
				\multirow{2}{*}{7} & \multirow{2}{*}{4robots (221)}
				& AE & 1.121 & 36.3$^\circ$ &  0.192 & 3.1$^\circ$   & 0.334   & 4.7$^\circ$  & 1.362 & 34.5$^\circ$   & \textbf{0.161} & \textbf{1.8$^\circ$ } \\
				& & RE & 1.283 & 37.0$^\circ$ & 0.187 & 2.7$^\circ$   & 0.212  & 3.4$^\circ$  & 0.973  & 35.9$^\circ$   & \textbf{0.123} & \textbf{2.1$^\circ$ } \\ \hline
			\end{tabular}
			%		\begin{tablenotes}
				%			TABLE: Estimation errors of different methods with real-world data. 
				%			The \textbf{first} and the second best results are ranked for each row. 
				%			AE: absolute error.
				%			RE: relative error.
				%		\end{tablenotes}
		\end{threeparttable}
	\end{table*}
	\subsection{Hardware Platforms and Implementation Details}
	\label{subsec:hardware}
	\begin{figure}[!t]
		\centering
		\includegraphics[width=0.5 \textwidth]{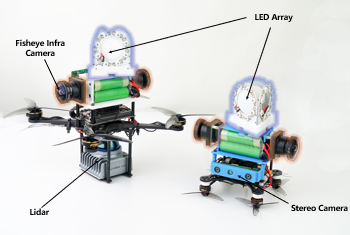}
		\caption{\label{fig:hardware} The heterogeneous sensing quadrotor platforms and the omnidirectional detection equipments used in our experiments.}
		\vspace{-0cm}
	\end{figure}
	
	The  hardware platforms in our real-world experiments is shown in Fig.\ref{fig:hardware}.
	We follow our previous inter-robot detection equipment (\cite{xun2023crepes}). 
	Each robots are equiped with a near-infra LED array and two fisheye infra cameras with $185^\circ$ FOV.
	During the flight, robots' LED array blinks with different frequency to idetify their id.
	Detected by the fisheye infra cameras, robots can obtain inter-robot bearing measurenments in 20Hz.
	We leverage motion capture to calibrate the bearing detection precision before experiments.
	When the distance between two robots is in a normal range (1m $\sim$ 5m), the average magnitude of the detection error in the real world is about $\xi = 0.04$.

	\begin{figure}[t]
		\centering
		\includegraphics[width=0.5 \textwidth]{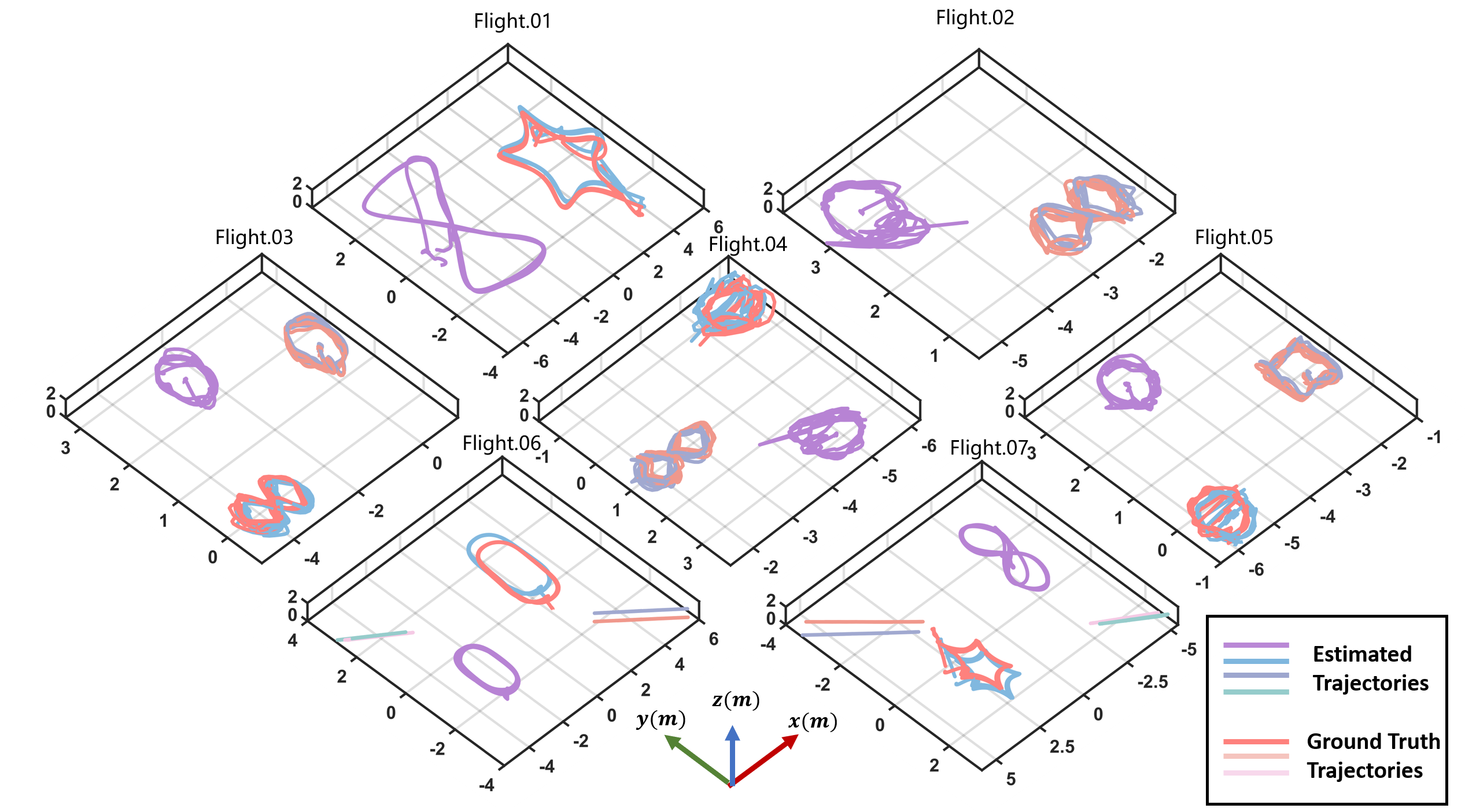}
		\caption{\label{fig:real_estimation} Demonstration of the ground-truth and the estimated trajectories in real-world experiments. }
		\vspace{-0.1cm}
	\end{figure}
	\subsection{Swarm Localization with Real-world Data}
	\label{subsec:exp_1_estimation}
%	\begin{table}[b]
%		\label{tab:real_used_time}
%		\centering
%		\begin{tabular}{c|c|c|c|c|c|c}
%			\hline
%			Exp.&  Config.       & NLS    & SDP-QP & SDP-Graph & AM   & Proposed \\ \hline
%			1&\begin{tabular}[c]{@{}c@{}}2robots\\ (2465)\end{tabular}  & 0.07  & 0.44 & 0.46 & 0.12 & 0.19     \\ \hline
%			2&\begin{tabular}[c]{@{}c@{}}2robots\\ (806)\end{tabular}   & 0.04  & 0.46 & 0.47 & 0.21 & 0.24     \\ \hline
%			3&\begin{tabular}[c]{@{}c@{}}3robots\\ (321)\end{tabular}   & 0.04  & 0.53 & 0.64 & 0.19 & 0.23     \\ \hline
%			4&\begin{tabular}[c]{@{}c@{}}3robots\\ (923)\end{tabular}   & 0.05  & 0.49 & 0.69 & 0.17 & 0.28     \\ \hline
%			5&\begin{tabular}[c]{@{}c@{}}3robots\\ (951)\end{tabular}   & 0.03  & 0.59 & 0.72 & 0.22 & 0.34     \\ \hline
%			6&\begin{tabular}[c]{@{}c@{}}4robots\\ (973)\end{tabular}   & 0.11  & 1.22 & 0.96 & 0.29 & 0.37     \\ \hline
%			7&\begin{tabular}[c]{@{}c@{}}4robots\\ (221)\end{tabular}   & 0.11  & 1.15 & 0.98 & 0.26 & 0.31     \\ \hline
%		\end{tabular}
%	\end{table}
	
	In this experiments, we aim to compare our estimator with previous methods using real-world data.
	We conduct experiments using two robots, three robots and four robots, respectively. 
	These robots are allowed to move along 3D pre-determined trajectories and observe each other in a motion capture room that provide the ground-truth poses at 120Hz.
	These trajectories are enough twisted to perform mutual localization.
	We then collect the local poses from visual-inertial odometry (VIO) or lidar-inertial odometry (LIO), and actual bearing measurements to recover relative poses.
	Except the absolute estimation error (AE), we also consider the relative estimation error (RE), defined as 
	\begin{align}
		e_{R}^* &= \sum_{(i,j) \in \mathcal{E}} \sum_{\tau \in J_{ij}} \Norm{\mathbf{R}_{ij}^\tau \tp \mathbf{\hat{R}}_{ij}^\tau} / N \\
		e_{t}^* &= \sum_{(i,j) \in \mathcal{E}} \sum_{\tau \in J_{ij}} \Norm{t_{ij}^\tau - \hat{t}_{ij}^\tau} / N.
	\end{align}

	Compared with AE, RE can better represent the consistency of the swarm at each moment in the movement.
	The estimation errors are show in Table \ref{tab:real_estimation}.
	As it shown, our proposed method surpasses the other methods in all experiments in accuracy and consistency.
	Fig.\ref{fig:real_estimation} illustrates the groud-truth trajectories and the estimated trajectories in the experiments.
	In all tests, our estimator can consistently and accrately recover the correct relative poses

	\subsection{Certifiable Swarm Planning in the Real World}
	\label{subsec:exp_2_planning}
	
	\begin{figure*}[!t]
		\centering
		\includegraphics[width=1.0 \textwidth]{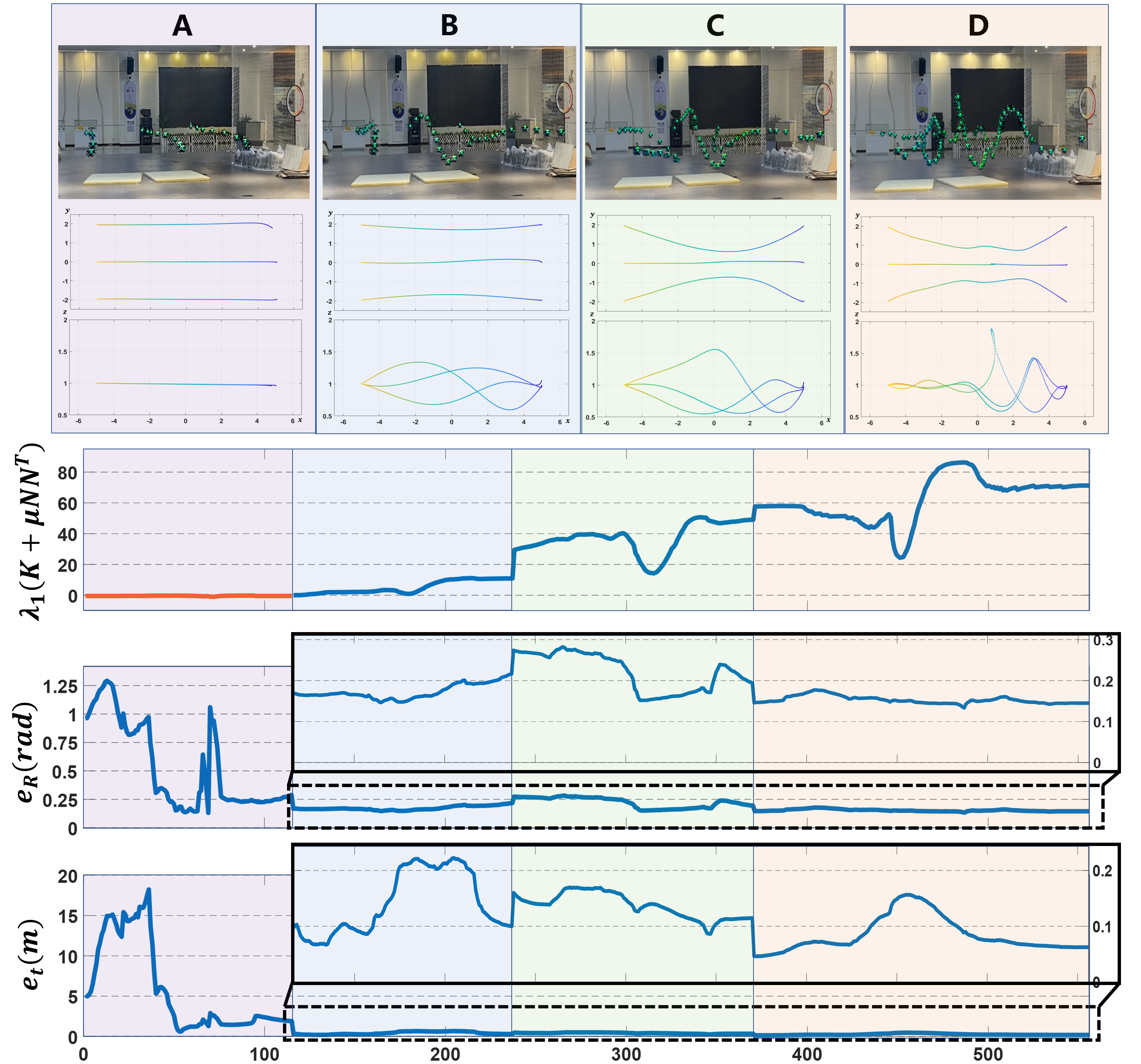}
		\caption{\label{fig:real_trajectory_planning} Demonstration of swarm trajectories, certificate eigenvalue and estimation error in four different swarm motion. 
		In the figure of certificate eigenvalue, blue curves denote that $\mathbf{K} \succeq 0$, otherwise red, and the estimation error in the motion of subplot B,C,D are zoomed for easy viewing.
		Our certificate swarm planning effectively guarantee the estimation optimality and improve the estimation accuracy in the real world.}
		\vspace{-0cm}
	\end{figure*}
	
	We conducted a real-world evaluation of our certifiable swarm planning using three quadrotors equipped with stereo cameras, navigating along 3D trajectories in a motion capture room
	TThe central robot was designated to move from  (-5,0,1) to (5,0,1).
	We set four different valuse for $\xi_{\text{max}}$ (0, 0.01, 0.05 and 0.1) and generate different swarm trajecotries for tracking.
	The snapshot of swarm trajectories are illustrated in Fig.\ref{fig:real_trajectory_planning}.
	Different colors represent swarm trajecotries with different $\xi_{\text{max}}$.
	
	While tracking the generated trajectories, the robots obtained real-world bearing measurements using blinking LED arrays and fisheye infrared cameras to perform relative pose estimation. 
	The estimation results, including the certificate eigenvalue of the obtained solution, absolute rotation error, and absolute translation error, are presented in Fig.\ref{fig:real_trajectory_planning}.
	These results clearly indicate that the performance of the estimation significantly varies with different swarm trajectories.
	For instance, when robots fly forward in parallel (as depicted in Fig.\ref{fig:real_trajectory_planning}-A), our theoretical analysis predicts that unfriendly estimation conditions due to the near-zero $\lambda_1(\hat{\mathbf{K}} + \mu \mathbf{NN\tp})$  will result in poor noise resistance.
	This prediction is confirmed in our experiments, where the certificate eigenvalue of the obtained solution falls below zero (indicated in red).
	Meanwhile, the unoptimal estimation directly leads to huge errors in rotation and translation and render the solution unusable.

\begin{figure*}[!t]
	\centering
	\includegraphics[width=1.0 \textwidth]{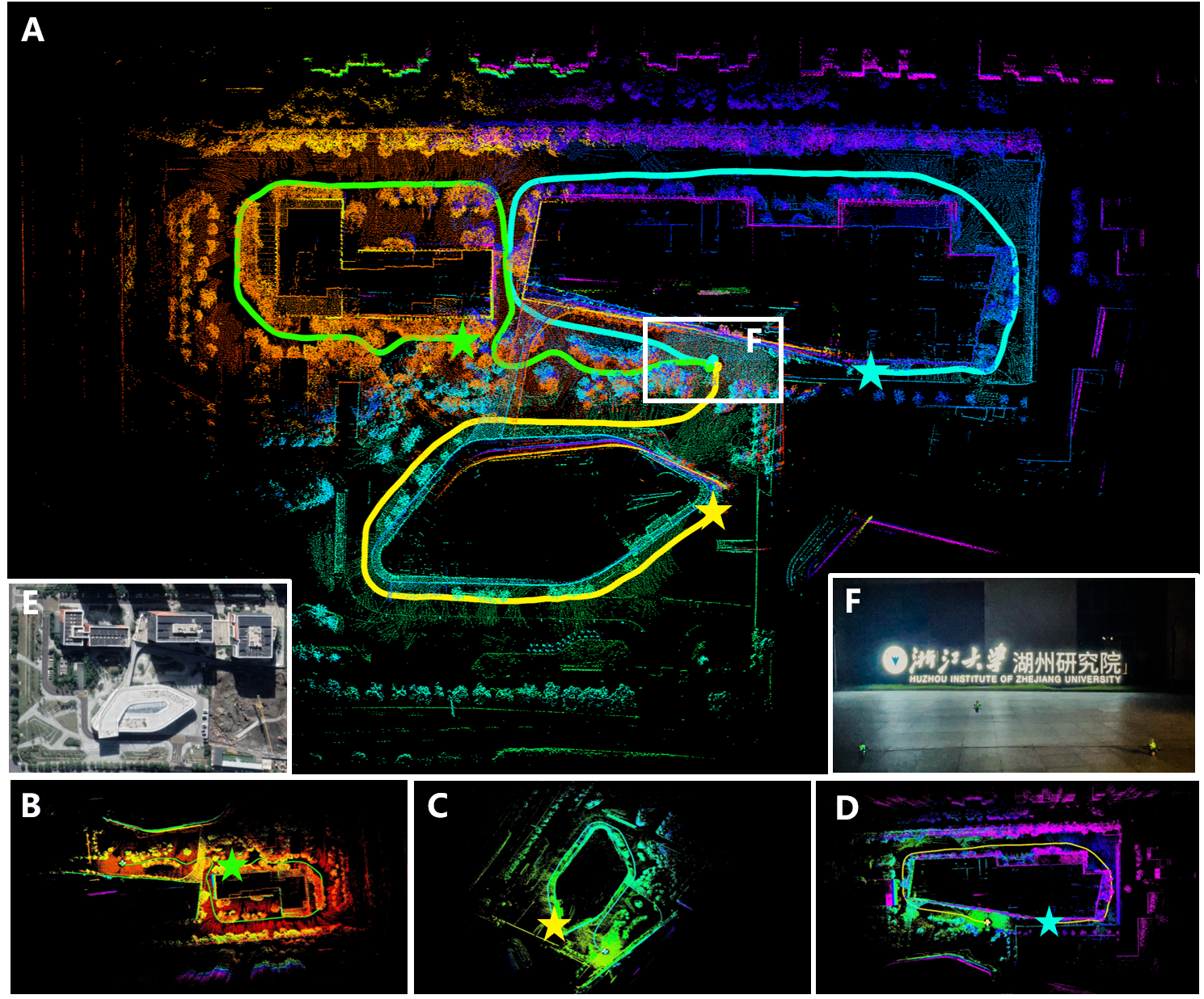}
	\caption{\label{fig:mapfusion} Large-scale map fusion with a redendezous-based strategy. 
		Robots start from three different points (denoted as colored stars) and move to construct local maps. 
		They meet each other at a designated location and fly along trajectories optimized from certifiable swarm planning.
		Generated bearing measurements and local odometry are used for certifiable mutual localization and map fusion.
		Subplot A: the fused global point map using each robot's local map and estimated relative pose by our estimator. 
		Subplot B,C,D: the local point maps of robots.
		Subplot E: a snapshot of robots in the rendezvous location.
		Subplot F: the satellite image of buildings from Google Earth.}
	\vspace{-0cm}
\end{figure*}

	In contrast, when $\xi_{\text{max}}$ increases (as shown in Fig.\ref{fig:real_trajectory_planning}-B,C,D), the certificate eigenvalues of the obtained solutions remain consistently positive, ensuring the optimality of the estimations. 
	It indicates that the impact of noise on estimation accuracy is effectively mitigated, resulting in accurate relative poses suitable for safe reference frame alignment. 
	This experiment validates that trajectories generated through certifiable swarm planning can successfully counteract real-world errors and consistently provide high-quality estimation results.
	
	Furthermore, these experimental results align with findings from our simulation experiments in Sec.\ref{subsubsec:impact_of_motion}, affirming that swarm trajectories planned with a certain $\xi_{\text{max}}$ can resist noise levels exceeding that specific $\xi_{\text{max}}$.  
	For instance, trajectories generated with $\xi_{\text{max}}=0.01$ demonstrate sufficient resistance to real-world noise. 
	This finding allows us to select a smaller $\xi_{\text{max}}$ in our swarm planning to strike a balance the needs for noise resistance and motion smoothnesss.
	This balance is crucial, as trajectories with excessively high $\xi_{\text{max}}$ tend to be overly twisted, making them difficult to track, as illustrated in Fig.\ref{fig:real_trajectory_planning}-D.
%	However, it can be seen that as the maximum allowed noise increases, the trajectories become too twisted and are not friendly to control.
%	Thus, in practical application, we choose $\xi_{\text{max}} = 0.05$ to balance resistance and smoothness.

	\begin{figure*}[!t]
		\centering
		\includegraphics[width=1.0 \textwidth]{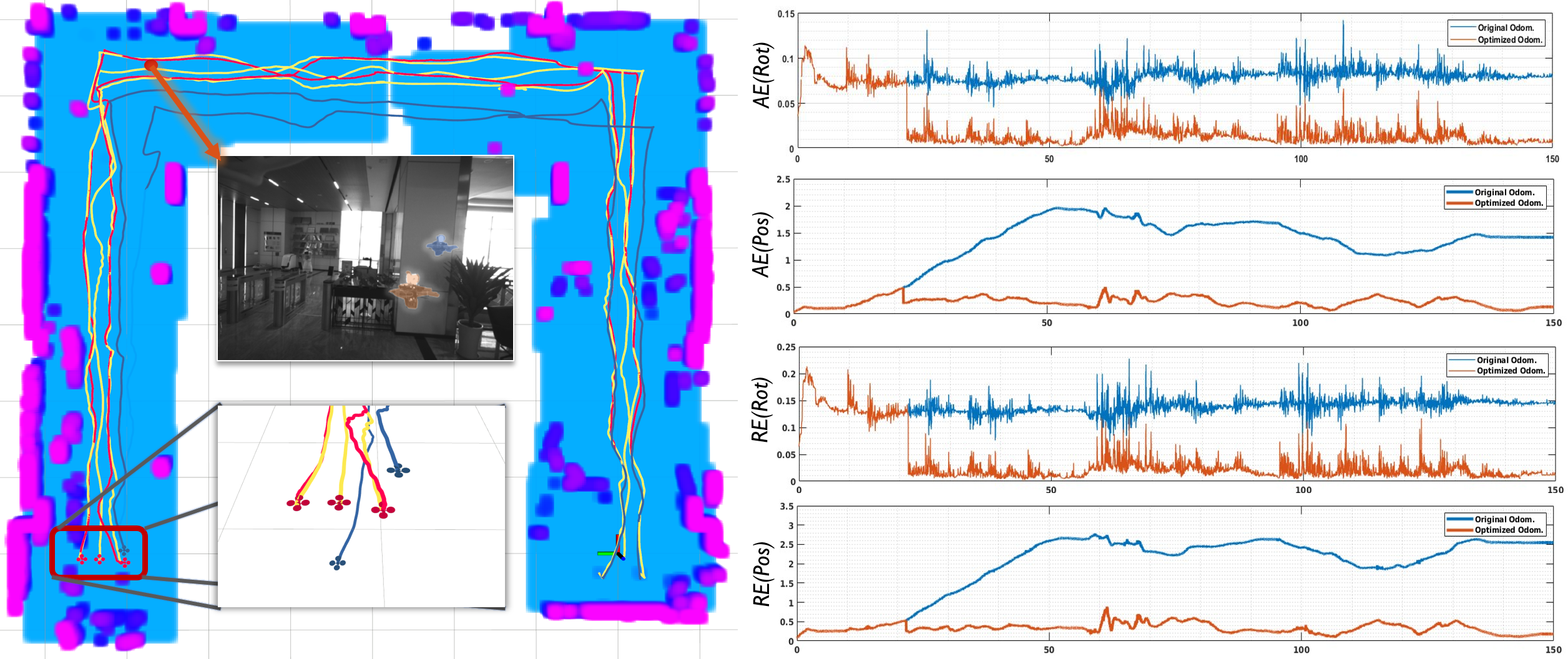}
		\caption{\label{fig:navigation} Swarm navigation to decrease accumated odometry drift.
			The left figure shows the occupancy map and robots' trajectories during 150m flight.
			The right figure shows the localization error of original and optimized odometry, respectively.}
		\vspace{-0cm}
	\end{figure*}
	\subsection{Applications}
	\label{subsec:exp_3_mapfusion}
	Both certifiable mutual localization and certifiable swarm planning are verified through respective real-world experiments.
	We further conduct two experiments to demonstrate the practical application of our algorithms, including large-scale map fusion and long-range swarm navigation.
	
	Map fusion is a crucial component in multi-robot tasks like exploration and reconstruction. 
	Typically, robots are deployed from different locations to independently conduct map building. 
	When the global starting locations are unknown, fusing these individual maps is challenging due to the absence of initial relative pose estimates.
	To address this, we employed our certifiable mutual localization method combined with a rendezvous-based strategy (\cite{gao2021meeting}).
	Three quadrotors, each equipped with LiDAR, take off from different locations and follow pre-determined trajectories.
	During movement, each robot estimates its state by a LIO (\cite{xu2022fast}) and construct local point maps.
	Upon completion, the robots reconvened at a common location. 
	There, they follow trajectories generated by our certifiable swarm planning and produce inter-robot bearing measurements.
	Then they perform certifiable mutual localization to recover the relative poses of their starting  points.
	Utilizing the solved relative poses, we fuse local maps into a global map.
	As shown in Fig.\ref{fig:mapfusion}, we effectively construct a global point cloud map (Fig.\ref{fig:mapfusion}-A) of buildings on a 400m $\times$ 300m institute campus (Fig.\ref{fig:mapfusion}-E), using individual maps from the robots (Fig.\ref{fig:mapfusion}-B,C,D).
	There is some misalignments in fused map, which might due to the accumulated odometry drifts and the estimation error.
	This experiment presents the potential of our methods in multi-robot reconstruction within large-scale environments such as underground areas and cities.
	
	Navigation is a fundamental capability of any robot system.
	For a swarm, robots odometry drift after long-range movement is unavoidable when a global localization is unavailable.
	Therefore, mutual localization is required, and its robustness and accuracy is extremely essential.
	Here, we apply our certifiable swarm planning on this mission.
	Three robots, one equiped with lidar and two equiped with stereo cameras, need to navigate through three waypoints in a clustered environments.
	The LiDAR-equipped robot runs a LIO (\cite{xu2022fast}), while the other two robots run a VIO (\cite{qin2018vins}) to estimate their states.
	Meanwhile, robots construct occupancy grid map and Euclidean signed distance field for safety.
	Following our previous work (\cite{hou2022enhanced}), we share incremental map among robots for perception consistency. 
	Then, robots perform certifiable swarm planning and certifiable mutual localzation simultaneouly during a 150m flight, as shown in Fig.\ref{fig:navigation}
	In the experiment, the optimized estimation-friendly trajectories provide a favorable condition for stable and high-quality bearing measurement generation, while the periodic mutual localization provides relative pose estimation with optimality guarantee.
%	In practice, due to the limited computation capacity, the camera robots can only maintain a small sliding window in VIO.
	Despite obvious accumulated drift after long-range flight, the robots can effectively mitigrate the drift and still maintain a correct consesus of reference frame.
	This experiments present the plug-and-play application of our methods in a heterogeneous sensing robot swarm.
	In actual, our proposed method can be migrated to any other platforms, such as swarm of wheeled robot or legged robot, without any fine-tuning.

	\section{Conclusion}
	\label{sec:conclusion}
	Complete and systematic methods for improving the flexibility, practicality and robustness of bearing-based robot swarn is proposed in this paper.
	We firstly provide a certifiable mutual localization algorithm.
	In noise-limited cases, it can produce globally optimal solutions for relative pose estimation.
	Meanwhile, it provide a way to certify whether a candidate solution is globally optimal in generally noisy cases.
	To free the estimator from the burden of detection noise and environment interaction, we then study how to leverage swarm planning to improve mutual estimation.
	A series of theroretical conclusions are put forth, including the dual description, noise analysis, degeneration identification and eigenvalue bound, which reveal that appropriate swarm motion can guarantee estimation optimality in any noisy cases.
	Based on it, we design a complete swarm planner, considering both inter-robot visibility and noise resistance.
	Our proposed estimator, theroretical findings and swarm planner are comprehensively evaluated in both simulation and real-world experiments.
	Our estimator outperfomrs all previous methods and our planner effectively provide performance assurance for estimation.
	Thanks to the systematic improvement in this paper, our bearing-based robot swarm begin to show a strong potential in significant applications, such as large-scale map fusion and long-range swarm navigation.
	We believe that our work makes the widespread application of the bearing-based swarm a promising direction in robotics field.
	
	\bibliography{ijrr2023}
	\section{Appendix A: Proof of Theorem \ref{theorem:1}}
	\label{appendix_a}
	
	Differentiating (\ref{equ:dual0}) gives the necessary optimal condition
	\begin{align*}
		(\mathbf{M} - \mathbf{\Lambda^*}) \mathbf{\Theta_R^{*\mathrm{T}}} = \mathbf{0}.
	\end{align*}
	It states that the rows of a local minimizer $\mathbf{\Theta_R^*}$ lies in the nullpace of the matrix $\mathbf{M} - \mathbf{\Lambda^*}$.
	Expanding it gives
	\begin{align*}
		\mathbf{\Lambda}_i^* \mathbf{R}_i^* \tp= \sum_{j=1}^N \mathbf{M}_{ij} \mathbf{R}_j^*\tp
	\end{align*}
	Thus, we can obtain $\mathbf{\Lambda}_i^*$ in closed-form $\mathbf{\Lambda}_i^* = \sum_{j=1}^N \mathbf{M}_{ij} \mathbf{R}^*_j\tp \mathbf{R}_i^*$.
	
	Next, the standard duality theory reveals that the problem (\ref{equ:r3}) gives a lower bound on the original problem (\ref{equ:pro}). 
	If $\mathbf{M-\Lambda^*} \succeq 0$, then $\mathbf{\Lambda}^*$ is feasible in (\ref{equ:r3}).
	According to $\mathbf{\Lambda}_i^* = \sum_{j=1}^N \mathbf{M}_{ij} \mathbf{R}^*_j\tp \mathbf{R}_i^*$, we have 
	\begin{align*}
		\textup{Tr}(\mathbf{\Lambda^*}) = \textup{Tr}\left(\mathbf{M \Theta_R^*\tp \Theta_R^*}\right) = \textup{Tr}\left(\mathbf{\Theta_R^* M \Theta_R^*\tp}\right).
	\end{align*}
	It implies that there is a zero duality gap between (\ref{equ:r3}) and (\ref{equ:pro}). 
	
	Then, since $\mathbf{M-\Lambda^*}\succeq 0 $, for any $\mathbf{\Theta_R} \in \mathcal{P_R}$, we have
	\begin{align*}
		0 &\leq \textup{Tr}\left(\mathbf{\Theta_R (M-\Lambda^*) \Theta_R^\mathrm{T}}\right)  \\
		&= \textup{Tr}\left(\mathbf{ \Theta_R M \Theta_R^\mathrm{T}}\right)-\textup{Tr}\left(\mathbf{\Lambda^*}\right)  \\
		&= \textup{Tr}\left(\mathbf{\Theta_R M \Theta_R^\mathrm{T}}\right) - \textup{Tr}\left(\mathbf{\Theta_R^* M \Theta_R^*\tp}\right).
	\end{align*}
	It means that $\mathbf{\Theta_R^*}$ is a global minimum for problem (\ref{equ:pro}).

	\section{Appendix B: Proof of Theorem \ref{theorem:2}}
	\label{appendix_b}
	Splitting $\Delta \mathbf{Q}_s$ into two matrices as $\Delta \mathbf{Q}_s = \mathbf{W}_1 + \mathbf{W}_2$, where
	\begin{align*}
		\mathbf{W}_1 &= 
		\begin{bmatrix}
			-\hat{\varphi}_{ij} \delta_{ji}^\mathrm{T} & \hat{\varphi}_{ij} \delta_{ji}^\mathrm{T} \\
			\delta_{ji} \hat{\varphi}_{ij}^\mathrm{T} & -\delta_{ji} \hat{\varphi}_{ij}^\mathrm{T}
		\end{bmatrix} \\
		\mathbf{W}_2 &= 
		\begin{bmatrix}
			-\delta_{ij} \varphi_{ji}^\mathrm{T}  & \delta_{ij} \varphi_{ji}^\mathrm{T}  \\
			\varphi_{ji} \delta_{ij}^\mathrm{T} & -\varphi_{ji} \delta_{ij}^\mathrm{T}
		\end{bmatrix}.
	\end{align*}
	We perform eigenvalue decomposition for $\mathbf{W}_1$ and $\mathbf{W}_2$ respectively.
	For $\mathbf{W}_1$, we have
	\begin{align*}
		\mathbf{W}_1 
		\begin{bmatrix}
			\hat{\varphi}_{ij} \\
			-\frac{\delta_{ji}}{\xi_{ji}}
		\end{bmatrix}  
		&= 
		\begin{bmatrix}
			-\hat{\varphi}_{ij} (\delta_{ji} \cdot \hat{\varphi}_{ij}) - \hat{\varphi}_{ij} (\delta_{ji} \cdot \frac{\delta_{ji}}{\xi_{ji}}) \\
			\delta_{ji} (\hat{\varphi}_{ij} \cdot \hat{\varphi}_{ij}) + \delta_{ji} (\hat{\varphi}_{ij} \cdot \frac{\delta_{ji}}{\xi_{ji}})
		\end{bmatrix} \\
		&=
		\begin{bmatrix}
			(- \xi_{ji} - \delta_{ji} \cdot \hat{\varphi}_{ij})\hat{\varphi}_{ij} \\
			\delta_{ji} +  (\hat{\varphi}_{ij} \cdot \frac{\delta_{ji}}{\xi_{ji}}) \delta_{ji}
		\end{bmatrix}  \\
		&= (- \xi_{ji} - \delta_{ji} \cdot \hat{\varphi}_{ij}) 
		\begin{bmatrix}
			\hat{\varphi}_{ij} \\
			- \frac{\delta_{ji}}{\xi_{ji}}
		\end{bmatrix}  
	\end{align*}
	\begin{align*}
		\mathbf{W}_1 
		\begin{bmatrix}
			\hat{\varphi}_{ij} \\
			\frac{\delta_{ji}}{\xi_{ji}}
		\end{bmatrix}  
		&= 
		\begin{bmatrix}
			-\hat{\varphi}_{ij}  (\delta_{ji} \cdot \hat{\varphi}_{ij}) + \hat{\varphi}_{ij}  (\delta_{ji} \cdot \frac{\delta_{ji}}{\xi_{ji}}) \\
			\delta_{ji} (\hat{\varphi}_{ij} \cdot \hat{\varphi}_{ij}) - \delta_{ji} (\hat{\varphi}_{ij} \cdot \frac{\delta_{ji}}{\xi_{ji}})
		\end{bmatrix} \\
		&=
		\begin{bmatrix}
			(\xi_{ji} - \delta_{ji} \cdot \hat{\varphi}_{ij})\hat{\varphi}_{ij} \\
			\delta_{ji} - (\hat{\varphi}_{ij} \cdot \frac{\delta_{ji}}{\xi_{ji}}) \delta_{ji}
		\end{bmatrix}  \\
		&= (\xi_{ji} - \delta_{ji} \cdot \hat{\varphi}_{ij}) 
		\begin{bmatrix}
			\hat{\varphi}_{ij} \\
			\frac{\delta_{ji}}{\xi_{ji}}
		\end{bmatrix}  
	\end{align*}
	\begin{align*}
		\mathbf{W}_1 
		\begin{bmatrix}
			\phi_{ij} \\
			-\phi_{ij}
		\end{bmatrix}  
		&= 
		\begin{bmatrix}
			-\hat{\varphi}_{ij}  (\delta_{ji} \cdot \phi_{ij}) - \hat{\varphi}_{ij}  (\delta_{ji} \cdot \phi_{ij}) \\
			\delta_{ji} (\hat{\varphi}_{ij} \cdot \phi_{ij}) + \delta_{ji} (\hat{\varphi}_{ij} \cdot \phi_{ij})
		\end{bmatrix} \\
		&=
		\begin{bmatrix}
			0\\
			0
		\end{bmatrix}  = 
		0
		\begin{bmatrix}
			\phi_{ij} \\
			-\phi_{ij}
		\end{bmatrix}  
	\end{align*}
	Thus, we obtain that $\mathbf{W}_1 = \mathbf{V}_1 \mathbf{U}_1 \mathbf{V}_1^\mathrm{T}$ where
	\begin{align*}
		\mathbf{V}_1 &= 
		\left[\begin{array}{c|c|c|c|c|c}
			\hat{\varphi}_{ij}           & \hat{\varphi}_{ij}            &  \phi_{ij}   &  e_1 & e_2 & e_3 \\
			-\frac{\delta_{ji}}{\xi_{ji}} & \frac{\delta_{ji}}{\xi_{ji}} &  -\phi_{ij}  &  e_1 & e_2 & e_3 \\
		\end{array}\right], \\
		\mathbf{U}_1 =&\ \textup{Diag}([-\xi_{ji} - \delta_{ji} \cdot \hat{\varphi}_{ij} , \xi_{ji} - \delta_{ji} \cdot \hat{\varphi}_{ij},\mathbf{0}_{1\times4}])
	\end{align*}
	Similar to $\mathbf{W}_1$, for $\mathbf{W}_2$ we have
	\begin{align*}
		\mathbf{W}_2 
		\begin{bmatrix}
			-\frac{\delta_{ij}}{\xi_{ij}}\\
			\varphi_{ji} 
		\end{bmatrix}  
		&= 
		\begin{bmatrix}
			\delta_{ij} (\varphi_{ji} \cdot \frac{\delta_{ij}}{\xi_{ij}}) + \delta_{ij} (\varphi_{ji} \cdot \varphi_{ji}) \\
			-\varphi_{ji} (\delta_{ij} \cdot \frac{\delta_{ij}}{\xi_{ij}}) - \varphi_{ji} (\delta_{ij}\cdot \varphi_{ji})
		\end{bmatrix} \\
		&=
		\begin{bmatrix}
			\delta_{ij} + (\varphi_{ji} \cdot \frac{\delta_{ij}}{\xi_{ij}}) \delta_{ij}  \\
			(-\xi_{ij} - \delta_{ij} \cdot \varphi_{ji}) \varphi_{ji} 
		\end{bmatrix}  \\
		&= (-\xi_{ij} - \delta_{ij} \cdot \varphi_{ji})
		\begin{bmatrix}
			-\frac{\delta_{ij}}{\xi_{ij}}\\
			\varphi_{ji} 
		\end{bmatrix} 
	\end{align*}
	\begin{align*}
		\mathbf{W}_2 
		\begin{bmatrix}
			\frac{\delta_{ij}}{\xi_{ij}}\\
			\varphi_{ji} 
		\end{bmatrix}  
		&= 
		\begin{bmatrix}
			-\delta_{ij} (\varphi_{ji} \cdot \frac{\delta_{ij}}{\xi_{ij}}) + \delta_{ij} (\varphi_{ji} \cdot \varphi_{ji}) \\
			\varphi_{ji} (\delta_{ij} \cdot \frac{\delta_{ij}}{\xi_{ij}}) - \varphi_{ji} (\delta_{ij}\cdot \varphi_{ji})
		\end{bmatrix} \\
		&=
		\begin{bmatrix}
			\delta_{ij} - (\varphi_{ji} \cdot \frac{\delta_{ij}}{\xi_{ij}}) \delta_{ij}  \\
			(\xi_{ij} - \delta_{ij} \cdot \varphi_{ji}) \varphi_{ji} 
		\end{bmatrix}  \\
		&= (\xi_{ij} - \delta_{ij} \cdot \varphi_{ji})
		\begin{bmatrix}
			\frac{\delta_{ij}}{\xi_{ij}}\\
			\varphi_{ji} 
		\end{bmatrix} 
	\end{align*}
	\begin{align*}
%		\phi_{ji} = \varphi_{ji} \times \delta_{ij} \\
		\mathbf{W}_2 
		\begin{bmatrix}
			\phi_{ji}\\
			-\phi_{ji} 
		\end{bmatrix}  
		&= 
		\begin{bmatrix}
			-\delta_{ij} (\varphi_{ji} \cdot \phi_{ji}) - \delta_{ij} (\varphi_{ji} \cdot \phi_{ji}) \\
			\varphi_{ji} (\delta_{ij} \cdot \phi_{ji}) + \varphi_{ji} (\delta_{ij}\cdot \phi_{ji})
		\end{bmatrix} \\
		&=
		\begin{bmatrix}
			0 \\
			0
		\end{bmatrix} = 0
		\begin{bmatrix}
			\phi_{ji}\\
			-\phi_{ji}  
		\end{bmatrix} 
	\end{align*}
	Thus, we obtain that $\mathbf{W}_2 = \mathbf{V}_2 \mathbf{U}_2 \mathbf{V}_2^\mathrm{T}$ where
	\begin{align*}
		\mathbf{V}_2 &= 
		\left[\begin{array}{c|c|c|c|c|c}
			-\frac{\delta_{ij}}{\xi_{ij}} & \frac{\delta_{ij}}{\xi_{ij}}  &  \phi_{ji}   &  e_1 & e_2 & e_3 \\
			\varphi_{ji}                 & \varphi_{ji}                  &  -\phi_{ji}  &  e_1 & e_2 & e_3 \\
		\end{array}\right], \\
		\mathbf{U}_2 =&\ \textup{Diag}([-\xi_{ij} - \delta_{ij} \cdot \varphi_{ji}, \xi_{ij} - \delta_{ij} \cdot \varphi_{ji},\mathbf{0}_{1\times4}])
	\end{align*}
	
	\section{Appendix C: Proof of Theorem \ref{theorem:3}}
	\label{appendix_c}
	For $d = 0,1,2$, $\mathbf{J}_{[6],S}^{(d)}$ can be written as 
	\begin{align*}
		\mathbf{J}_{[6],S}^{(d)} = 
		\begin{bmatrix}
			\mathbf{H}_d & \mathbf{P}_d \\
			\mathbf{S}_3 & -\mathbf{S}_3
		\end{bmatrix}, 
		\mathbf{S}_3=
		\begin{bmatrix}
			\varphi^{\tau_i} \\
			\varphi^{\tau_j} \\
			\varphi^{\tau_k}
		\end{bmatrix},
	\end{align*}
	where 
	\begin{align*}
		\mathbf{H}_d = 
		\begin{bmatrix}
			(*)_{d\times3} \\
			(**)_{(3-d)\times3}
		\end{bmatrix},
		\mathbf{P}_d = 
		\begin{bmatrix}
			(*)_{d\times3} \\
			-(**)_{(3-d)\times3}
		\end{bmatrix}.
	\end{align*}
	Here, each row in $(*)_{d\times3}$ is sampled rows from $\mu\mathbf{N}\tp$, and each row in $(**)_{(3-d)\times3}$ is sampled from $\{\varphi^\tau\}$.
	Thus, we have 
	\begin{align*}
		\mathbf{H}_d + \mathbf{P}_d = 
		\begin{bmatrix}
			2(*)_{d\times3} \\
			\mathbf{0}_{(3-d)\times3}
		\end{bmatrix}.
	\end{align*}
	Then, we consider the following lemma
	\begin{lemma}
		\label{lemma:1} $\mathbf{A}, \mathbf{B}, \mathbf{C}, \mathbf{D}$ are $n \times n$ matrices, if $\mathbf{CD} = \mathbf{DC}$, then
		\begin{align*}
			\textup{det}\left(\begin{bmatrix}
				\mathbf{A} & \mathbf{B} \\
				\mathbf{C} & \mathbf{D}
			\end{bmatrix}\right) =  \textup{det}(\mathbf{AD-BC}).
		\end{align*}
	\end{lemma}
	Based on the lemma, we can obtain that
	\begin{align*}
		\textup{det}(\mathbf{J}_{[6],S}^{(d)}) &= \textup{det}(-\mathbf{H}_d \mathbf{S}_3 - \mathbf{P}_d \mathbf{S}_3) \\
		&= -\textup{det}\left((\mathbf{H}_d + \mathbf{P}_d) \mathbf{S}_3\right) \\
		&= -\textup{det}(
		\begin{bmatrix}
			2(*)_{d\times3} \\
			\mathbf{0}_{(3-d)\times3}
		\end{bmatrix}) \textup{det}(\mathbf{S}_3) \\
		&= 0.
	\end{align*}

	Next, for $d = 3$, $\mathbf{J}_{[6],S}^{(3)}$ can be written as
	\begin{align*}
		\mathbf{J}_{[6],S}^{(3)} = 
		\begin{bmatrix}
			\sqrt{\mu}\mathbf{I}_3 & \sqrt{\mu}\mathbf{I}_3 \\
			\mathbf{S}_3 & -\mathbf{S}_3
		\end{bmatrix},
		\mathbf{S}_3=
		\begin{bmatrix}
			\varphi^{\tau_i} \\
			\varphi^{\tau_j} \\
			\varphi^{\tau_k}
		\end{bmatrix}.
	\end{align*}
	Similar to $d = 0,1,2$, based on the lemma, we can also obtain that
	\begin{align*}
		\textup{det}(\mathbf{J}_{[6],S}^{(3)}) &= -2\textup{det}(\sqrt{\mu}\mathbf{I}_3) \cdot \textup{det}(\mathbf{S}) = -2 \mu^{3/2} \text{det}(\mathbf{S}).
	\end{align*}
	
	For $\mathbf{S}$, $\text{det}(\mathbf{S}) = (\varphi^{\tau_i} \times \varphi^{\tau_j}) \cdot \varphi^{\tau_k} = \vert \text{sin}(\alpha_{ij}) \vert \text{sin}(\beta_{ijk})$, where $\alpha_{ij} = \measuredangle(\varphi^{\tau_i}, \varphi^{\tau_j})$ and $\beta_{ijk} = \frac{\pi}{2} - \measuredangle(\varphi^{\tau_i} \times \varphi^{\tau_j}, \varphi^{\tau_k})$.
	The geometric interpretation of $\text{det}(\mathbf{S})$ is the volume occupied by $\varphi^{\tau_i}, \varphi^{\tau_j}$ and $\varphi^{\tau_k}$.
	Thus, if and only if the three bearings are coplanar, then $\text{det}(\mathbf{S}) = 0$ and $\textup{det}(\mathbf{J}_{[6],S}^{(3)}) = 0$.

	\section{Appendix D: Proof of Theorem \ref{theorem:4}}
	\label{appendix_d}
	For the sake of simplicity,  let's  assume the first robot $A$ in $\mathcal{V}_c$ has no connections with the robots in $\mathcal{V}_r$.
	Consequently, the bearings between the robot $A$ and the other robots in $\mathcal{V}_c$ are denoted as $\varphi_{c_i}$.
	Additionally, the bearings between the robots in $\mathcal{V}_c / \{A\}$ and the robots in $\mathcal{V}_r$, as well as the bearings between any pair of robots in $\mathcal{V}_r$, are collectively denoted as $\varphi_{r_i}$.
	Based on these definitions, a generalized $\mathbf{J}_{[3N],S}$ can be formulated as 
	\begin{align*}
		\mathbf{J}_{[3N],S} = 
		\begin{bmatrix}
			\mathbf{1}_N\tp \otimes  \sqrt{\mu}\mathbf{I}_d \\
			\mathbf{E} \\ 
			\mathbf{F} \\ 
		\end{bmatrix} \in \mathbb{R}^{3N \times 3N},
	\end{align*}
	where $\mathbf{1}_N \in \mathbb{R}^N$ is the all-ones vector, $d\in[3]$.
	In this context, the matrix $\mathbf{E}$ corresponds to the bearings between the robot $A$ and the other robots in $\mathcal{V}_c / \{A\}$, while the matrix  $\mathbf{F}$ is associated with the bearings $\varphi_{r_i}$.
	The structure of these two matrices is as follows:
	\begin{align*}
		\mathbf{E} &= 
		\begin{bmatrix}
			\begin{bmatrix}
				\cdots \varphi_{c_1}^\mathrm{T}  \cdots  -\varphi_{c_1}^\mathrm{T}\cdots\\
				\vdots \\
				\cdots \varphi_{c_m}^\mathrm{T}  \cdots  -\varphi_{c_m}^\mathrm{T}\cdots\\
			\end{bmatrix} & \mathbf{0}_{m\times 3(N-|\mathcal{V}_c|)} 
		\end{bmatrix} \in \mathbb{R}^{m\times 3N} ,\\
		\mathbf{F} &= 
		\begin{bmatrix}
			\mathbf{0}_{n\times3} & 
			\begin{bmatrix}
				\cdots \varphi_{r_1}^\mathrm{T}  \cdots  -\varphi_{r_1}^\mathrm{T}\cdots\\
				\vdots \\
				\cdots \varphi_{r_n}^\mathrm{T}  \cdots  -\varphi_{r_n}^\mathrm{T}\cdots\\
			\end{bmatrix}
		\end{bmatrix} \in \mathbb{R}^{n\times 3N}.
	\end{align*}
	In the matrix $\mathbf{E}$, since the robot $A$ has no connections with the robots in $\mathcal{V}_r$, the right block of the matrix is entirely zeros.
	Regarding the matrix $\mathbf{F}$, in line with the definition of $\varphi_{r_i}$, its first column is completely zero. 
	Here $m$ and $n$ are the numbers of bearings $\{\varphi_{c_i}\}$ and $\{\varphi_{r_i}\}$ respectively, and $m+n+d = 3N$.
	
	Our goal is to find a non-zero vector $w$  such that $\mathbf{J}_{[3N],S} w = \mathbf{0}_{3N}$.
	Given that the robots in $\mathcal{G}_c$ move in coplanar, there exists a non-zero vector $\zeta$ perpendicular to all  $\varphi_{c_i}$, $\varphi_{c_i}^\mathrm{T}\zeta = 0$.
	Consequently, we can construct the vector $w$ in such a way that 
	\begin{align*}
		w = \zeta \otimes [-(N-1), 1, 1, \cdots, 1]^\mathrm{T} \in \mathbb{R}^{3N}.
	\end{align*}
	It satisfies that
	\begin{align*}
		(\mathbf{1}_N\tp \otimes  \sqrt{\mu}\mathbf{I}_d)w &=\left(-(N-1)\sqrt{\mu}\mathbf{I}_d + \sum_{j=1}^{N-1} \sqrt{\mu}\mathbf{I}_d\right)w \\
		&= \mathbf{0}_d.
	\end{align*}

	For the matrix $\mathbf{E}$, we have 
	\begin{align*}
		\mathbf{E}w &= [k_1\varphi_{c_1}^\mathrm{T}\zeta, \cdots, k_m\varphi_{c_m}^\mathrm{T}\zeta]\tp,
	\end{align*}
	where $k_i$ takes the values $-(N-1), 1$ or $0$, depending on the position of $\varphi_{c_i}^\mathrm{T}$.
	However, regardless of the specific value of $k_i$ is, $\mathbf{E}w = 0$ holds true since $\varphi_{c_i}^\mathrm{T}\zeta = 0, \forall i \in [1,m]$.
	
	For the matrix $\mathbf{F}$, we also have 
	\begin{align*}
		\mathbf{F}w &= [(\varphi_{r_1}-\varphi_{r_1})^\mathrm{T}\zeta, \cdots, (\varphi_{r_n}-\varphi_{r_n})^\mathrm{T}\zeta]\tp = \mathbf{0}_n.
	\end{align*}
	Thus, $\mathbf{J}_{[3N],S} w = \mathbf{0}_{3N}$. 
	It means that $\mathbf{J}_{[3N],S}$ has zero eigenvalue, and therefore $\textup{det}(\mathbf{J}_{[3N],S}) = 0$.
	
	\section{Appendix E: Proof of Theorem \ref{theorem:5}}
	\label{appendix_e}
	\begin{figure}[h]
		\centering
		\includegraphics[width=0.5\textwidth]{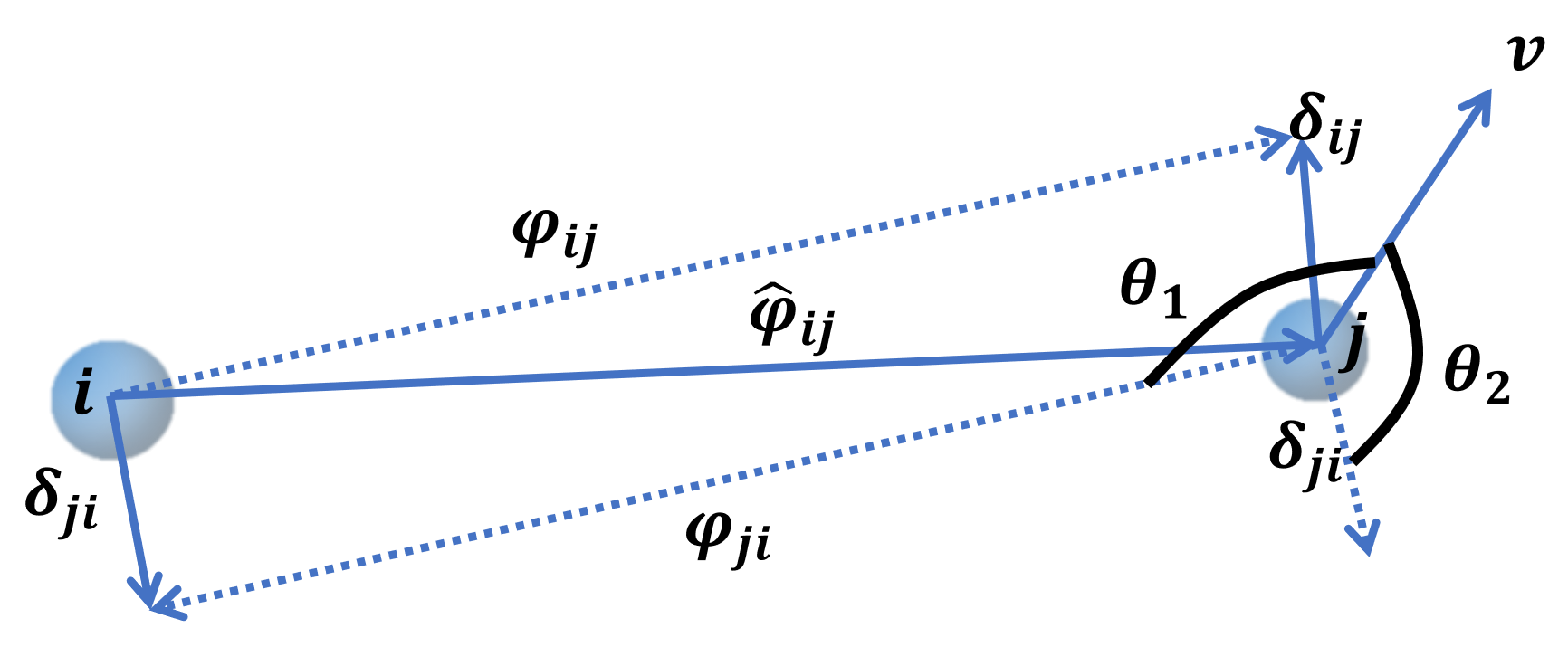}
		\caption{\label{fig:appendix} Demonstration of the groud-truth bearing $\hat{\varphi}_{ij}$, the actual bearings $\varphi_{ij}$ and $\varphi_{ji}$, and the perturbations $\delta_{ij}$ and $\delta_{ji}$.}
		\vspace{-0.4cm}
	\end{figure}

	By (\ref{equ:deltaQ1}) (\ref{equ:deltaKs}) and (\ref{equ:deltaK}), we obtain $\Delta \mathbf{K}$ such as
	\begin{align*}
		\Delta \mathbf{K}_{ij} &= \sum^{\mathcal{T}} (\delta_{ij} \hat{\varphi}_{ji}^\mathrm{T} + \hat{\varphi}_{ij} \delta_{ji}^\mathrm{T} + \delta_{ij}\delta_{ji}^\mathrm{T}) \\
		&= \sum^{\mathcal{T}} (\delta_{ij} \varphi_{ji}^\mathrm{T} + \hat{\varphi}_{ij} \delta_{ji}^\mathrm{T}),  \ \text{if} \ i \neq j. \\
		\Delta \mathbf{K}_{ii} &= -\sum_{k\neq i}^N \Delta \mathbf{K}_{ik}.
	\end{align*}
	where $\mathcal{T}$ is the sampling number. Then, we have the following lemma:
	\begin{lemma}
		\label{lemma:2}
		Let $\Delta\mathbf{K}_{ij}$ be the $(i,j)$-block of $\Delta\mathbf{K}$. If $\lambda$ is an eigenvalue of $\Delta\mathbf{K}$, then 
		\begin{align*}
			| \lambda | \leq \sum_{j=1}^N{\norm{\Delta\mathbf{K}_{ij}}}.
		\end{align*}
	\end{lemma}
	The proof is similar to that of Gerschgorin's theorem. Let $\Delta\mathbf{K} x = \lambda x$ with $\norm{x}=1$.
	Then  $\lambda x_i = \sum_{j} \Delta\mathbf{K}_{ij} x_j$.
	Pick $i$ such that $\norm{x_i} \geq \norm{x_j}$ for all $j$.
	Then
	\begin{align*}
		| \lambda | = \norm{\lambda \frac{x_i}{\norm{x_i}}} = \norm{\sum_{j=1}^N \Delta\mathbf{K}_{ij} \frac{x_j}{\norm{x_i}}} \leq \sum_{j=1}^N{\norm{\Delta\mathbf{K}_{ij}}}.
	\end{align*}

	Based on the lemma, we analyze the off-diagonal blocks $\Delta\mathbf{K}_{ij}$ and diagonal blocks $\Delta\mathbf{K}_{ii}$, respectively. 
	For any unit vector $v$, it fulfills that
	\begin{align*}
		\sqrt{\norm{\Delta\mathbf{K}_{ij}v}^2} &=
		\sqrt{\norm{\sum^\mathcal{T}  (\delta_{ij} \varphi_{ji}^\mathrm{T} + \hat{\varphi}_{ij} \delta_{ji}^\mathrm{T}) v}^2} \\
		&\leq \sum^\mathcal{T} \norm{(\delta_{ij} \varphi_{ji}^\mathrm{T} + \hat{\varphi}_{ij} \delta_{ji}^\mathrm{T}) v}
	\end{align*}
	We denote $\theta_1 = \angle(v,\varphi_{ji})$ and $\theta_2 = \angle(v,\delta_{ji})$.
	Then, as shown in Fig.\ref{fig:appendix}, due to $\norm{\hat{\varphi}_{ij}} = \norm{\varphi_{ij}} = 1$, we have 
	\begin{align*}
		\hat{\varphi}_{ij}^\mathrm{T} \delta_{ij} = -\xi_{ij}^2/2.
	\end{align*}
	Then, we have 
	\begin{align*}
		&\norm{(\delta_{ij} \varphi_{ji}^\mathrm{T} + \hat{\varphi}_{ij} \delta_{ji}^\mathrm{T})v}^2 = \norm{cos\theta_1 \delta_{ij} + cos\theta_2 \xi_{ji} \hat{\varphi}_{ij}}^2 \\
		&= cos^2\theta_1 \xi_{ij}^2 + cos^2\theta_2 \xi^2_{ji} + 2 cos\theta_1 cos\theta_2 \xi_{ji} \hat{\varphi}_{ij}^\mathrm{T} \delta_{ij}\\
		&= cos^2\theta_1 \xi_{ij}^2 + cos^2\theta_2 \xi^2_{ji} - 2 cos\theta_1 cos\theta_2 \xi_{ji} \xi_{ij}^2/2\\
		&= cos^2\theta_1 \xi_{ij}^2 + cos^2\theta_2 \xi^2_{ji} - cos\theta_1 cos\theta_2 \xi_{ji} \xi_{ij}^2 \\
		&\leq \xi_{ij}^2 + \xi^2_{ji} + \xi_{ji} \xi_{ij}^2 \leq 2\xi_{\text{max}}^2 + \xi_{\text{max}}^3.
	\end{align*}
	Thus $\norm{\Delta\mathbf{K}_{ij}} \leq \mathcal{T} \sqrt{2\xi_{\text{max}}^2 + \xi_{\text{max}}^3}$.
	Then, it is similar to  for $\Delta\mathbf{K}_{ii}$ that $\norm{\Delta\mathbf{K}_{ii}} = \norm{\sum_{k\neq i}\Delta \mathbf{K}_{ik}} \leq \sum_{k\neq i}\norm{\Delta \mathbf{K}_{ik}} \leq d_i \mathcal{T} \sqrt{2\xi_{\text{max}}^2 + \xi_{\text{max}}^3}$.
	Finally, we obtain
	\begin{align*}
		|\lambda(\Delta \mathbf{K})| \leq \norm{\Delta\mathbf{K}_{ii}} + \sum_{k\neq i} \norm{\Delta \mathbf{K}_{ik}} \leq 2 d_{\text{max}} \mathcal{T} \sqrt{2\xi_{\text{max}}^2 + \xi_{\text{max}}^3}.
	\end{align*}
	
\end{document}